\documentclass{article}
\usepackage{arxiv}

\usepackage[hidelinks]{hyperref}

\usepackage{times}  
\usepackage{helvet}  
\usepackage{courier}  
\usepackage{graphicx} 

\usepackage[numbers]{natbib}

\usepackage{caption} 
\usepackage{algorithm}
\usepackage{algpseudocode}

\usepackage{paralist}
\usepackage{newfloat}
\usepackage{listings}
\DeclareCaptionStyle{ruled}{labelfont=normalfont,labelsep=colon,strut=off} 
\floatstyle{ruled}
\newfloat{listing}{tb}{lst}{}
\floatname{listing}{Listing}
\usepackage{times}
\usepackage{amsmath}
\usepackage{soul}
\usepackage[utf8]{inputenc}
\usepackage{caption}
\usepackage{graphicx}
\usepackage{graphicx}  
\usepackage{amsmath}
\usepackage{amsthm}

\usepackage{booktabs}

\usepackage{tikz}
\usepackage{pgfplots}
\usepackage{helvet}  
\usepackage{courier}  
\usepackage{natbib}  
\usepackage{caption} 
\usepackage{listings}
\usepackage{xcolor}

\definecolor{codegreen}{rgb}{0,0.6,0}
\definecolor{codegray}{rgb}{0.5,0.5,0.5}
\definecolor{codepurple}{rgb}{0.58,0,0.82}
\definecolor{backcolour}{rgb}{0.95,0.95,0.92}

\lstdefinestyle{mystyle}{
    backgroundcolor=\color{backcolour},   
    commentstyle=\color{codegreen},
    keywordstyle=\color{magenta},
    numberstyle=\color{codegray},
    stringstyle=\color{codepurple},
    basicstyle=\ttfamily\footnotesize,
    breakatwhitespace=false,         
    breaklines=true,                 
    captionpos=b,                    
    keepspaces=true,                 
    numbers=left,                    
    numbersep=5pt,                  
    showspaces=false,                
    showstringspaces=false,
    showtabs=false,                  
    tabsize=2
}

\usepackage{newfloat}

\usepackage[utf8]{inputenc} 
\usepackage[T1]{fontenc}    
\usepackage{booktabs}       
\usepackage{amsfonts}       
\usepackage{nicefrac}       
\usepackage{microtype}      
\usepackage{xcolor}         

\usepackage[utf8]{inputenc} 
\usepackage[T1]{fontenc}    
\usepackage{booktabs}       
\usepackage{amsfonts}       
\usepackage{nicefrac}       
\usepackage{microtype}      
\usepackage{xcolor}         

\usepackage{natbib} 
\usepackage{mathtools} 
\usepackage{booktabs} 
\usepackage{tikz} 
\usepackage{amsmath}
\usepackage{xcolor}
\usepackage{amsfonts}
\usepackage{amsmath}
\usepackage{array}
\usepackage{booktabs}
\usepackage{multirow}
\usepackage{graphicx}
\usepackage{newfloat}
\usepackage{subcaption}
\usepackage[export]{adjustbox}
\usepackage{amsthm}
\usepackage{soul}
\usepackage{amssymb}
\usepackage{bibentry}

\usepackage{tikz}
\usepackage{pgfplots}

\pgfplotsset{compat=newest}
\usepgfplotslibrary{units}

\DeclareMathOperator*{\argmax}{arg\,max}

\newtheorem{theorem}{Theorem}[section]
\newtheorem{definition}{Definition}[section]

\newtheorem{assumption}{Assumption}[section]

\newcommand{\ouralgo}{\texttt{FROC}}
\newcommand{\ourdef}{-Equalized ROC}
\newcommand{\roc}{\texttt{ROC}_s}
\newcommand{\auc}{\texttt{AUC}_s}

\usepackage{microtype}
\usepackage{graphicx}
\usepackage{booktabs} 

\newcount\mycount
\usepackage{amsmath}
\usepackage{amssymb}
\usepackage{mathtools}
\usepackage{amsthm}

\usepackage[capitalize,noabbrev]{cleveref}

\title{\ouralgo: Building Fair ROC from a Trained Classifier}

\author{
Avyukta Manjunatha Vummintala \\
Machine Learning Lab\\
IIIT Hyderabad\\
\texttt{avyukta.v@research.iiit.ac.in} \\
\And
Shantanu Das \\
Machine Learning Lab\\
IIIT Hyderabad\\
\texttt{shantanu.das@alumni.iiit.ac.in} \\
\And
Sujit Gujar \\
Machine Learning Lab\\
IIIT Hyderabad\\
\texttt{sujit.gujar@iiit.ac.in} \\
}

\begin{document}

\maketitle

\begin{abstract}
This paper considers the problem of fair probabilistic binary classification with binary protected groups. The classifier assigns scores, and a practitioner predicts labels using a certain cut-off threshold based on the desired trade-off between false positives vs. false negatives. It derives these thresholds from the ROC of the classifier. The resultant classifier may be unfair to one of the two protected groups in the dataset. It is desirable that no matter what threshold the practitioner uses, the classifier should be fair to both the protected groups; that is, the $\mathcal{L}_p$ norm between FPRs and TPRs of both the protected groups should be at most $\varepsilon$. We call such fairness on ROCs of both the protected attributes $\varepsilon_p$\ourdef. Given a classifier not satisfying $\varepsilon_1$\ourdef, we aim to design a post-processing method to transform the given (potentially unfair) classifier's output (score) to a suitable randomized yet fair classifier. That is, the resultant classifier must satisfy $\varepsilon_1$\ourdef. First, we introduce a threshold query model on the ROC curves for each protected group. The resulting classifier is bound to face a reduction in AUC. With the proposed query model, we provide a rigorous theoretical analysis of the minimal AUC loss to achieve $\varepsilon_1$\ourdef. To achieve this, we design a linear time algorithm, namely \ouralgo, to transform a given classifier's output to a probabilistic classifier that satisfies $\varepsilon_1$\ourdef. We prove that under certain theoretical conditions, \ouralgo\ achieves the theoretical optimal guarantees. We also study the performance of our \ouralgo\ on multiple real-world datasets with many trained classifiers.
\end{abstract}
\section{Introduction}
\label{sec:intro}
The use of \emph{Machine Learning based Models} (MLM) in decision-making is prevalent today. Practitioners use MLMs' predictions in college admissions, credit scores, recidivism, employment, recommender systems, etc.~\cite {portugal2018use,berger2005credit}. However, there have been several reports of such MLMs discriminating against individuals belonging to certain groups based on \emph{protected attribute} such as gender, age, race, color, and religion. E.g., in ~\cite{angwin2022machine}, predictive models are found to be biased against the black population, or the Amazon recruitment team has to stop using the AI tool for shortlisting candidates as it was biased against females~\cite{dastin2018amazon}. ~\cite{bickel1977sex};~\cite{berger2005credit};~\cite{zhao2018gender} show that many of such predictive models are unfair to females. Such unfair instances have driven researchers toward building a fair MLM. 

An MLM that achieves fairness with the least possible compromise on traditional performance guarantees such as accuracy is \emph{desirable} MLM. Building a desirable MLM involves two main steps: a) formalizing and quantifying a fairness measure and b) designing algorithms to train MLM for quantified fairness. 
Researchers proposed many fairness measures, majorly belonging to two categories: (i) \emph{individual fairness}~\cite{dwork2012fairness} -- individuals with similar input features receive similar decision treatment irrespective of their protected attribute. (ii) \emph{Group fairness} -- a particular statistical property must be similar across each protected group, e.g., \emph{Disparate Impact (DI)},\emph{Equalized odds (EO)}~\cite{madras2018learning}.\\
\noindent\textbf{Building Fair MLM} Fair machine learning models (MLMs) can be developed by targeting different stages of the model training cycle. Approaches include:
(i) \emph{Pre-processing} methods, which act on input data to eliminate bias \citep{feldman15, zemel2013learning}.
(ii) \emph{In-processing} algorithms, which intervene during training to incorporate fairness as a constraint or within the learning objective \citep{padala2020fnnc}.
(iii) \emph{Post-processing} methods, which adjust the outputs of trained MLMs to produce fair results, requiring access to sensitive attributes.

In-processing and pre-processing methods are tailored to specific fairness criteria and models, necessitating retraining for each new fairness definition. Post-processing methods, in contrast, are model-agnostic and do not depend on the training process, making them suitable for domain experts with limited MLM knowledge \citep{sleeman1995consultant}. These methods are especially favored when retraining is infeasible, such as in large-scale systems like recommender systems \citep{nandy2022achieving}.

Given a potentially biased scoring function, this paper addresses the challenge of constructing a fair probabilistic binary classifier with a binary-protected attribute. The goal is to ensure fairness without retraining the MLM, minimizing performance loss.

\noindent\textbf{Fairness and Performance Trade-offs} For classification, one of the desired characteristics of an MLM is \emph{calibration}~\cite{kleinberg2016inherent}. Suppose a classifier predicts that a given input is accepted ($Y=1$) with probability $p$, then calibration demands that the fraction of the accepted population, with the same features, is $p$. \cite{kleinberg2016inherent,chouldechova2017fair} have shown that calibration and equalized odds cannot be satisfied simultaneously except for highly constrained cases. Hence, researchers have been focusing on building classifiers (MLMs) with an appropriate approximate version of fairness~\cite{madras2018learning}. 
When it comes to practitioners, they focus on \emph{Receiver Operator Characteristics} (ROC) for evaluating a classifier as it best describes the classifiers. ROC measures the relative scores of the positive versus negative instances. The area under ROC-curve (AUC) is an appropriate performance metric to measure the predictive quality of such classifiers and to segregate positive and negative samples through ranking (\cite{huang2005using,clemenccon2008ranking,zehlike2021fairness}). AUC is particularly beneficial when the classifier is expected to segregate positive and negative labels, and the predictions must be fair across all threshold scores.

To make the practitioner's job effortless, we introduce a novel fairness measure, namely $\varepsilon_p$\ourdef\ -- no matter what threshold it uses for classification, the classifier is approximately fair, i.e., for all possible thresholds, the distance between the corresponding points of the ROC curves for both the protected group should be withing $\varepsilon$ distance in the $\mathcal{L}_p$ norm.
We aim to build a new probabilistic classifier that satisfies $\varepsilon_1$\ourdef\ with the minimal loss in AUC w.r.t. to the scoring function $s$.

\noindent\textbf{Our Approach: }
We assume query access to the ROC of $s$. First, we make sufficiently large $k$ queries to the ROC for the protected groups and make a piece-wise linear approximation of the ROC curves of both the protected groups. Next, we transport ROCs within $\varepsilon$ distance of each other to minimize the loss in AUC of the resultant ROC. We can achieve such transportation by randomizing scores across certain feasible classifiers for the given ROC curve. We call the space of these classifiers as \emph{ROC Space} of $s$. The resultant classifier from such randomization across the ROC Space is a convex combination of these classifiers. In a nutshell, we \emph{transform} the given $s$ to a fair scoring function by such ROC transport. We refer to this procedure of ROC transport as \ouralgo. We then geometrically prove that under certain conditions, \ouralgo\ is \emph{optimal}.\\
\noindent\textbf{Our Contributions: }
\begin{compactitem}
    \item We introduce a novel group fairness notion $\varepsilon_p$\ourdef, enforcing fairness over all thresholds in a score-based classification, {which is extremely useful for practitioners.}
    \item Next, we model a post-processing problem as a problem of finding an optimal transformation $\mathcal{H}$ on a given scoring function $s$ to minimize the performance loss due to transformation while ensuring $\varepsilon_1$\ourdef.  
    \item To achieve $\varepsilon_1$\ourdef, we propose a ROC transport, \ouralgo, a \emph{post-processing} algorithm  (Algorithm~\ref{alg:fairroc}). {Thus, it avoids re-training the existing MLM, which might not be fair. It also helps in explaining the decisions.}
    \item We perform rigorous theoretical analysis. We prove that (under some conditions) \ouralgo\ is optimal in terms of AUC loss. (Theorem~4.2).
    \item Finally, we demonstrate the efficacy of \ouralgo\ via experiments.
\end{compactitem}

\subsection{Related Work}
\textbf{Fairness in Binary Classification and Ranking}
\emph{Demographic Parity} (DP), \emph{Disparate Impact} (DI), and \emph{Equalized Odds} (EO) are widely studied group fairness notions. DP \citep{dwork2012fairness} and DI \citep{feldman15} ensure that the fraction of positive outcomes is identical across all sensitive groups. \cite{barocas16} introduced the $80\%$ rule, requiring that the positive outcome rate for a minority group must be at least $4/5$ of that for the majority group. EO \citep{hardt16} ensures similar distributions of error rates, specifically false positives and false negatives \citep{verma18}. Techniques to achieve fair MLMs include those discussed by \cite{padala2020fnnc}.
Group fairness has been shown to be inadequate for score-based classifiers, which classify across all thresholds \citep{gorantla21}. Consequently, researchers have proposed fairness notions based on the area under the curve (AUC). Examples include \emph{intra-group pairwise} AUC fairness \citep{beutel19}, \emph{BNSP} \citep{borkan19}, and \emph{inter-group pairwise} AUC (xAUC) fairness \citep{kallus19}. \cite{yang2023minimax} present a minimax learning and bias mitigation framework that integrates intra-group and inter-group AUC metrics to address algorithmic bias.
\cite{vogel2021} examine fairness in ranking problems, developing a general class of AUC-based fairness notions. They demonstrate that AUC-based fairness notions do not capture all forms of bias, as AUC summarizes classifier performance. They propose a stronger notion called point wise ROC-based fairness and design an in-processing algorithm for this purpose.

Our fairness definition ($\varepsilon_p$\ourdef) is inspired by equalized odds for all thresholds in ranking-based classification and is suitable for post-processing algorithms. It generalizes the approach of \cite{chen2020towards}, which uses the Manhattan distance as its norm. We later demonstrate the equivalency of both fairness notions {(ours $\varepsilon_1$). Note that the notion in \cite{chen2020towards} is not motivated by the same error rates at all thresholds,  and also, ours is more of a geometric approach from ROC curves, and theirs is an algebraic approach; ours is more general.}

\noindent \textbf{Post-processing for fair classification}
Post-processing techniques range from simple adjustments, such as thresholding or re-scaling, to complex methods like re-weighting or re-sampling. \cite{hardt16} argue that many existing fairness criteria are too restrictive, leading to sub-optimal solutions. They propose a fairness notion allowing some variation in prediction outcomes, defined by ``equality of opportunity'' constraints, ensuring the classifier is unbiased regarding the sensitive attribute. Their approach involves adjusting prediction thresholds for different groups based on their base rates to equalize false positive and false negative rates across groups. However, it does not involve \textit{transporting} ROC curves.
\cite{wei20} examine post-processing from the perspective of transformers, defining fairness as the expectation of scores and bounding the differences between true positive rates (TPRs) and false positive rates (FPRs) across protected groups. \cite{cui2021} propose a model-agnostic post-processing framework for balancing fairness in bipartite ranking scenarios.
\cite{zhao2024fair} introduces a novel approach using Wasserstein barycenters to quantify and address the cost of fairness, demonstrating that the complexity of learning an optimal fair predictor is comparable to learning the Bayes predictor.
\cite{tifreafrappe} propose a framework that transforms any regularized in-processing method into a post-processing approach, extending its applicability across a broader range of problem settings. 
\cite{cruz2023unprocessing} identifies two key methodological errors in prior work through empirical analysis: comparing methods with different unconstrained base models and differing levels of constraint relaxation.
\cite{jang22} introduce a method to optimize multiple fairness constraints through group-aware threshold adaptation, learning classification thresholds for each demographic group by optimizing the confusion matrix estimated from the model's probability distribution. Unlike \cite{jang22}, our approach starts with the fairness notion that differences between TPRs and FPRs of different groups must be bounded.
\cite{mishler2021fairness} use the bounded difference of counterfactual TPRs and FPRs as their fairness criterion, which differs from our $\varepsilon_p$\ourdef\ definition. Our $\varepsilon_p$\ourdef\ focuses on the bounded difference between TPRs and FPRs of different groups as the fairness criterion.

\section{Preliminaries}
\label{sec:prelim}

Consider a practitioner interested in binary classification, each data point having a binary-protected attribute. He/she is equipped with a scoring-based classifier trained on dataset $D=\{(x_i,a_i,y_i)_{i\in 1:n}\}$. Here, for $i$th data sample, $x_i \in \mathcal{X}\subset \mathbb{R}^d$ denotes features, $y_i\in\{0,1\}$ denotes the binary label, and $a_i\in\mathcal{A}=\{0,1\}$ denotes its binary protected attribute. We consider all these three as drawn from random variables $X, A, Y$, respectively.
There could be two scenarios - when the protected attribute is included or excluded from training (\cite{wei20})—our post-processing works for both cases as long as protected attributes are accessible during post-processing.

The random variables $X, A, Y$ are jointly distributed according to an unknown probability distribution over $(x_i, a_i, y_i)$. The cumulative conditional distributions on ${X\mid( Y=1)}$  and $X\mid (Y=0)$ are denoted by $G, H$, respectively. $ G^a, H^a$ are the corresponding distributions conditioned on $A=a$ (i.e. $G^a$ denotes the distribution of $X\mid (Y =1 , A = a)$) 

\subsection{Probabilistic Binary Classification}

Probabilistic Binary Classifier is equipped with a scoring function $s:\mathcal{X} \times \mathcal{A} \rightarrow \mathbb{R}$ mapping the feature space to a score.  
A deterministic classifier returns $s(X) \in \{0,1\}$ and a randomized one returns $s(X) \in [0, 1]$.
The higher the score $s(x)$, the higher the chance of the corresponding label $y = 1$.
The model prediction $\widehat{Y}$ , based on certain threshold $t\in [0,1]$, is given by $\widehat{Y} = \mathbb{I}(s(X) \geq t)$. $\mathcal{S}$ denotes the space of such scoring functions. 

The practitioner decides the threshold $t$ depending on the corresponding true positive rate (\emph{TPR}) and false positive rate (\emph{FPR}) (\cite{provost00,zhou05}). For deciding $t$, he is supplied with ROC -- \emph{receiver operator characteristic curve} for $s$. The ROC depicts the relation between TPR ($G_s(t)$) and FPR ($H_s(t)$) for $s$ at all possible thresholds $t$.  Note that, 
$G_s(t) \triangleq \mathbb{P}(s(X) \geq t\mid Y = 1)$ and $ H_s(t) \triangleq \mathbb{P}(s(X) \geq t\mid Y = 0)$.

\subsection*{ROC Curve and AUC}
The plot of a ROC-curve 
(Definition \eqref{def:roc}) 
is used to visualize homogeneity between two cumulative distributions~\cite{vogel2021}. 
The ROC curve is defined as:
\begin{definition}[ROC-Curve] \label{def:roc}
For any two cumulative distributions $g_1, g_2$ defined over the set $\mathbb{R}$, the ROC-curve is defined as the plot of 
$
    ROC_{g_1,g_2}(\alpha) \triangleq 1 - g_1 \circ g_2^{-1}(1-\alpha)
$
with domain $\alpha \in [0,1]$.
\end{definition}

The area under ROC-curve, \emph{AUC}, represents a summary of point-wise dissimilarity between the concerned distributions. Formally, let $S, S'$ be two independent random variables distributed according to $g_1, g_2$ respectively, then $AUC_{g_1,g_2} = \mathbb{P}(S'>S) + \frac{1}{2} \mathbb{P}(S' = S)$.

For a given scoring function $s$, we get two RVs, $G_s$ and $H_s$, by varying decision thresholds. We call the corresponding ROC curve $\roc$.
The area under $\roc$, i.e., $\auc =AUC_{H_s, G_s}$, is used to measure the ranking performance of a score function $s(.)$ (\cite{cortes03}; \cite{clemenccon2008ranking}). 
For a perfect classifier, $\auc = 1$, but such a classifier does not exist. Therefore, the optimal scoring function $s^{*}$ maximizes the $\auc$ amongst a certain subset of $\mathcal{S'}\subset \mathcal{S}$. Formally, $ s^* \in \argmax_{s\in \mathcal{S'}} \auc$.
In section ~\ref{ssec:randomclass}, we illustrate how a sub-optimal score function with lower TPRs can be achieved by randomizing outputs of $s(\cdot)$. This process is crucial in ensuring fairness. Let $\mathcal{S}|_s$ be the space of possible scoring functions through such randomization. We call it ROC-space of $s$.
Before designing our fair classifier, we formally define our notion of fairness in the next section.
\subsection{Fairness in Classification}

The typical group fairness notions in binary classifiers such as \emph{Demographic Parity }(DP) and \emph{Equalized Odds} (EO) are defined on deterministic predictions, i.e., in score-based classification, they work with a single threshold on scoring function $s$. Let $t^*$ be the threshold set by the practitioner. The resultant classifier is said to satisfy DP if
$
    G_s^0(t^*) + H_s^0(t^*) = G_s^1(t^*) + H_s^1(t^*)
$.
It satisfies the equivalence of \emph{acceptance rates} across groups.
Similarly, EO enforces equality of positive and negative error rates across protected groups, $1 - G_s^0(t^*) = 1 - G_s^1(t^*)$ and $H_s^0(t^*) = H_s^1(t^*)$.

\subsubsection*{$\varepsilon_p$\ourdef}
As discussed earlier, all group fairness notions are characterized by equality of a particular statistic across both the protected groups.  
In scoring-based probabilistic classifiers, these fairness notions depend on the selected threshold.
To achieve fairness across all thresholds, the practitioner can choose to retrain the model and achieve the right trade-offs between TPR and FNR.
However, retraining is expensive. Therefore, a desirable solution is 
To offer fair treatment to both protected groups using the pre-trained classifier.
However, this leads to invoking the post-processing technique every time the practitioner needs to update the threshold $t^*$. 
Instead, we propose a novel fairness measure to simplify the practitioner's job.
We perform post-processing on the given classifier once, and it ensures that no matter what threshold $t^*$ they choose to make decisions, the classifier offers similar treatment to both the protected groups. That is, the individual ROCs (Here on, we shall denote the ROCs of the protected groups, i.e., $ROC_{H_s^{0}, G_s^{0}}$ and $ROC_{H_s^{1}, G_s^{1}}$ by $\roc^0$ and $\roc^1$ respectively) should be within $\varepsilon$ distance ($\mathcal{L}_p$ norm) of each other.   We call it \emph{$\varepsilon_p$\ourdef}. More formally, 
\begin{definition}[$\varepsilon_p$\ourdef] \label{def:eroc}
A scoring function for binary classification $s$ with label prediction $\widehat{Y} = \mathbb{I}(s(x) \geq t)$ is said to satisfy \ourdef\ if for all $\alpha \in (0,1)$ the following holds:
\begin{equation}
    {\mid \mid \roc^1(\alpha) - \roc^0(\alpha) \mid \mid}_p \leq \varepsilon
\end{equation}
\end{definition}

\noindent In $\varepsilon_p$\ourdef, we utilize standard metrics (i.e. $\mathcal{L}_p$ norms) as the fairness statistic to quantify fairness. Thus, $\varepsilon_p$\ourdef\ is feasible for post-processing algorithms.
Next, we formulate the problem of fair post-processing. Note: $\varepsilon_1$\ourdef\ is a generalization of Equalized Odds to all the given thresholds of the scoring function. The proofs and detailed discussion are in Appendix B.

\subsection{Problem Formulation}

Given $s\in \mathcal{S}$, we would like to find $h\in\mathcal{S}|_s = \mathcal{H}(s)$ -- a transformation of a given scoring function such that $h$ satisfies $\varepsilon_1$\ourdef. Additionally, we want the loss in AUC due to transformation $\mathcal{H}$ minimal. That is, $\mathcal{L}_{F} = \auc - \texttt{AUC}_h$ must be minimal to retain the maximum performance guarantee of $s$. Thus, our goal is to get transformation $\mathcal{H}$ that solves the following optimization problem and returns the optimal transformed score $h^*$:

\begin{equation} \label{eq:fpp}
    h^* \in \argmax_{h \in \mathcal{S}|_s} \; \texttt{AUC}_h 
\end{equation}
    \[\text{s.t. } \mid \mid {\texttt{ROC}_{h}}^0(\alpha) - {\texttt{ROC}_{h}}^1(\alpha)\mid \mid_1 \leq \varepsilon , \;
    \forall \alpha \in [0,1]
\]

\section{Our Approach}
\label{sec:froc}
First, we explain query access to $\roc$ to sample from the desired statistic at various thresholds and its piece-wise linear approximation in Section~\ref{ssec:query} and Section~\ref{ssec:pla}, respectively. Since we cannot sample a continuum of thresholds, our $\roc$ will be discrete. In Section~\ref{ssec:ouralgo}, we describe the transport of ROCs.
Finally, we summarize our transformation as \ouralgo\ in Section~\ref{ssec:randomclass}.

\subsection{Query Model} \label{ssec:query}
Let $\mathcal{T} = \{ t_1, \ldots t_k \}$ be the set of thresholds at which we sample $\roc$ for each sensitive group ($t_i=\frac{i}{k}$).
Let $\mathcal{Q}^a(t_i)$ denote the query output at threshold $t_i$ for sensitive group $A = a$ on the $\roc^a$.
$
    \mathcal{Q}^a(t_i) \triangleq ROC_{H_s^{a}, G_s^{a}}(t_i)
$.

Abusing notations, we use $\mathcal{Q}^a(t_i)$ and $\mathcal{Q}^a_i$ interchangeably. Let $\mathcal{Q}^a = (\mathcal{Q}^a_1, \ldots, \mathcal{Q}^a_k)$ be the sequence of all query outputs for group $a$. In the next section, we construct the piece-wise linear approximation of the group-wise ROC curves using the group-wise query outputs $\mathcal{Q}^a$.

\subsection{Piece-wise Linear Approximation (PLA) of ROC-curves} \label{ssec:pla}
To obtain the piece-wise linear approximation (PLA), we sample $k$ points from ROC and construct a straight line from $\mathcal{Q}_i^a$ to $\mathcal{Q}_{i+1}^a$ for all $i = 1 \ldots k-1$. Lastly, we join $(0,0)$ to $\mathcal{Q}_1^a$ (see Figure \ref{fig:losslpa}). Following these steps on the query sets $\mathcal{Q}^a$ will generate the PLAs for protected groups $a \in \{0,1\}$.
We denote by $\widehat{G_s^{a}}, \widehat{H_s^{a}}$, the cumulative distributions induced by the linear approximation of the ROC-curve on $s$.

Due to PLA, we incur a loss $\mathcal{L}_{LPA}$ in $AUC_{H_s, G_s}$(shaded region in Figure \eqref{fig:losslpa}). $\mathcal{L}_{LPA}$ is inversely proportional to the number of queries $k$, see Section \ref{ssec:lossplabound} for bounds on this loss. Hence, we shall ignore this loss in our fairness analysis as it can be brought  arbitrarily close to $0$ by increasing $k$.

\begin{figure*}
    \centering
    \hspace{-1.5cm}
    \begin{minipage}{0.3\linewidth}
         \centering
         \includegraphics[scale = 0.07]{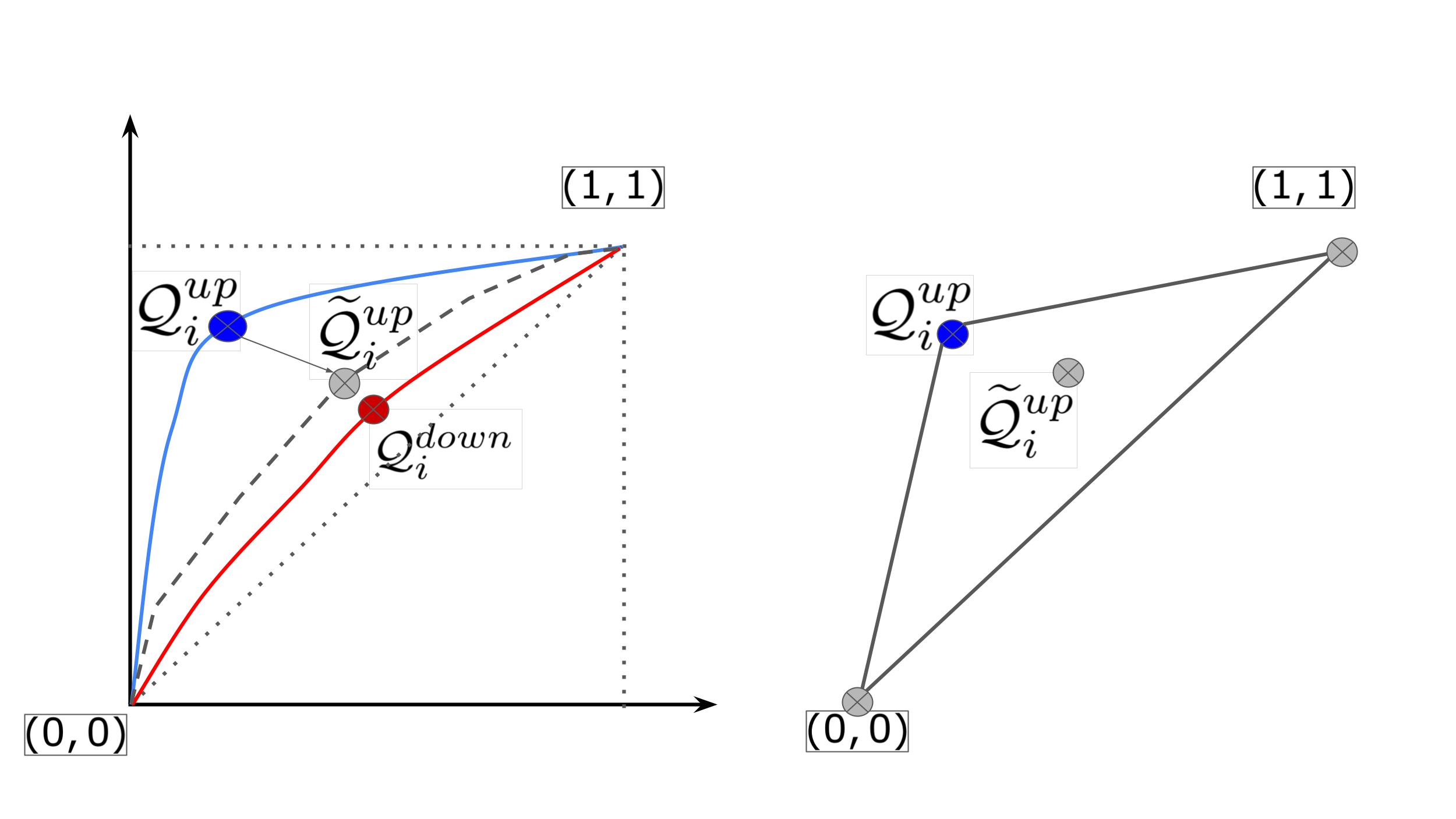}
         \caption{ROCs and convex hull}
         \label{fig:conv}
     \end{minipage}
    \hspace{1.5cm}
     \begin{minipage}{0.3\linewidth}
         \centering
         \includegraphics[scale = 0.05]{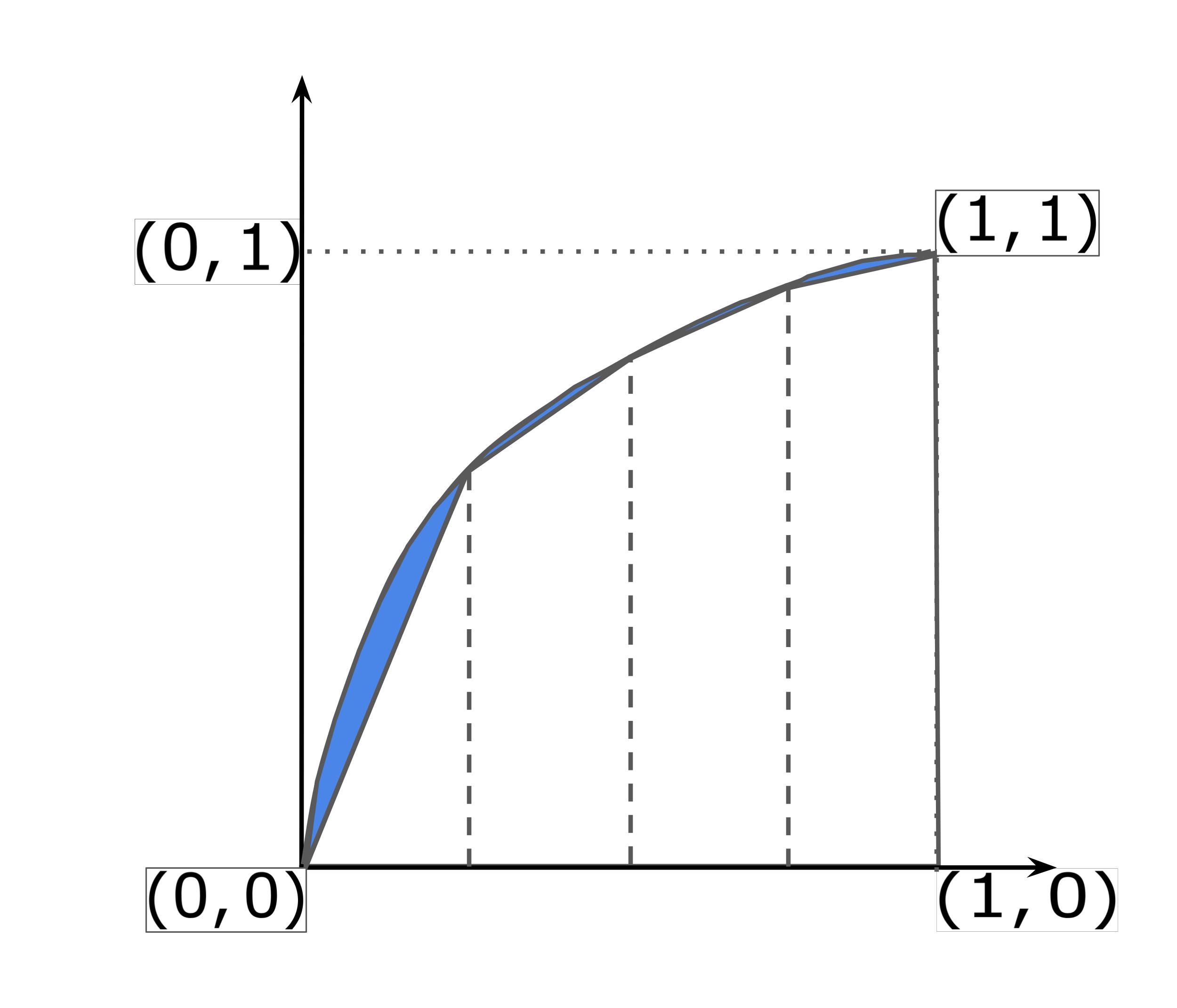}
         \caption{Shaded Area indicates $\mathcal{L}_{PLA}$}
         \label{fig:losslpa}
     \end{minipage}
     \begin{minipage}{0.3\linewidth}
         \centering
         \includegraphics[scale = 0.07]{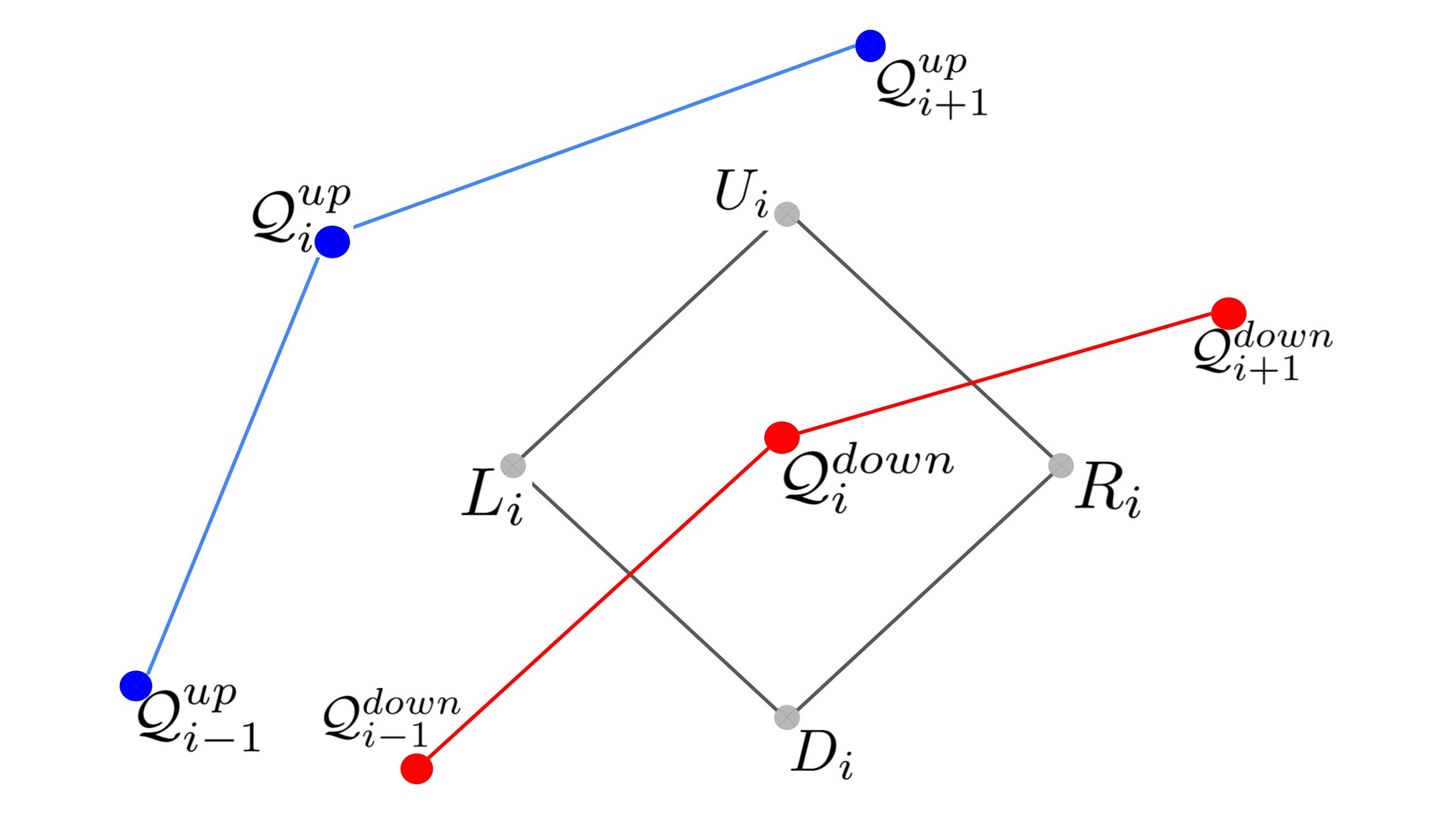}
         \caption{Norm Boundary}
         \label{fig:five over x}
     \end{minipage}
        \label{fig:three graphs}
\end{figure*}

\subsection{Transporting ROCs for \texorpdfstring{$\varepsilon_1$\ourdef}{}} \label{ssec:ouralgo}

Since we are using post-processing technique to ensure fairness, it is impossible to shift any \texttt{ROC} above its current position, i.e., build a classifier corresponding to any point in the epigraph (the points above the ROC curve) of $\roc$ just with the help of $s$. Interestingly, a classifier representing a point in the hypograph (points below the curve) of $s$ $\cap \mathcal{S}$ can be obtained through randomization on the predicted scores (see Chapter 3 in~\cite{fairmlbook19}). The key idea involves abstracting out the convex hull (Fig \ref{fig:conv}) formed by the three points $(0,0)$, $(1,1)$ and $\mathcal{Q}_i^{up}$, and sampling outcomes from classifiers representing $(0,0)$, $(1,1)$\footnote{Note that $(0,0)$ and $(1,1)$ represent `always reject' and `always accept' classifiers.}  and $\mathcal{Q}_i^{up}$ with specific probabilities. By taking convex combinations of the three aforementioned points in the ROC space, we can represent any point lying in their convex hull. The exact convex combinations are described in \textbf{C2}. We leverage this property to achieve $\varepsilon_1$\ourdef. We denote this space as \emph{ROC-space of $s$} -- $\mathcal{S}|_s$. Each point in $\mathcal{S}|_s$ represents a binary classifier in terms of its performance at a certain threshold $t$. Each point is of the form $(FPR(t), TPR(t))$. 

In the realm of binary classification, it is a common occurrence for one group to be subject to discrimination. Specifically, if we plot $\roc^0,\roc^1$, we will find that one of the ROCs is notably situated below the other.
For this study, the ROC predominantly above the other will be designated as $ROC_{up}$, while the other ROC will be referred to as $ROC_{down}$. We believe this is a reasonable assumption because we observed that in most classifiers (for which present the results and others we explored on the datasets mentioned in Section~E3) the ROCs don't intersect or intersect at regions where $FPR \le 0.2$ or $TPR \ge 0.5$. Typically, no practitioner will work in those areas of ROCs. We leave for future work to address intersecting ROCs.

Let $\mathcal{Q}^{up}, \mathcal{Q}^{down}$ be the corresponding set of query points for $\texttt{ROC}_{up}, \texttt{ROC}_{down}$ respectively. We also denote their fair counterparts by $\widetilde{\mathcal{Q}}^{up}, \widetilde{\mathcal{Q}}^{down}$.

\subsubsection{Algorithm Definitions}
We need to transport $\texttt{ROC}_{up}$ towards $\texttt{ROC}_{down}$ such that the new ROCs are within $\varepsilon$ distance of each other. Our approach is geometric. We need to identify certain points/curves in the epigraph of $\texttt{ROC}_{down}$ as follows.
\begin{definition}[Norm Boundary] \label{def:eNBoundary}
The set of all points within $\varepsilon$ distance ($\ell_1$ norm) from $\mathcal{Q}_i^{down}$ is known as the \emph{norm set} $\mathfrak{C}_i$. Formally, we have \[\mathfrak{C}_i \triangleq \{ x: x\in [0,1]^2 , ||x - \mathcal{Q}_i^{down}||_1 \le \varepsilon\}\]
The set of all points exactly $\varepsilon$ distance (in $\mathcal{L}_1$ norm) from $\mathcal{Q}_i^a$ is known as \emph{Norm Boundary} $\mathfrak{B}_i$. Formally, 
\[\mathfrak{B}_i \triangleq \{ x: x\in [0,1]^2 , ||x - \mathcal{Q}_i^{down}||_1 = \varepsilon\}\]
Additionally, we denote the vertices of the Norm Boundary Rhombus (starting from the top most point and moving clockwise) as $U_i$, $R_i$, $D_i$, and $L_i$.
\end{definition}

We say that an index $i \in [1 ,2 , \hdots, k]$ is a Boundary Cut index when $ROC_{up}$ intersects the Norm Boundary $\mathfrak{B}_i$. Formally,
\begin{definition}[Boundary Cut] \label{def:BCut}
 Index $i \in [1 ,2 , \hdots, k]$ is a \emph{Boundary Cut index} when $\mathfrak{B}_i \cap ROC_{up} \neq \phi$.
\end{definition}

We now define the three kinds of shifts that will be used in our Algorithm:
For a given $ i \in [1 , 2, \hdots ,k]$, Upshift is the transportation of $\mathcal{Q}_i^{up}$ to the point $U_i$. 
\begin{definition}[UpShift] \label{def:ush}
    For a given $ i \in [1 , 2, \hdots ,k]$, Upshift is the transportation of $\mathcal{Q}_i^{up}$ to the point $U_i$. Formally, \emph{UpShift} can be defined as the function that returns a fair threshold $\widetilde{\mathcal{Q}}_i^{up}$ (i.e. $U_i$) by taking the $\mathcal{Q}_i^{down}$ and $\varepsilon$ as the arguments.
\end{definition}

For a given $ i \in [1 , 2, \hdots ,k]$, Leftshift is the transportation of $\mathcal{Q}_i^{up}$ to the point $L_i$. Formally,
\begin{definition}[LeftShift] \label{def:lsh}
LeftShift is a function that returns a fair threshold $\widetilde{\mathcal{Q}}_i^{up}$ (i.e. $L_i$) by taking the $\mathcal{Q}_i^{down}$ and $\varepsilon$ as the arguments.
\end{definition}

\begin{definition}[CutShift] \label{def:csh}
    For a given $i \in [1 , 2, \hdots ,k]$ (representing the index of the $ROC_{down}$), we run through all the points of the $ROC_{up}$ and return the set of all points that intersect the Norm Boundary $\mathfrak{B_i}$. Formally, we define \emph{Cutshift} as a function that takes $\mathcal{Q}_i^{down}$ and $\varepsilon$ as the arguments and returns $ROC_{up} \cap \mathfrak{B}_i$. The set $ROC_{up} \cap \mathfrak{B}_i$ can be represented as $\{p_{left} , p_{right}\}$ denoting the points at the intersection of $ROC_{up}$ at the \textbf{left-side} of the Norm Boundary and the \textbf{right-side} of the Norm Boundary respectively.
\end{definition}

Now, we elaborate on the above procedure to transport points from $ROC_{up}$ towards $ROC_{down}$.

\subsubsection{Algorithm for ROC Transport}
We provide a geometric algorithm that returns a classifier equivalent to the scoring function $h^*$ in $\mathcal{S}|_s$.

\begin{algorithm}
\caption{\textsc{FairROC Algorithm}}
\label{alg:fairroc}
\begin{algorithmic}[1]
\Require $ROC_{up}$, $ROC_{down}$, $\varepsilon$
\Ensure $FairROC_{up}$, $FairROC_{down}$
\State Initialize $i \gets 1$, $k \gets \text{length}(ROC_{up})$
\State $FairROC_{up} \gets \emptyset$, $FairROC_{down} \gets ROC_{down}$

\While{$i < k-1$}
    \State $i \gets i+1$
    \If{\Call{BoundaryCut}{$i, \varepsilon$} == TRUE}
        \State $p_{left}, p_{right} \gets \Call{CutShift}{i, ROC_{up}, ROC_{down}}$
        \If{$FPR(\mathcal{Q}_i^{up}) \geq FPR(\mathcal{Q}_i^{down})$}
            \State $\widetilde{\mathcal{Q}}_i^{up} \gets p_{right}$
        \Else
            \State $\widetilde{\mathcal{Q}}_i^{up} \gets p_{left}$
        \EndIf
    \ElsIf{$\mathcal{Q}_i^{up} \in \Call{Hypograph}{ROC_{down}}$}
        \State $\widetilde{\mathcal{Q}}_i^{up} \gets \mathcal{Q}_i^{up}$
        \State \textbf{continue}
    \Else
        \If{$\text{Area}(\square \mathcal{Q}_{i+1}^{up} \mathcal{Q}_{i}^{up} \mathcal{Q}_{i-1}^{up} {L}_i) \geq \text{Area}(\square \mathcal{Q}_{i+1}^{up} \mathcal{Q}_{i}^{up} \mathcal{Q}_{i-1}^{up} {U}_i)$}
            \State $\widetilde{\mathcal{Q}}_i^{up} \gets U_i$
        \Else
            \State $\widetilde{\mathcal{Q}}_i^{up} \gets L_i$
        \EndIf
    \EndIf
    \State $FairROC_{up} \gets \Call{Append}{\widetilde{\mathcal{Q}}_i^{up}}$
\EndWhile

\end{algorithmic}
\end{algorithm}

Note that, Algorithm \ref{alg:fairroc} treats $ROC_{down}$ as \emph{implicitly} fair.
Also, by $Area(\square ABCD)$, we denote the area of the quadrilateral whose vertices are $A,B,C$, and $D$. This area is easily found in this context by splitting $\square ABCD$ into two disjoint triangles- $\Delta ABC$ and $\Delta ACD$ and using the Herons formula \cite{kendig20002000} on each triangle. 

For example, consider $Area(\Delta \mathcal{Q}_i^{up} \mathcal{Q}_{i-1}^{up} L_i)$. Let $a = ||\mathcal{Q}_i^{up} \mathcal{Q}_{i-1}^{up}||_2$, $b = ||\mathcal{Q}_i^{up} L_i||_2$ and $c = ||\mathcal{Q}_{i-1}^{up} L_i||_2$. Additionally, we define $s = \frac{a+b+c}{2}$. Then, it is true that:
\[Area(\Delta \mathcal{Q}_i^{up} \mathcal{Q}_{i-1}^{up} L_i ) = \sqrt{s(s-a)(s-b)(s-c)}\]

\subsection{Obtaining fair classifier from the updated ROCs} \label{ssec:randomclass}

The algorithm described in the previous subsection returns the fair ROC curves according to $\varepsilon_1$\ourdef. As a final step, we 
need to find the transformed classifier.  We call it \texttt{ConstructClassifier}($FairROC_{up}$,$FairROC_{down}$
,$\roc^0$,$\roc^1$) which returns a probabilistic binary classifier representing $h = \mathcal{H}(s)$ such that it represents the FairROCs. We construct one using the procedure explained in Section 3.3.

Now, we establish the optimality of our solution within specific assumptions.
\section{Theoretical Analysis}
As described in Section ~\eqref{ssec:pla}, we work with PLA of the ROC curves $ROC_{H_s^a, G_s^a}$, $a\in \{0,1\}$. This causes a \emph{loss} in area under ROC. We denote this loss by $\mathcal{L}_{PLA}$ and is quantified as the difference in AUCs of $ROC_{H_s^a, G_s^a}$ and $ROC_{\widehat{H_s^a}, \widehat{G_s^a}}$.

In Section \ref{ssec:ouralgo}, transporting the ROC query points, $\mathcal{Q}^{up}$ introduces a decrease of the area under the ROC curve due to the transformation of scoring function $s$ to $h$. We denote this loss by $\mathcal{L}_{AUC}$. 
This loss can be quantified as the difference in AUCs of $ROC_{\widehat{H_s^a}, \widehat{G_s^a}}$ and $ROC_{H_h^a, G_h^a}$
The total loss in AUC, $\mathcal{L}$, induced by \ouralgo\ is given by:
$\mathcal{L} = \mathcal{L}_{PLA} + \mathcal{L}_{AUC}$

\subsection{PLA Loss analysis} \label{ssec:lossplabound}

We start our analysis by making a few standard assumptions regarding the continuity and differentiability of the cumulative distributions on the family of scoring functions $\mathcal{S}$.
We adopt a less stringent assumption than that presented in \cite{vogel2021}, as we impose only an upper bound on the slopes. This contrasts with the approach in \cite{vogel2021}, which necessitates both an upper and lower bound on the slopes.
\begin{assumption} \label{assumption1}

We assume that the rate of change (with respect to the thresholds $t$) of the $TPR$s and $FPR$s are upper bounded. I.e. we assume that $\exists~ u_T , u_F \in \mathbb{R}$ such that $\frac{d~TPR}{dt}\le u_T$ and $\frac{d~FPR}{dt}\le u_F$. 
\end{assumption}


\begin{theorem} \label{th:pla}
Let $ROC_{\widehat{H_s^a}, \widehat{G_s^a}}$ be the \emph{PLA} of $ROC_{H_s^a,G_s^a}$ over the query set of $k$ equidistant thresholds, $\mathcal{T} = \{ t_i \mid t_i = i/k \ \forall i \in [k] \}$. The corresponding $\mathcal{L}_{PLA}$ is bounded as:
    $\mathcal{L}_{PLA} \le\frac{1}{2} \frac{u_Tu_F}{k}$
\end{theorem}

\subsection{AUC loss analysis}
We start our analysis by making a few assumptions regarding the spacing of the ROC thresholds and the ROC curve.
\begin{assumption} \label{assumption2}
We have two assumptions:
\begin{compactitem}
    \item $\forall i\in \{1 ,2 , \hdots , k\}$, we assume that $FPR(\mathcal{Q}_{i-1}^{down}) \le FPR(\mathcal{Q}_i^{up}) \le FPR(\mathcal{Q}_{i+1}^{down})$.
    \item We assume that the $ROC_{up}$ can intersect any Norm boundary (i.e. $(\mathfrak{B}_i)_{i\in \{1 ,2 , \hdots , k\}}$) at most 2 times.
\end{compactitem}
\end{assumption}
We note that even if \textbf{Assumption 4.2} does not hold, \ouralgo remains operational and continues to produce outputs that are \ourdef fair. However, under these conditions, the optimality with respect to AUC is not guaranteed, as \textbf{Theorem 4.4} no longer applies.

\begin{theorem}
\label{thm:}If a given classifier $s$ is piece-wise linear and satisfies assumption 4.2, the ROCs returned by \ouralgo\ represent the classifier solving optimization problem~\ref{eq:fpp}.
\end{theorem}

\subsection{Optimally fair points and Norm Boundary}
This section proves that all optimally fair points must lie on some Norm Boundary. We do this by establishing that the performance of any point in the Norm Set can be improved by appropriate transportation to a point on the Norm Boundary.

\begin{theorem}(Norm Boundary)
    If $(\widetilde{\mathcal{Q}}_i^{up})_{i\in \{1 ,2 , \hdots , k\}}$ is the set of optimal fair (points that maximize the AUC and also satisfy the $\varepsilon_1$ \ourdef) thresholds must necessarily be a subset of $(\mathfrak{B}_i)_{i\in \{1 ,2 , \hdots , k\}}$. 
\end{theorem}

\begin{theorem}(CutShift)
    If index $i$ is a Boundary cut point, then the CutShift operation must be performed. Of the 2 points ($p_{left}$ and $p_{right}$) returned by the Cutshift operation, the point that is closer to $\mathcal{Q}_{i}^{up}$ must be chosen i.e.$\mathcal{\widetilde{Q}}_i^{up} = argmin_{p \in \{ p_{left} , p_{right}\}} |FPR(\mathcal{Q}_i^{up}) - FPR(p)|$
\end{theorem}

\begin{theorem}(UpShift)
    If index $i$ is not a Boundary cut point and if $Area(\square \mathcal{Q}_{i+1}\mathcal{Q}_{i}\mathcal{Q}_{i-1}{L}_i \ge Area(\square \mathcal{Q}_{i+1}\mathcal{Q}_{i}\mathcal{Q}_{i-1}{U}_i$), then UpShift operation must be performed. The resulting point ($U_i$) is the new fair point $\mathcal{\widetilde{Q}}_i^{up}$. Otherwise, the LeftShift operation must be performed. The resulting point ($L_i$) is the new fair point $\mathcal{\widetilde{Q}}_i^{up}$.
\end{theorem}
The proofs of all the above theorems are given in the appendix.
However, the following is brief sketch of the proof:\\
\underline{Step 1:} We prove that all optimally fair points $(\widetilde{\mathcal{Q}}_{i}^{up})_{i \in \{ 1 ,2 , \hdots , k\}}$ must lie on the Norm Boundaries of the corresponding $\mathcal{Q}_i^{down}$. (i.e. $(\mathfrak{B}_i)_{i \in\{ 1 ,2 , \hdots , k\} }$)\\
\underline{Step 2:} We then prove that if $\mathfrak{B}_i \cap ROC_{up} \neq \phi $, then the CutShift transportation is the optimal transportation.\\
\underline{Step 3:} We then prove that if $\mathfrak{B}_i \cap ROC_{up} = \phi $, then, based on the Cover and aforementioned area condition, the UpShift or the LeftShift transportation is the optimal transportation.

In the next section, we experimentally analyze \ouralgo. 
\section{Empirical Analysis}
\subsection{Experimental Setup}
\noindent \textbf{Datasets}: 
We train different classifiers on the widely-used ADULT~\cite{misc_adult_2} and COMPAS~\cite{angwin2022machine} benchmark datasets, selecting MALE and FEMALE as protected groups in ADULT, and BLACK and OTHERS in COMPAS. ROCs are generated, with additional experiments on datasets like CelebA in Appendix E and F.

\noindent \textbf{Classifiers}: We test \ouralgo\ on ROCs from the following classifiers:
\footnote{We choose these classifiers as per the availability of experiment hyper-parameters from other in-processing and post-processing benchmarks.}.
    C1: \emph{FNNC}(~\cite{padala2020fnnc}): This is a neural network-based classifier with a target parameter for fairness.
    C2: \emph{Logistic Regression} and 
    C3: \emph{Random Forest}
We used the code from the author's GitHub for C1 and sklearn implementations for C2 and C3.

\begin{figure*}
    \centering
    \hspace{-2cm}
    \begin{minipage}{0.3\linewidth}
         \centering
         \includegraphics[scale = 0.3]{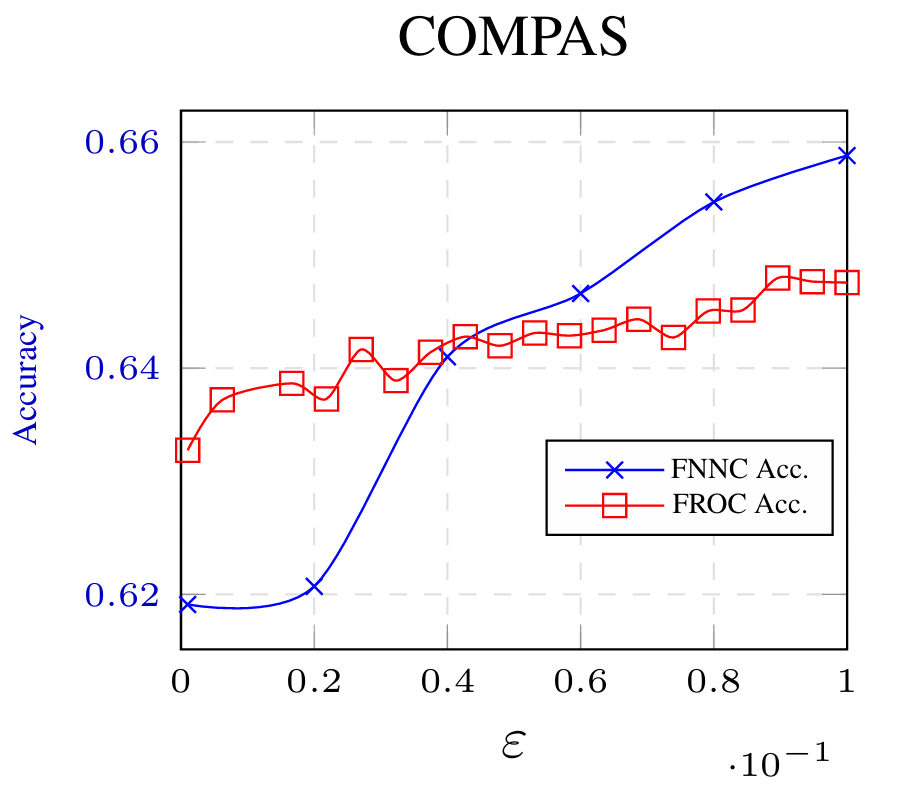}
         \caption{C1  vs. C1-\ouralgo}
        \label{fig:COMPAS_acc}
     \end{minipage}
     \begin{minipage}{0.3\linewidth}
         \centering
        \includegraphics[scale = 0.3]{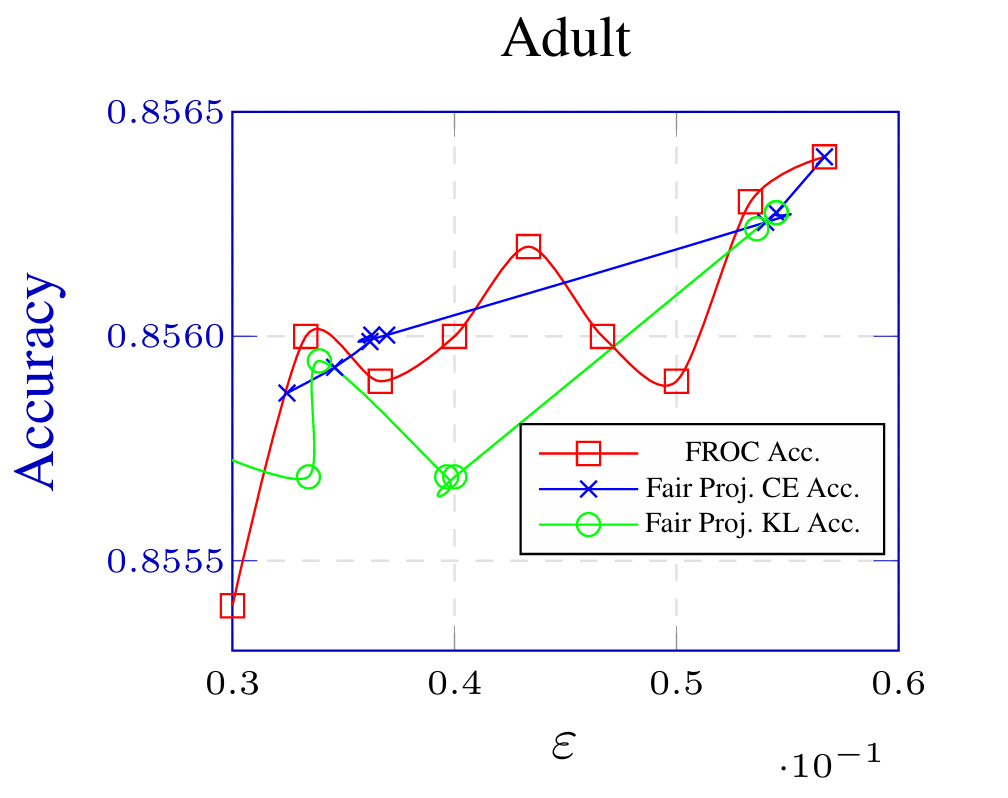}
    \caption{C3-Fair Fair vs. C3-\ouralgo}
    \label{fig:FairProj} 
     \end{minipage}
     \begin{minipage}{0.3\linewidth}
         \centering
         \includegraphics[scale = 0.3]{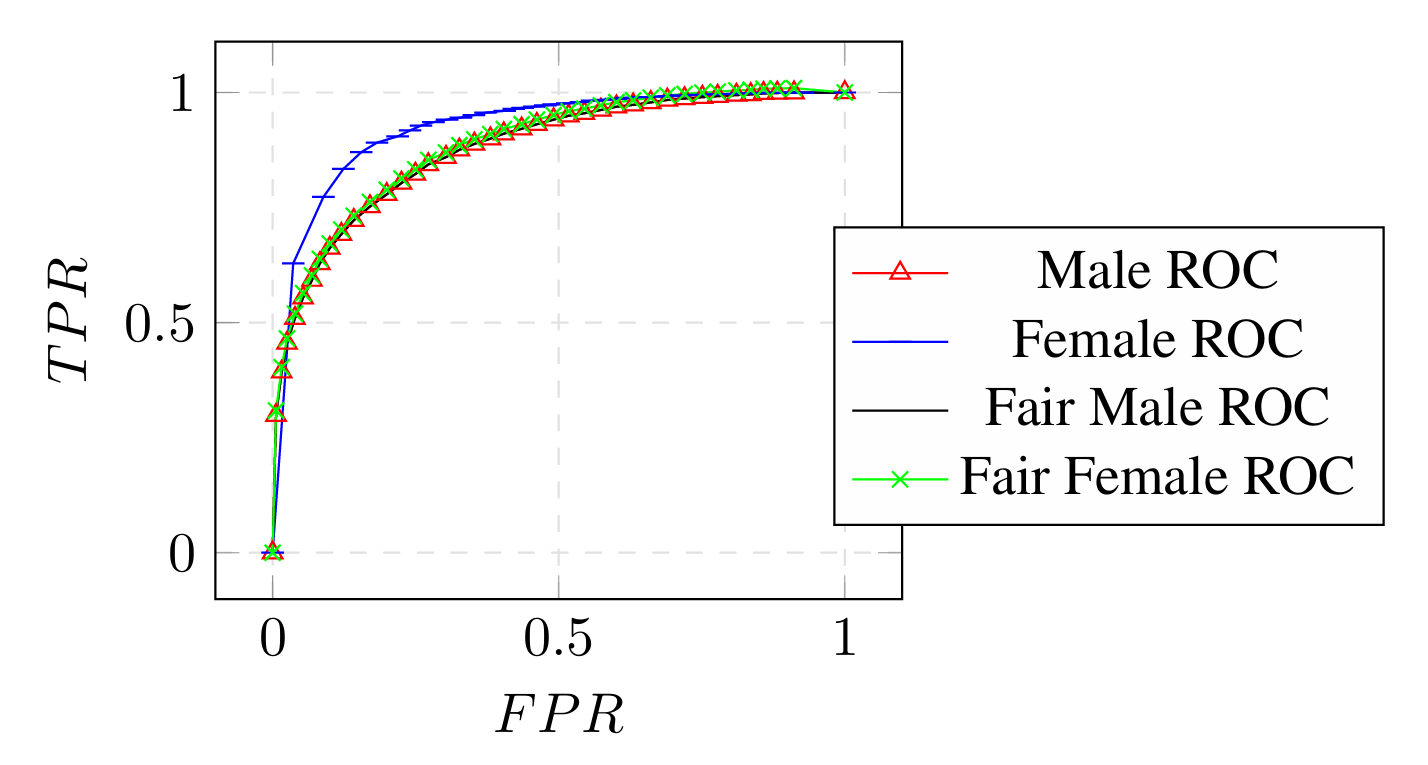}
   \caption{C2 Before and After \ouralgo}
    \label{fig:roc_froc}
     \end{minipage}
        \label{fig:three graphs exp}
\end{figure*}

\noindent \textbf{Post-Processing methods}: 
We compare \ouralgo\  against the following baselines: 
    B1: \emph{FairProjection-CE} and \emph{FairProjection-KL} ~\cite{alghamdi2022}: Transforms the score to achieve mean equalized odds fairness through information projection.

\subsection{Experiments} 

We train C1 on both datasets, C2 and C3 on the Adult dataset, and generate their ROCs for all the protected groups. FNNC, we train by ignoring its fairness components in the loss function and then generate ROC. We then invoke \ouralgo\ for different $\varepsilon$ values and check the best possible threshold for accuracy. We refer to the new classifier as C1-C3-\ouralgo.

\noindent \textbf{Baseline Post-Processing Method}: We evaluate \ouralgo, and the baselines B1 on ADULT dataset against the fairness metric \emph{mean equalized odds}(B2)~\cite{alghamdi2022}  in Figs.~\ref{fig:FairProj}.
For consistent comparison, we adopt the training parameters for base classifiers from ~\cite{alghamdi2022} and keep it identical across all experiments.

\subsection{Results}
We show the results on the COMPAS and Adult dataset (using FNNC and \ouralgo) here, along with a comparison with existing post-processing baselines.
The remaining experimental observations are detailed in the supplementary.
\textbf{Figure \ref{fig:roc_froc}} displays the ROC curves (Before and After \ouralgo) for both males and females, on the ADULT dataset for C2. The female ROC consistently occupies the higher position, indicating a positive bias for males. This establishes $ROC_{0}$ as our counterpart to $ROC_{down}$. Thus, we apply \ouralgo\ to the alternate curve, $ROC_{1}$, showcased in the figure. Before \ouralgo, the maximum difference between Male ROC and Female ROC is $0.08$. However, after post-processing with \ouralgo, the loss in accuracy is $<0.1\%$ for $\varepsilon=0.05$. 
In general, across all experiments (more experiments in Appendix), we observe a 7-8\% improvement in fairness, \ouralgo incurs at most a 2\% drop in accuracy. As seen in \textbf{Figure \ref{fig:COMPAS_acc}} and \textbf{Figure \ref{fig:FairProj}} for smaller values of $\varepsilon$, we also observe the performance may beat FNNC and the post-processing methods.
We assign it to the fact that FNNC (and the other methods) may overachieve the target fairness for smaller values of $\varepsilon$ (Evident from Table 2~\cite{padala2020fnnc}). \ouralgo\ drops AUC minimally to achieve target fairness.

\section{Conclusion}
In this work, we addressed the problem of practitioners aiming to achieve fair classification without retraining MLMs. Specifically, we provide a post-processing framework that takes a potentially unfair classification score function and returns a probabilistic fair classifier. 
The practitioner need not worry about fairness across different thresholds, so we proposed a new notion $\varepsilon_1$\ourdef\ (Definition \ref{def:eroc}), which ensures fairness for all thresholds. To achieve $\varepsilon_1$\ourdef, we proposed \ouralgo\ (Algorithm \ref{alg:fairroc}), which transports the ROC for each sensitive group within $\epsilon$ distance while minimizing the loss in AUC of the resultant ROC. 
We geometrically proved its optimality conditions (Theorem 4.2) and bounds under certain technical assumptions. We observed empirically that its performance might differ at most by 2\% compared to an in-processing technique while ensuring stronger fairness and avoiding retraining. We leave it for future work to explore the possibility of different distance metrics for fairness and optimizing for different performance measures.

\section*{Note}
The official code for this paper can be found in this \href{https://github.com/Avyukta-Manjunatha-Vummintala/FROC_code/tree/main}{link}.
\newpage
\section*{Appendix}

\appendix
\section{Notation Table}
\begin{table}[H]
  \centering
  \begin{tabular}{c p{0.35\textwidth}}
    \toprule
    \textbf{Notation} & \textbf{Description} \\
    \midrule
    $\varepsilon$ & Fairness measure of $\varepsilon_1$\ourdef\ and \ouralgo \\
    $ROC$ & Receiver Operator Characteristic (plot of FPR vs. TPR) \\
    $AUC$ & Area under ROC curve \\
    $s$ & Scoring function \\
    $k$ & Number of queries submitted to the scoring function \\
    $D$ & Dataset \\
    $x_i$ & Feature vector \\
    $\mathcal{X}$ & Sample space of feature vectors  \\
    $y_i$ & Binary label \\
    $a_i$ & Binary protected attribute \\
    $X$ & Random vector modeling feature vectors \\
    $Y$ & Random variable modeling labels \\
    $A$ & Random variable modeling protected attributes \\
    $\mathcal{S}$ & Space of scoring functions \\
    $\mathcal{S}|_s$ & Space of feasible scoring functions\\
    $G_s(t)$ & $\mathbb{P}(s(X) \ge t | Y = 1)$ \\
    $H_s(t)$ & $\mathbb{P}(s(X) \ge t | Y = 0)$ \\
    $G^a_s(t)$ & $\mathbb{P}(s(X) \ge t | Y = 1, A=a)$ \\
    $H^a_s(t)$ & $\mathbb{P}(s(X) \ge t | Y = 0, A=a)$ \\
    $\roc$& $\mathrm{ROC}_{H_s , G_s}$\\
    $\texttt{AUC}_s$ & AUC of $\roc$\\
    $\mathcal{Q}^a$ & Sequence of query point from Group $a$\\
    $\mathcal{Q}^a_i$ & Query point of Group $a$ at threshold $t_i$\\
    $\mathcal{L}_{LPA}$ & Loss due to Linear Piecewise Approximation\\
    $\mathfrak{C}_i$ & Norm Set of $i^{th}$ threshold\\
    $\mathfrak{B}_i$ & Norm Boundary of $i^{th}$ threshold\\
    \bottomrule
  \end{tabular}
  \caption{Mathematical Notations}
  \label{tab:notations}
\end{table}

\section{Relation to Equalized Odds}
Equalized Odds is defined in \cite{padala2020fnnc} and \cite{madras2018learning}, is the sum of the absolute differences of the $FNR$ and the $FPR$ of both the protected groups. Formally,
\[EO \triangleq |FPR_0 - FPR_1| + |FNR_0 - FNR_1|\]
However, this defintion is equivalent to $\varepsilon_1$\ourdef since $|FPR_0 - FPR_1| + |FNR_0 - FNR_1| = |FPR_0 - FPR_1| + |(1 - TPR_0) -(1 +  TPR_1)| =|FPR_0 - FPR_1| + |TPR_0 - TPR_1|$.

\section{Algorithm Description}

\subsection{\ouralgo}
\begin{definition}[Norm Boundary] \label{def:eNBoundary_appendix}
The set of all points within $\varepsilon$ distance ($\ell_1$ norm) from $\mathcal{Q}_i^{down}$ is known as the \emph{norm set} $\mathfrak{C}_i$. Formally, we have:
\[\mathfrak{C}_i \triangleq \{ x: x\in [0,1]^2 , ||x - \mathcal{Q}_i^{down}||_1 \le \varepsilon\}\]
The set of all points exactly $\varepsilon$ distance from $\mathcal{Q}_i^a$ is known as \emph{Norm Boundary} $\mathfrak{B}_i$. Formally, 
\[\mathfrak{B}_i \triangleq \{ x: x\in [0,1]^2 , ||x - \mathcal{Q}_i^{down}||_1 = \varepsilon\}\]
Additionally, we denote the vertices of the Norm Boundary Rhombus (starting from the top most point and moving clockwise) as $U_i$, $R_i$, $D_i$, and $L_i$.
\end{definition}

\begin{figure}[!ht]
    \centering
    \includegraphics[scale = 0.2]{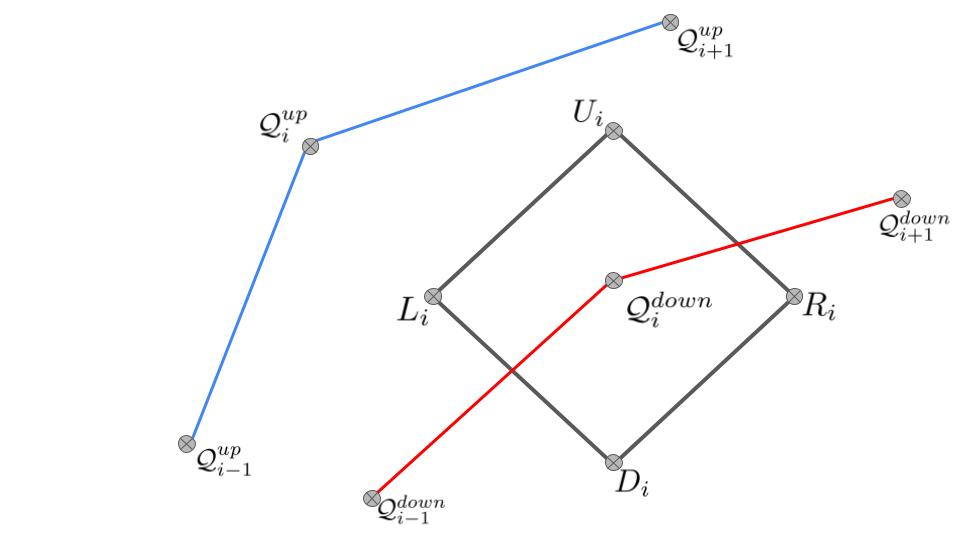}
    \caption{The area inside the rhombus is the Norm set $\mathfrak{C}_i$. The boundary (denoted by the thick bold border) is $\mathfrak{B}_i$. The topmost point of $\mathfrak{B}_i$ is denoted by $U_i$}
    \label{fig:NBoundary_appendix}
\end{figure}



We say that a $i \in [1 ,2 , \hdots, k]$ is a Boundary Cut point when $ROC_{up}$ intersects the Norm Boundary $\mathfrak{B}_i$. Formally,
\begin{definition}[Boundary Cut] \label{def:BCut_appendix}
 $i \in [1 ,2 , \hdots, k]$ is a \emph{Boundary Cut point} when $\mathfrak{B}_i \cap ROC_{up} \neq \phi$.
\end{definition}
This is illustrated in the \textbf{Figure 8}.
\begin{figure}[!ht]
    \centering
    \includegraphics[scale = 0.1]{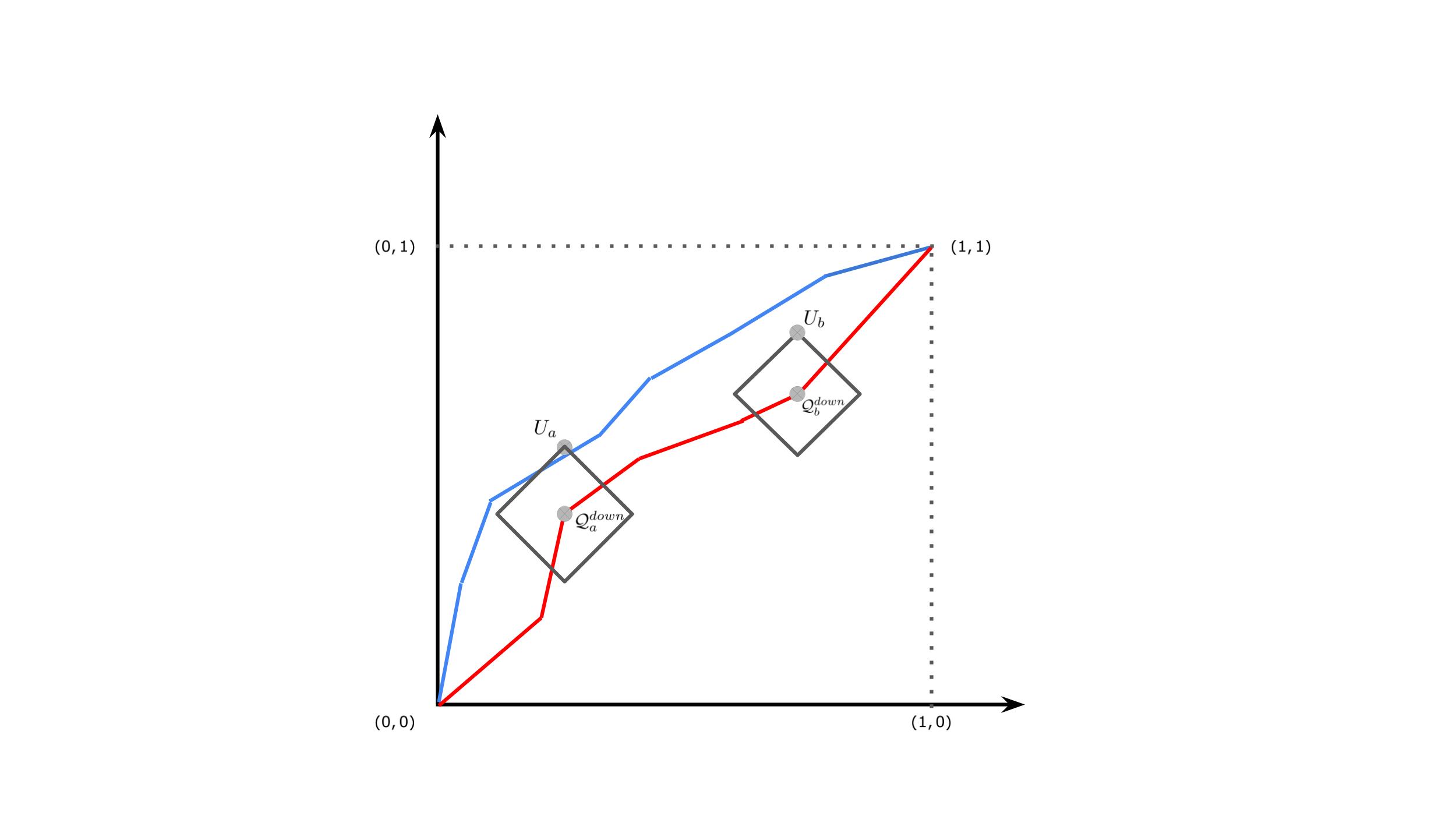}
    \caption{We have two points - $\mathcal{Q}_a^{down}$ and $\mathcal{Q}_b^{down}$ (in increasing order of FPR). We find that $\mathcal{Q}_a^{down}$ is a Boundary Cut point, whereas $\mathcal{Q}_b^{down}$ is not.}
    \label{fig:BoundCut_Appendix}
\end{figure}

We now define the three kinds of shifts that will be used in our Algorithm:
For a given $ i \in [1 , 2, \hdots ,k]$, Upshift is the transportation of $\mathcal{Q}_i^{up}$ to the point $U_i$. 
\begin{definition}[UpShift] \label{def:ush_appendix}
    For a given $ i \in [1 , 2, \hdots ,k]$, Upshift is the transportation of $\mathcal{Q}_i^{up}$ to the point $U_i$. Formally, \emph{UpShift} can be defined as the function that returns a fair threshold $\widetilde{\mathcal{Q}}_i^{up}$ (i.e. $U_i$) by taking the $\mathcal{Q}_i^{down}$ and $\varepsilon$ as the arguments.
\end{definition}
This is illustrated in the following \textbf{Figure \ref{fig:US}}.
\begin{figure}[!ht]
    \centering
    \includegraphics[scale = 0.35]{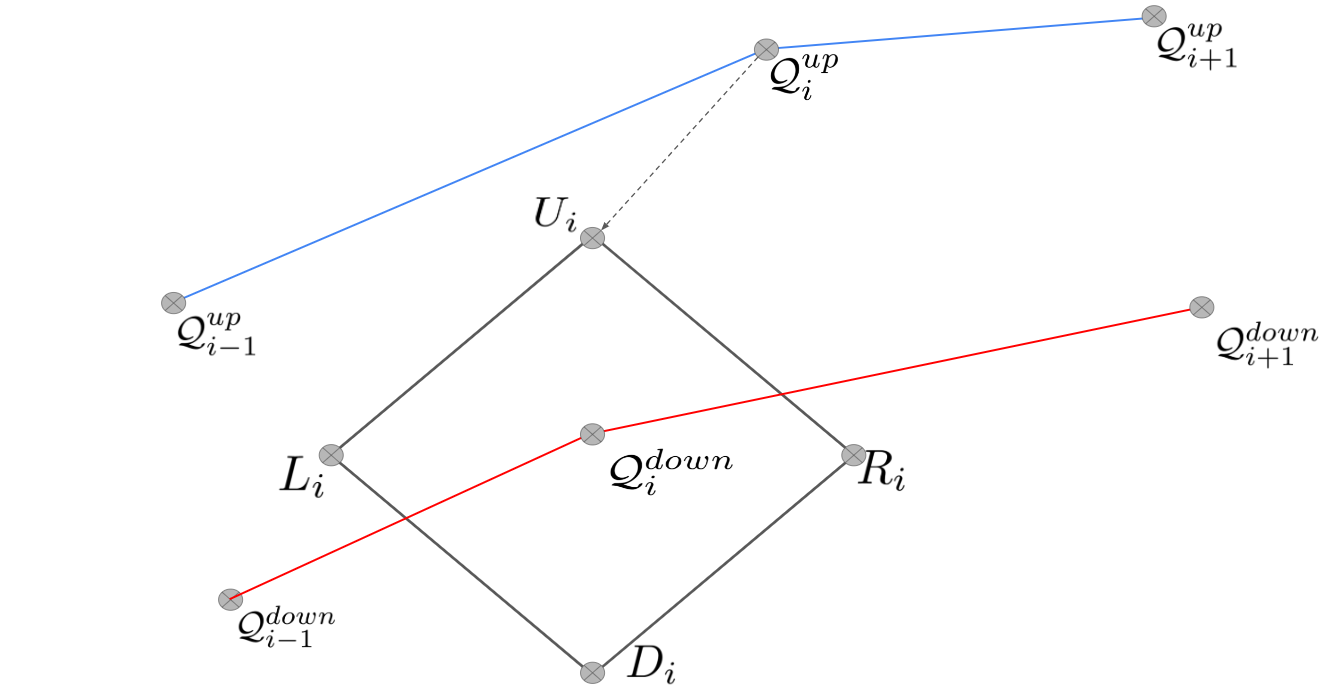}
    \caption{UpShift}
    \label{fig:US}
\end{figure}

For a given $ i \in [1 , 2, \hdots ,k]$, Leftshift is the transportation of $\mathcal{Q}_i^{up}$ to the point $L_i$. Formally,
\begin{definition}[LeftShift] \label{def:lsh_appendix}
LeftShift is a function that returns a fair threshold $\widetilde{\mathcal{Q}}_i^{up}$ (i.e. $L_i$) by taking the $\mathcal{Q}_i^{down}$ and $\varepsilon$ as the arguments.
\end{definition}
This is illustrated in the following \textbf{Figure \ref{fig:LS}}.
\begin{figure}[!ht]
    \centering
    \includegraphics[scale = 0.35]{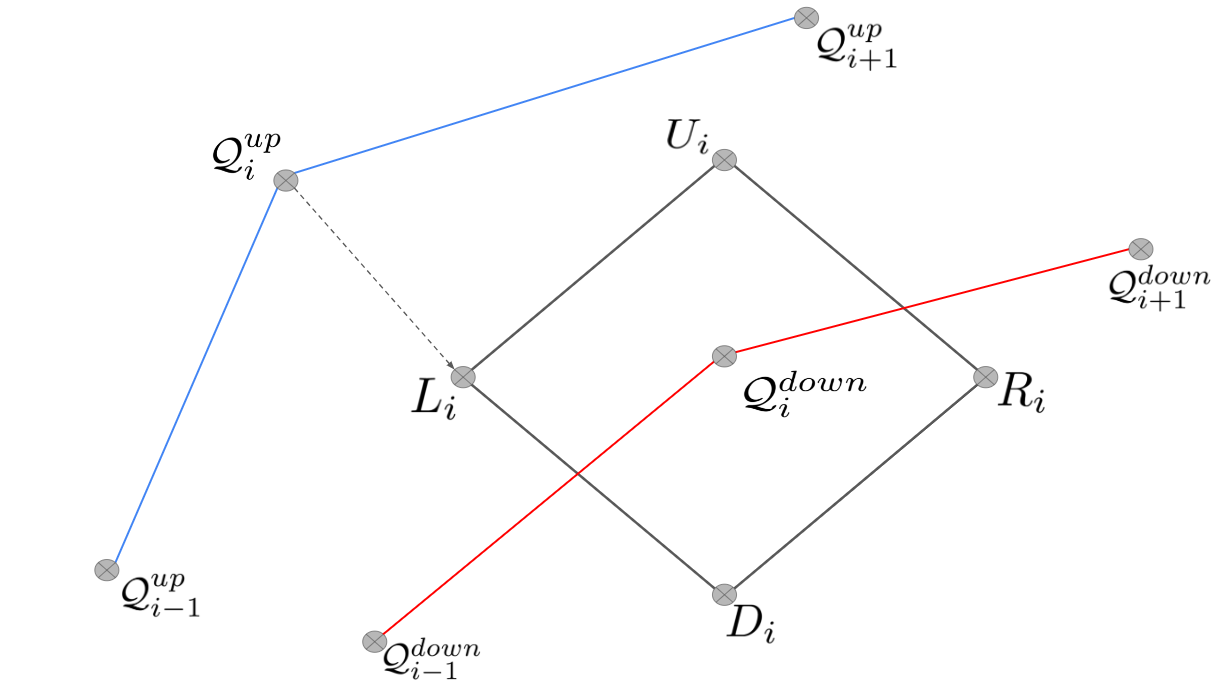}
    \caption{LeftShift}
    \label{fig:LS}
\end{figure}
\begin{definition}[CutShift] \label{def:c_appendixsh}
    For a given $i \in [1 , 2, \hdots ,k]$ (representing the index of the $ROC_{down}$), we run through all the points of the $ROC_{up}$ and return the set of all points that intersect the Norm Boundary $\mathfrak{B_i}$. Formally, we define \emph{Cutshift} as a function that takes $\mathcal{Q}_i^{down}$ and $\varepsilon$ as the arguments and returns $ROC_{up} \cap \mathfrak{B}_i$. The set $ROC_{up} \cap \mathfrak{B}_i$ can be represented as $\{p_{left} , p_{right}\}$ denoting the points at the intersection of $ROC_{up}$ at the \textbf{left-side} of the Norm Boundary and the \textbf{right-side} of the Norm Boundary respectively.
\end{definition}

This is illustrated in the following \textbf{Figure \ref{fig:CS}}.
\begin{figure}[!ht]
    \centering
    \includegraphics[scale = 0.35]{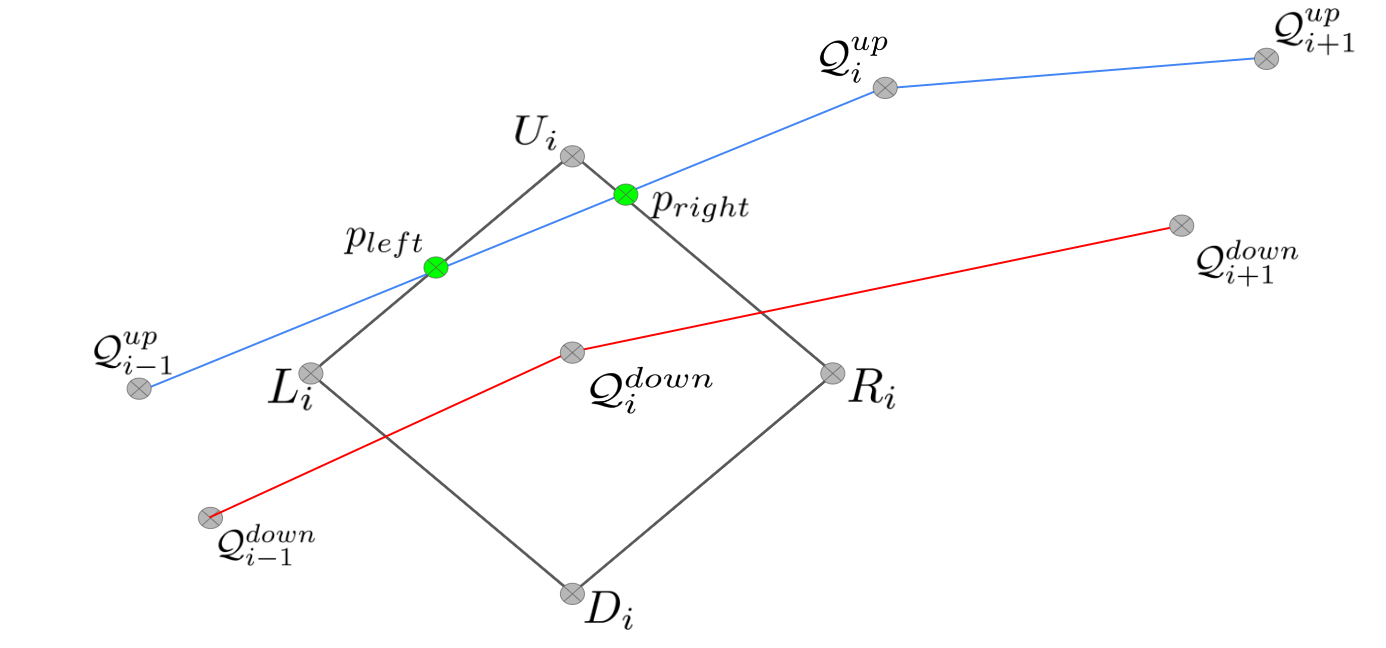}
    \caption{CutShift}
    \label{fig:CS}
\end{figure}

Note that the two intersection points - $p_{left}$ and $p_{right}$ will be to the right of $\mathcal{Q}_{i}^{up}$ when $FPR(\mathcal{Q}_{i}^{up}) \le FPR(\mathcal{Q}_{i}^{down})$. Note that it is also possible for $p_{left}$ to lie on the line segment $\overline{L_i D_i}$ instead of line segment $\overline{U_i L_i}$ when $\mathcal{Q}_i^{up}$ has sufficiently low TPR.

We elaborate on the above procedure to transport points from $ROC_{up}$ towards $ROC_{down}$ in the following subsection.

\subsection{Randomization to obtain new classifiers}
\begin{theorem}
    If $\mathcal{Q}_a,\mathcal{Q}_b,\mathcal{Q}_c$ are points in $\mathcal{S}|_s$ forming a convex hull $\Delta$ and $\mathcal{Q} \in \Delta$, then the classifier equivalent to $\mathcal{Q}$ can be obtained by following the below procedure. For each test data point $x$, use the following randomization scheme:
\begin{equation}
    Classifier_{\mathcal{Q}}(x)=
    \begin{cases}
        Classifier_{\mathcal{Q}_a}(x) & \text{w.p. } p_a \\
        Classifier_{\mathcal{Q}_b}(x) & \text{w.p. } p_b \\
        Classifier_{\mathcal{Q}_c}(x) & \text{w.p. } 1-p_a-p_b \\
    \end{cases}
\end{equation}
\end{theorem}
Here, we have, 
$p_a = \frac{c_1b_2 - c_2b_1}{a_1b_2 - a_2b_1}$, $p_b = \frac{c_1a_2 - c_2a_1}{a_1b_2 - a_2b_1}$ and 
\[a_1 = TPR(\mathcal{Q}_a) - TPR(\mathcal{Q}_c)~and~a_2 = FPR(\mathcal{Q}_a) - FPR(\mathcal{Q}_c)\] 
\[b_1 = TPR(\mathcal{Q}_b) - TPR(\mathcal{Q}_c)~and~b_2 = FPR(\mathcal{Q}_b) - FPR(\mathcal{Q}_c)\] 
\[c_1 = TPR(\mathcal{Q}) - TPR(\mathcal{Q}_c)~and~c_2 = FPR(\mathcal{Q}) - FPR(\mathcal{Q}_c)\]

\section{Theoretical Results}
\subsection{Piecewise Linear Approximation}

\begin{theorem} \label{th:pla_appendix}
Let $ROC_{\widehat{H_s^a}, \widehat{G_s^a}}$ be the \emph{PLA} of $ROC_{H_s^a,G_s^a}$ over the query set of $k$ equidistant thresholds, $\mathcal{T} = \{ t_i \mid t_i = i/k \ \forall i \in [k] \}$ then the corresponding $\mathcal{L}_{PLA}$ is bounded as:
\begin{equation*}
    \mathcal{L}_{PLA} \le \frac{1}{2} \frac{u_T u_F}{k^2} \times k  = \frac{1}{2} \frac{u_T u_F}{k}
\end{equation*}
\end{theorem}

\begin{proof}

In \textbf{Figure \ref{fig:PLA_p1}}, shaded area is the approximation loss $\mathcal{L}_{PLA}$. 
\begin{figure}[!ht]
    \centering
    \includegraphics[scale = 0.06]{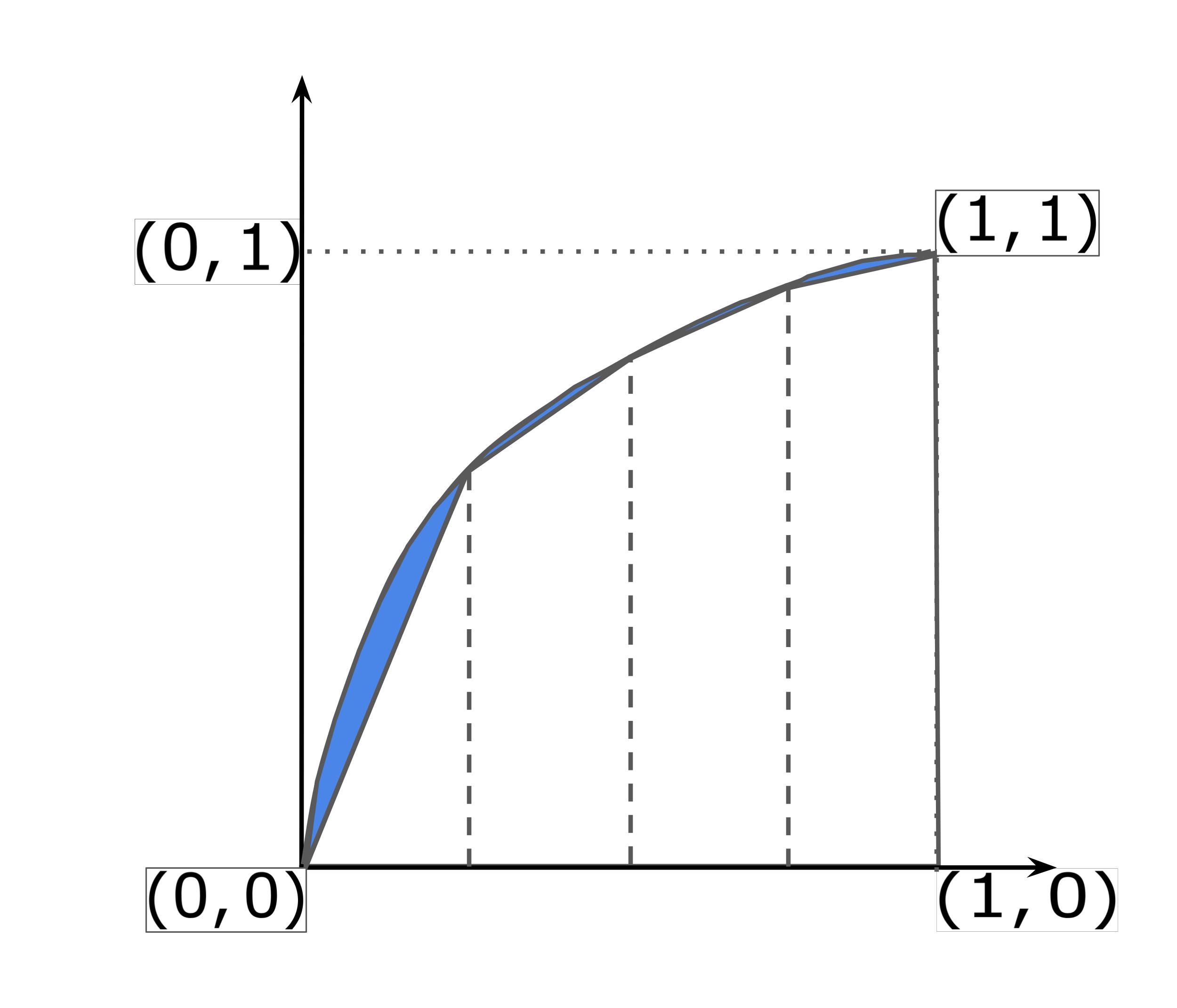}
    \caption{$\mathcal{L}_{PLA}$}
    \label{fig:PLA_p1}
\end{figure}
Let us consider the situation where $ROC_{H_s^a, G_s^a}$ \emph{maximally} deviates from its PLA $ROC_{\widehat{H_s^a}, \widehat{G_s^a}}$. To find an upper bound to this area, we must stretch it till the dotted line 

\begin{figure}[!ht]
    \centering
    \includegraphics[scale = 0.3]{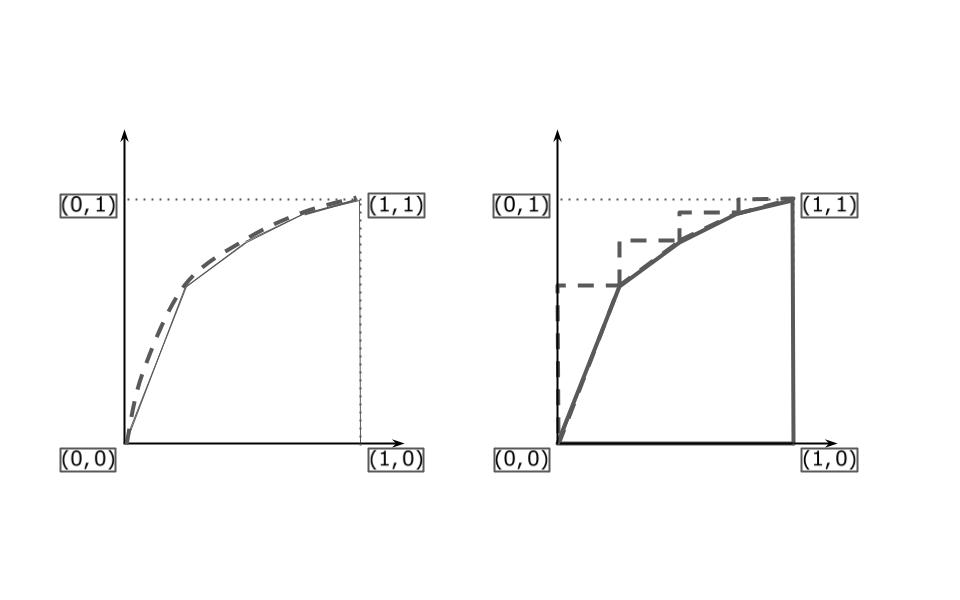}
    \caption{Maximally streching the ROC (Dotted line)}
    \label{fig:pla_proof_appendix}
\end{figure}

The area cannot go beyond the dotted line (\textbf{Figure \ref{fig:PLA2_p2}})because ROCs are one-to-one and monotonically increasing functions.
\begin{figure}[!ht]
    \centering
    \includegraphics[scale = 0.25]{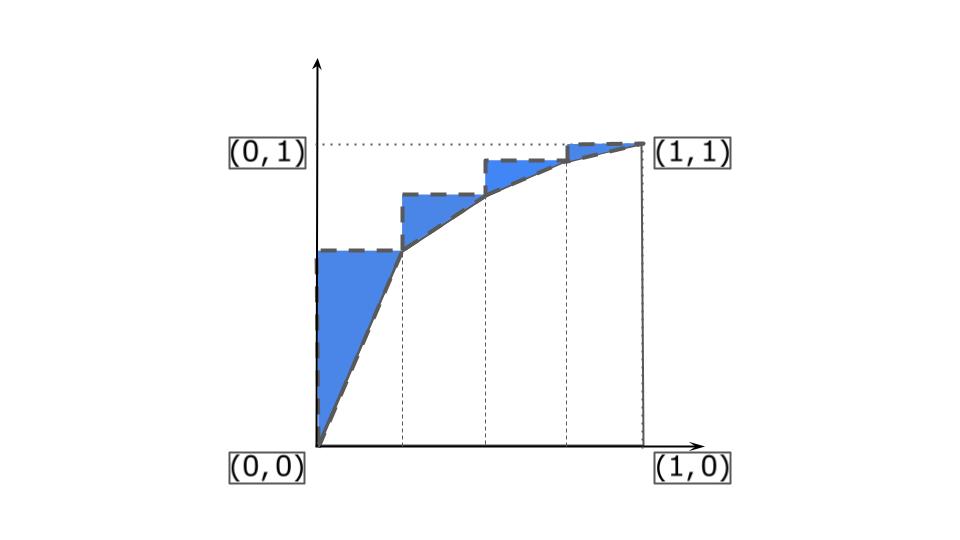}
    \caption{The area shaded by the darker shade of blue is the maximum possible loss of AUC due to Linear Interpolation.}
    \label{fig:PLA2_p2}
\end{figure}
So, our goal now, is to bound the sum of areas of the blue shaded triangles. We have the base of each triangle to be $\frac{1}{k}\times u_F$ (since $k$ thresholds and maximum slope of $FPR$ with respect to the thresholds is $u_F$). We have the maximum possible height of each triangle $\frac{1}{k}\times u_T$ (since $k$ thresholds and maximum slope of $TPR$ with respect to the thresholds is $u_T$). This makes the maximum possible area of each triangle $\frac{u_T u_F}{2k^2}$.
So, for an interval between thresholds $t_i, t_{i+1}$, the loss incurred is $\leq \frac{1}{2} \frac{u_T u_F}{k^2}$. To extend this for the entire ROC over $k$ intervals, we have:
\begin{equation*}
    \mathcal{L}_{PLA} \leq \frac{1}{2} \frac{u_Tu_F}{k^2} \times k  = \frac{1}{2} \frac{u_Tu_F}{k}    
\end{equation*}
\end{proof}

Therefore, we can infer:
\[ \lim_{k\to\infty} \mathcal{L}_{LPA} =\lim_{k\to\infty} \frac{1}{2} \frac{u_Tu_F}{k} =0 \]
\subsection{Boundary Optimality}
\subsubsection{All optimal points lie on the Norm boundary}



\begin{theorem}(Norm Boundary)
    If $(\widetilde{\mathcal{Q}}_i^{up})_{i\in \{1 ,2 , \hdots , k\}}$ is the set of optimal fair (points that maximize the AUC and also satisfy the $\varepsilon$ fairness criteria) thresholds must necessarily be a subset of $(\mathfrak{B}_i)_{i\in \{1 ,2 , \hdots , k\}}$. 
\end{theorem}

\begin{proof}
(Proof by Contradiction) 
Let us assume that some point $C$ in the interior of the Norm Set is the optimal fair (point that leads to ROC with maximum possible AUC while satisfying $\varepsilon_1$\ourdef) point. 
As we can see in \textbf{Figure \ref{fig:n1}}, we have transported $\mathcal{Q}_{i}^{up}$ to $C$ in the interior of the Norm set. The shaded area denotes the AUC loss due to this transformation. However, as seen in the next figure \textbf{Figure \ref{fig:n2}}, the AUC loss can be decreased by choosing a point (we choose the CutShift point) on the Norm boundary. Thus, we can always decrease AUC loss by choosing a point on the Norm Boundary. Formally, if point $C$ was the optimal fair point, then the AUC loss with respect to that point is $Area(\square \mathcal{Q}^{up}_{i-1} C \mathcal{Q}^{up}_{i+1} \mathcal{Q}^{up}_{i})$. 
\\
However, if $A$ is the optimal fair point (Fig \ref{fig:n2}), then the AUC loss with respect to that point is $Area(\square \mathcal{Q}^{up}_{i} A \mathcal{Q}^{up}_{i+1})$. However, we notice that:$Area(\square \mathcal{Q}^{up}_{i-1} C \mathcal{Q}^{up}_{i+1} \mathcal{Q}^{up}_{i}) = Area(\square \mathcal{Q}^{up}_{i} A \mathcal{Q}^{up}_{i+1}) + Area(\square \mathcal{Q}^{up}_{i-1} C \mathcal{Q}^{up}_{i+1} A)$. Since $Area(\square \mathcal{Q}^{up}_{i-1} C \mathcal{Q}^{up}_{i+1} A) \ge 0$, we have:
    \[Area(\square \mathcal{Q}^{up}_{i-1} C \mathcal{Q}^{up}_{i+1} \mathcal{Q}^{up}_{i}) \ge Area(\square \mathcal{Q}^{up}_{i} A \mathcal{Q}^{up}_{i+1}) \]
This is a contradiction to the assumption that $C$ is the optimal fair point. Therefore, $C$ is not an optimal fair point.
\end{proof}

\begin{figure}[!ht]
    \centering
    \includegraphics[scale = 0.23]{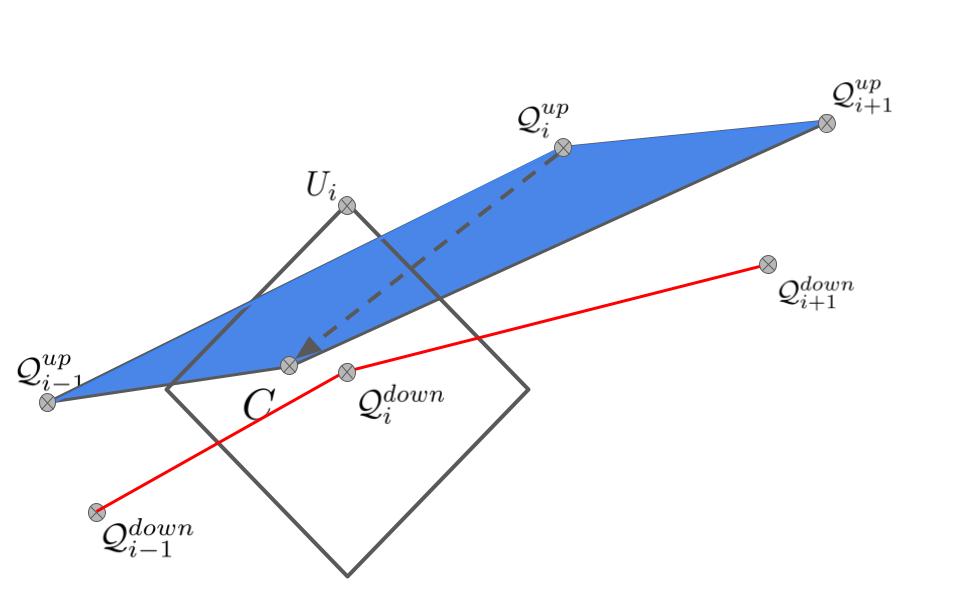}
    \caption{The blue colored region indicates the AUC loss.}
    \label{fig:n1}
\end{figure}
\begin{figure}[!ht]
    \centering
    \includegraphics[scale = 0.23]{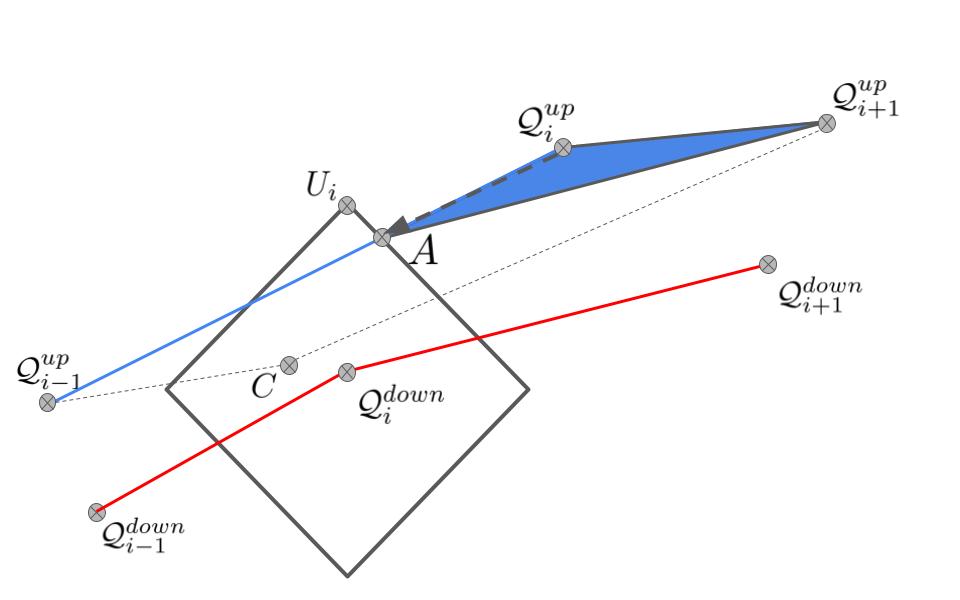}
    \caption{The dark blue colored region indicates the new AUC loss. The light blue region indicates the previous AUC loss.}
    \label{fig:n2}
\end{figure}

\subsection{CutShift Optimality}


\begin{theorem}
If $i$ is a Boundary cut point, then the CutShift operation must be performed. Of the 2 points ($p_{left}$ and $p_{right}$) returned by the Cutshift operation, the point that is closer to $\mathcal{Q}_{i}^{up}$ must be chosen i.e.
    
    $\mathcal{\widetilde{Q}}_i^{up} = argmin_{p \in \{ p_{left} , p_{right}\}} |FPR(\mathcal{Q}_i^{up}) - FPR(p)|$
\end{theorem}

\begin{proof}
(Proof by Contradiction) 
Let us assume that some point $C$ on the Norm Boundary is the optimal fair (point that leads to ROC with maximum possible AUC while satisfying $\varepsilon_1$\ourdef) point. 
As we can see in \textbf{Figure \ref{fig:c1}}, we have transported $\mathcal{Q}_{i}^{up}$ to $C$ in the interior of the Norm set. The shaded area denotes the AUC loss due to this transformation. However, as seen in the next figure Fig \ref{fig:n2}, the AUC loss can be decreased by choosing a point (we choose the CutShift point) on the Norm boundary. Thus, we can always decrease AUC loss by choosing a point on the Norm Boundary. Formally, if point $C$ was the optimal fair point, then the AUC loss with respect to that point is $Area(\square \mathcal{Q}^{up}_{i-1} C \mathcal{Q}^{up}_{i+1} \mathcal{Q}^{up}_{i})$. 
\\
However, if $A$ is the optimal fair point (Fig \ref{fig:c2}), then the AUC loss with respect to that point is $Area(\square \mathcal{Q}^{up}_{i} A \mathcal{Q}^{up}_{i+1})$. However, we notice that:$Area(\square \mathcal{Q}^{up}_{i-1} C \mathcal{Q}^{up}_{i+1} \mathcal{Q}^{up}_{i}) = Area(\square \mathcal{Q}^{up}_{i} A \mathcal{Q}^{up}_{i+1}) + Area(\square \mathcal{Q}^{up}_{i-1} C \mathcal{Q}^{up}_{i+1} A)$. Since $Area(\square \mathcal{Q}^{up}_{i-1} C \mathcal{Q}^{up}_{i+1} A) \ge 0$, we have:
    \[Area(\square \mathcal{Q}^{up}_{i-1} C \mathcal{Q}^{up}_{i+1} \mathcal{Q}^{up}_{i}) \ge Area(\square \mathcal{Q}^{up}_{i} A \mathcal{Q}^{up}_{i+1}) \]
This is a contradiction to the assumption that $C$ is the optimal fair point. Therefore, $C$ is not an optimal fair point.

\end{proof}

\begin{figure}[!ht]
    \centering
    \includegraphics[scale =0.2]{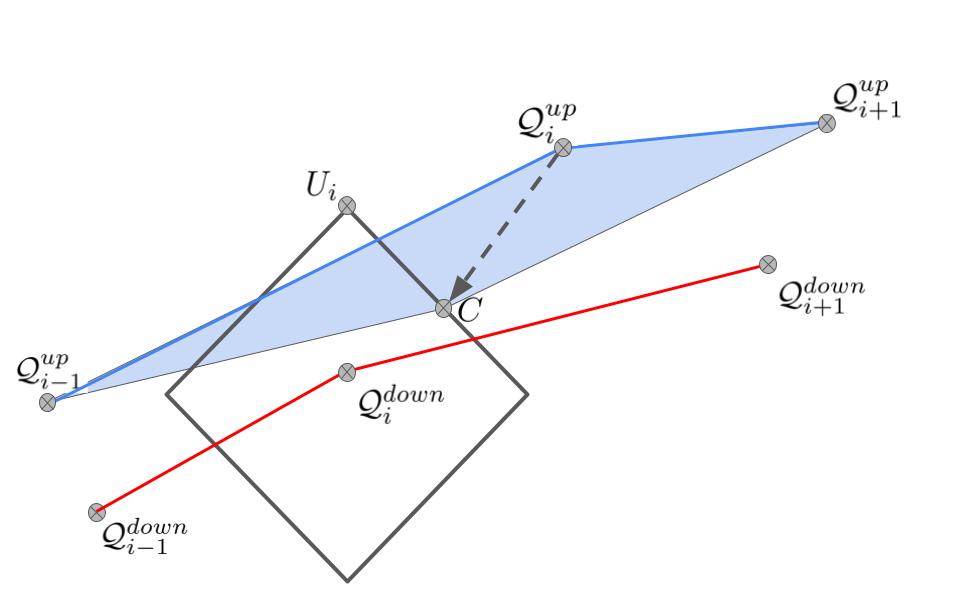}
    \caption{CutShift Operation is not followed. The light blue area indicates the AUC loss due to this operation.}
    \label{fig:c1}
\end{figure}

\begin{figure}[!ht]
    \centering
    \includegraphics[scale =0.2]{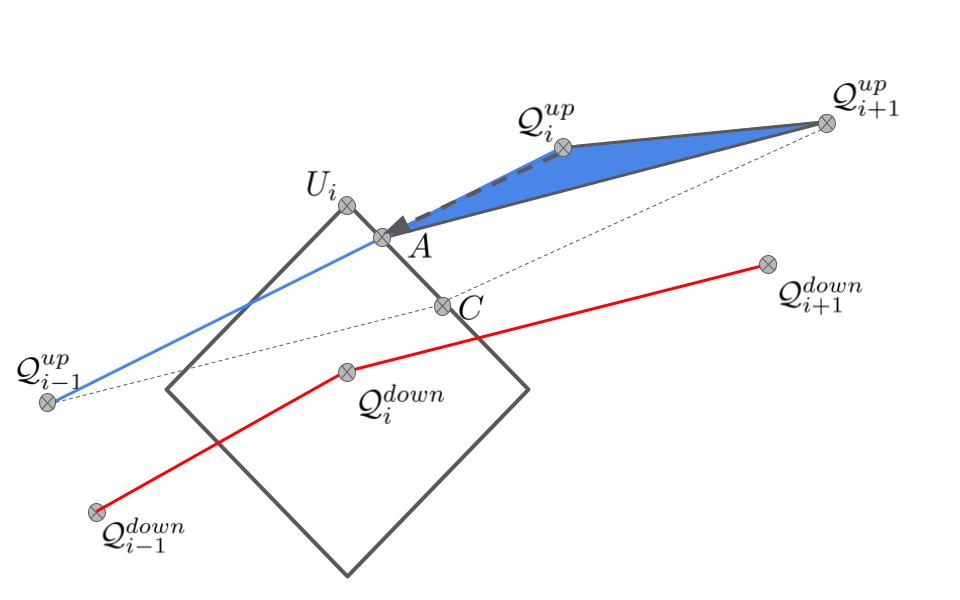}
    \caption{CutShift Operation is followed. The dark blue area indicates the AUC loss due to this operation. It is lesser than the previous AUC loss as seen in Figure 9.}
    \label{fig:c2}
\end{figure}

\subsection{Upshift and Left Shift}
\begin{theorem}[UpShift]
    If $i$ is not a Boundary cut point and if $Area(\square \mathcal{Q}_{i+1}\mathcal{Q}_{i}\mathcal{Q}_{i-1}{L}_i \ge Area(\square \mathcal{Q}_{i+1}\mathcal{Q}_{i}\mathcal{Q}_{i-1}{U}_i$), then UpShift operation must be performed. The resulting point ($U_i$) is the new fair point $\mathcal{\widetilde{Q}}_i^{up}$. Else, LeftShift operation must be performed. The resulting point ($L_i$) is the new fair point $\mathcal{\widetilde{Q}}_i^{up}$.
\end{theorem}

\begin{proof}
    By a similar argument, as the previous proofs, we argue (through \textbf{Figure \ref{fig:Upsh}}, \textbf{Figure \ref{fig:nUpSh}} and \textbf{Figure \ref{fig:UpShp}}), we can prove that either the point recommended by UpShift ($U_i$) or LeftShift ($L_i$) is the optimal point. So, to decide between them, we use Heron's formula to find the area of both quadrilaterals and then compare their areas to find the least AUC loss. We can use Heron's formula to find the area of a quadrilateral in the following way: If $\square ABCD$ is a quadrilateral with vertices $A,B,C$ and $D$. This area is easily found in this context by splitting $\square ABCD$ into two disjoint triangles- $\Delta ABC$ and $\Delta ACD$ and using the Herons formula \cite{kendig20002000} on each triangle. For example, consider $Area(\Delta \mathcal{Q}_i^{up} \mathcal{Q}_{i-1}^{up} L_i)$. Let $a = ||\mathcal{Q}_i^{up} \mathcal{Q}_{i-1}^{up}||_2$, $b = ||\mathcal{Q}_i^{up} L_i||_2$ and $c = ||\mathcal{Q}_{i-1}^{up} L_i||_2$. Additionally, we define $s = \frac{a+b+c}{2}$. Then, it is true that:
\[Area(\Delta \mathcal{Q}_i^{up} \mathcal{Q}_{i-1}^{up} L_i ) = \sqrt{s(s-a)(s-b)(s-c)}\]

\end{proof}

\begin{figure}[!ht]
    \centering
    \includegraphics[scale = 0.35]{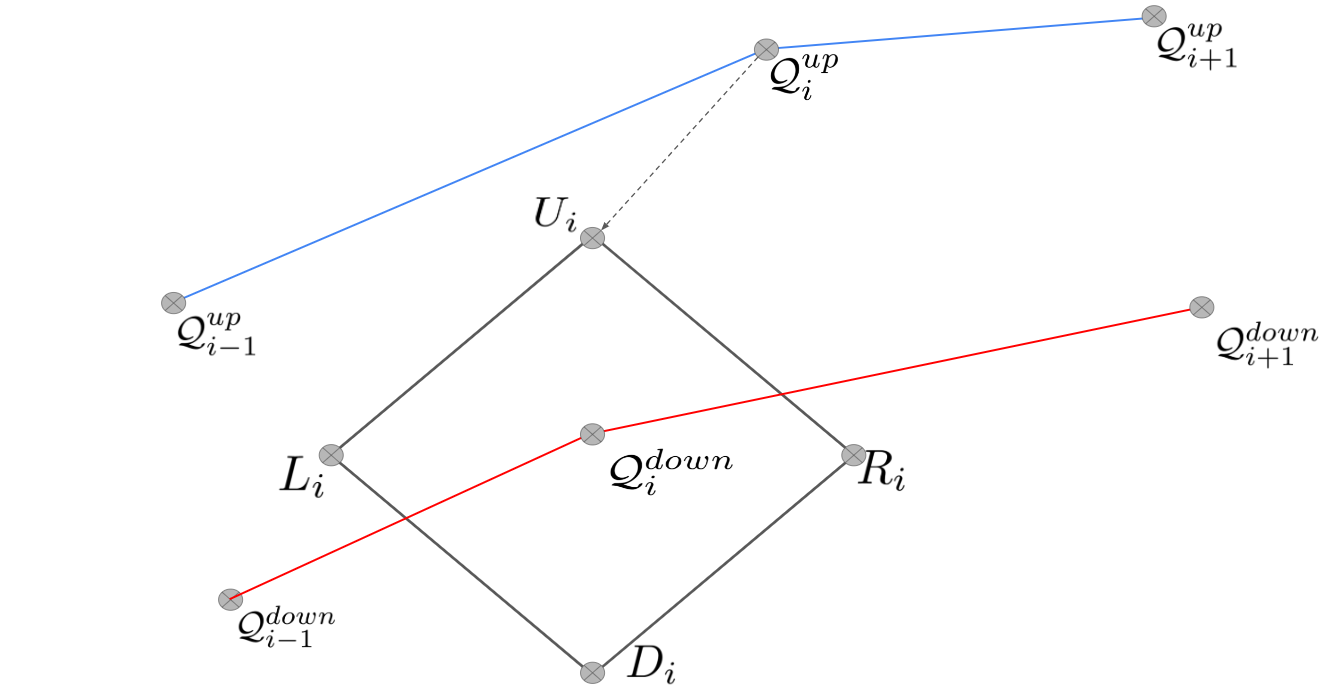}
    \caption{The dotted arrow represents the UpShift transportation of the point from $\mathcal{Q}_i^{up}$ to $U_i$}
    \label{fig:Upsh}
\end{figure}
\begin{figure}[!ht]
    \centering
    \includegraphics[scale = 0.4]{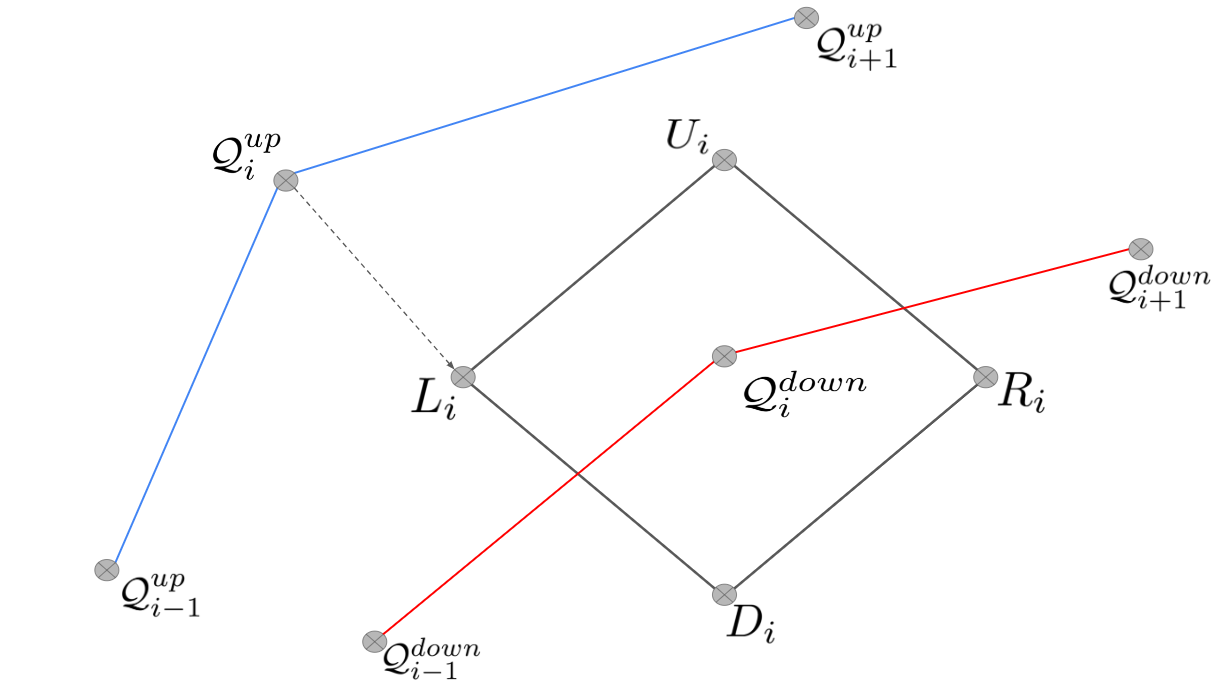}
    \caption{The dotted arrow represents the LeftShift transportation of the point from $\mathcal{Q}_i^{up}$ to $U_i$}
    \label{fig:lesh}
\end{figure}

\begin{figure}[!ht]
    \centering
    \includegraphics[scale =0.1]{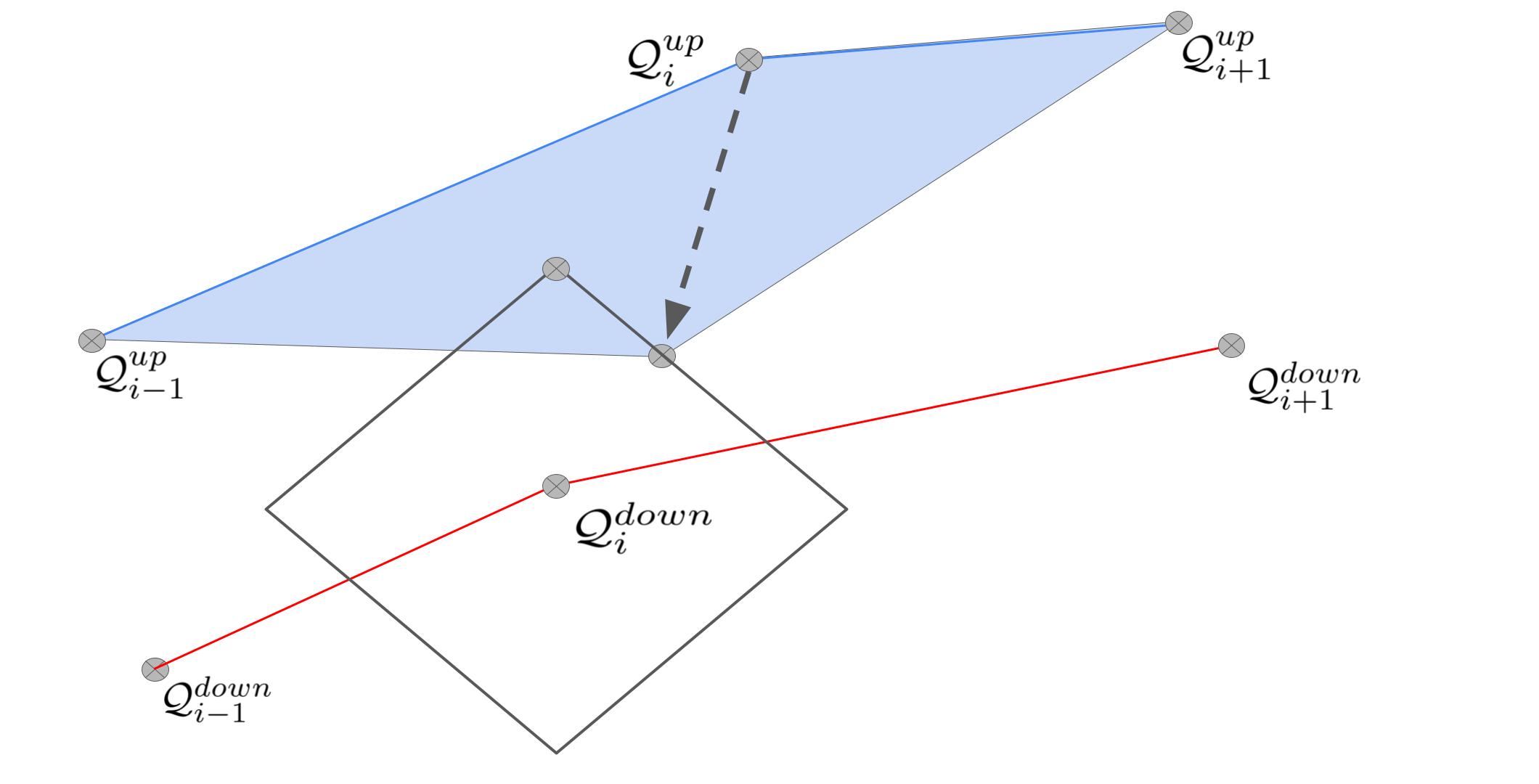}
    \caption{UpShift Operation is not followed. The light blue area indicates the AUC loss due to this operation.}
    \label{fig:nUpSh}
\end{figure}

\begin{figure}[!ht]
    \centering
    \includegraphics[scale =0.1]{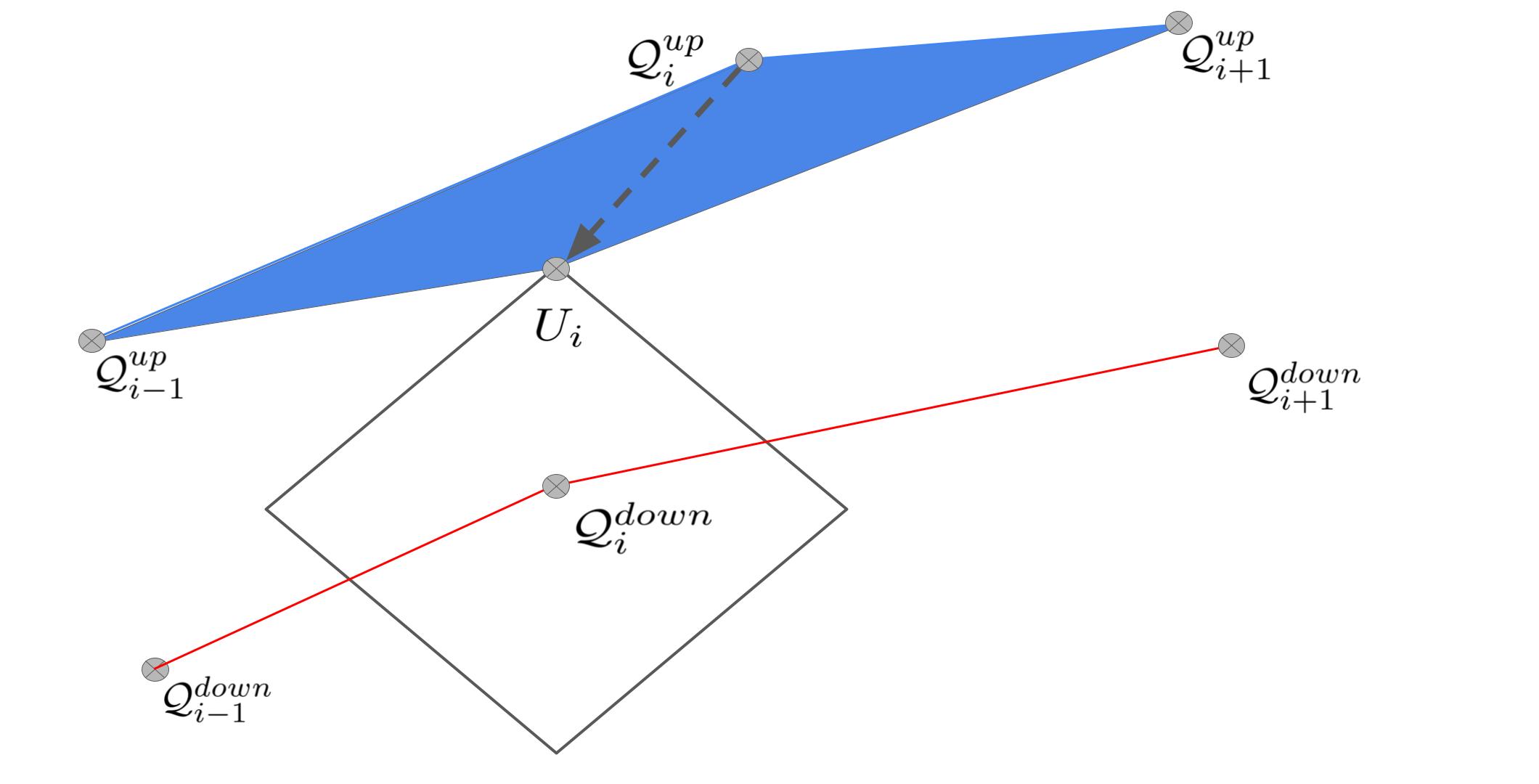}
    \caption{UpShift Operation is followed. The dark blue area indicates the AUC loss due to this operation. It is lesser than the previous AUC loss as seen in Figure 11.}
    \label{fig:UpShp}
\end{figure}

 The optimality of AUC (Theorem 4.2) follows from Theorem D.2, Theorem D.3 and Theorem D.4.

\subsection{Sample Complexity}
If the \textbf{Assumption 4.2} holds true, then we have the following analysis:
\begin{itemize}
    \item All UpShift Operations will be constant time ($O(1)$).
    \item All CutShift Operations will also be constant time ($O(1)$). This is because \textbf{Assumption 4.2} ensures that we do not have to run through the entire length of $ROC_{up}$ to find the intersection points i.e. $p_{left}$ and $p_{right}$.
\end{itemize}
Therefore, the running time of \ouralgo~ is $O(k)$.
However, when no assumptions are made, then the CutShift operation is no longer $O(1)$. We may have to run through the entire length of $ROC_{up}$ to find the intersection points i.e. $p_{left}$ and $p_{right}$. This makes the CutShift operation $O(k)$. Therefore, the time complexity of \ouralgo~ is $O(k^2)$.

\subsection{Further Variants}
\subsubsection{Multiple Protected Groups}
Our approach is extendable to scenarios involving multiple protected groups. The procedure begins by applying the FROC algorithm to the ROC curve that is immediately above the bottom-most ROC curve. Subsequently, FROC is applied to the ROC curve directly above the one previously processed. This iterative application continues until the top ROC curve is reached. While this method ensures \ourdef fairness across all protected groups, the proof of optimality remains an open question.

\subsubsection{Intersection of ROC Curves}
In cases where the ROC curves intersect more than twice, our algorithm will still produce a fair output. However, the existing optimality theorems do not apply in such scenarios. When intersections occur, the FROC algorithm can be applied to the dominant segments of the ROC curves—those portions where no intersections are present.

\section{Experiments}
\subsection{Datasets}
\subsubsection{UCI Adult Dataset}
The Adult Dataset \cite{misc_adult_2} comprises 48,842 instances, each containing 14 attributes, including both categorical and continuous variables. The dataset was designed to predict whether an individual's income exceeds \$50,000 per year, making it suitable for binary classification tasks. The features include demographic information such as age, education level, marital status, occupation, work hours per week, and native country, among others.
\subsubsection{COMPAS Recidivism Dataset}
COMPAS Dataset \cite{angwin2022machine} is a widely-discussed and controversial dataset utilized in the field of criminal justice and fairness-aware machine learning. The COMPAS (Correctional Offender Management Profiling for Alternative Sanctions) dataset is commonly employed to explore the potential bias and fairness issues that may arise in predictive models used for criminal justice decisions.
The COMPAS dataset consists of historical data on defendants who were considered for pretrial release in a U.S. county. The data includes various features extracted from defendant profiles, such as age, race, gender, past criminal history, pending charges, and other pertinent factors. Additionally, the dataset contains binary labels indicating whether a defendant was rearrested within a specific period after their release.
\subsection{Protected Groups}
In the context of this paper, we consider the relative performance of the classifiers with respect to the different protected groups - sex (Male and Female) for the Adult Dataset and Race (African Americans and Others) for the COMPAS Dataset.

\subsection{Experiment Details}
We have performed statistical analysis on FROC, but not on the original classifier. This is because studying the fairness-accuracy trade-of is our goal (as opposed to studying the performance of the baseline classifier). However, it must be noted that since the ROC shifting is deterministic, all randomness emerges from the post-shift classifier builder. For the statistical analysis, we have $10$ iterations of the experiment as $\varepsilon$ runs from $0.001$ to $0.1$ in $20$ intervals. (Except for the case of Random Forest Gini (Adult) : $0.001$ to $0.2$ in $20$ intervals.)

\subsection{Plots}
\subsubsection{Adult Dataset - Weighted ensemble L2}
\begin{figure}[!ht]
    \centering
    \includegraphics[width=1\linewidth]{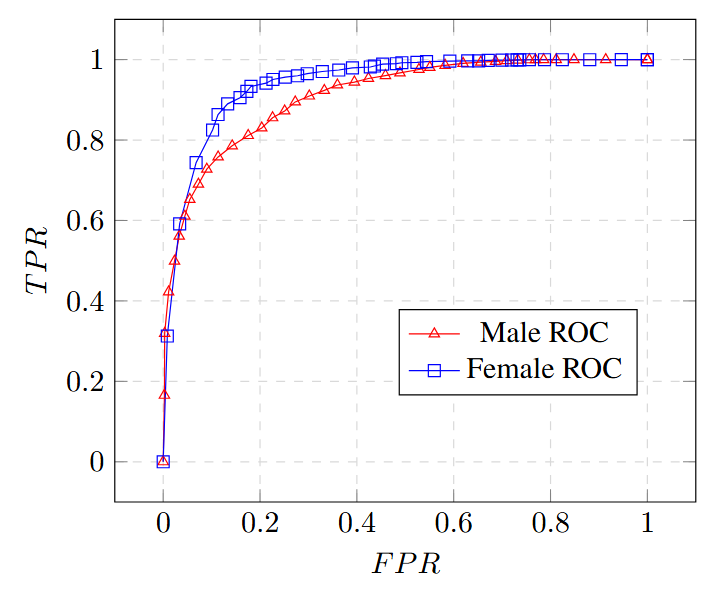}
    \caption{Weighted Ensemble L2 Baseline ROCs for Adult Dataset}
    \label{fig:WEL2_Adult_ROC}
\end{figure}

\begin{figure}[!ht]
    \centering
    \includegraphics[width=1\linewidth]{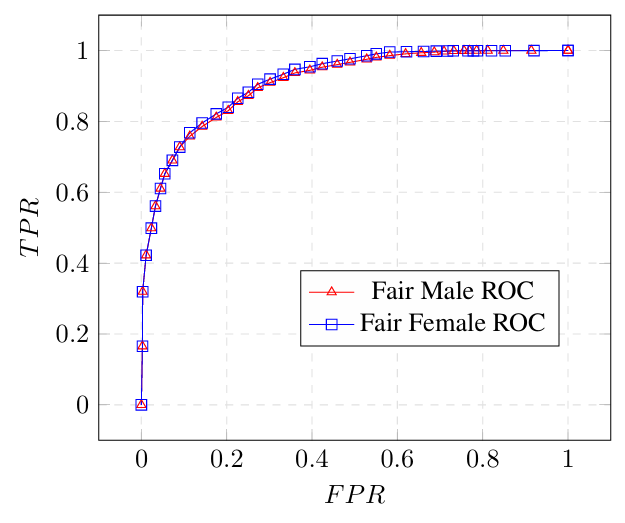}
    \caption{(Fair $\varepsilon_1 = 0.01$) Weighted Ensemble L2-\ouralgo\   ROCs for Adult Dataset}
    \label{fig:WEL2_Adult_ROC_FROC}
\end{figure}

\begin{figure}[!ht]
    \centering
    \includegraphics[width=1\linewidth]{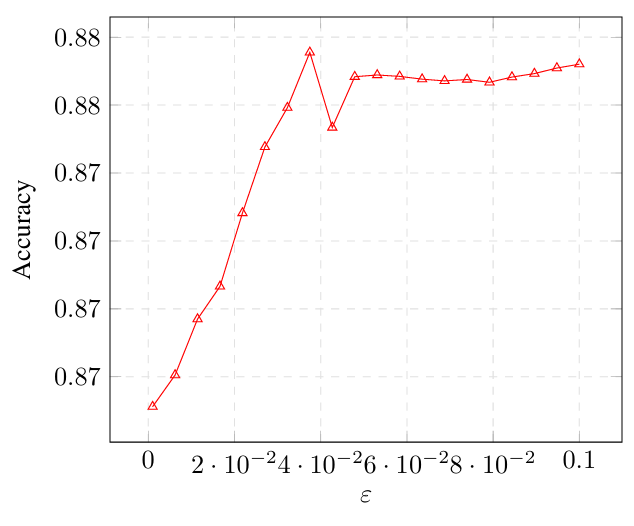}
    \caption{Weighted Ensemble L2-\ouralgo\  Accuracy vs. $\varepsilon_1$ (Adult)}
    \label{fig:WEL2_Adult_Accuracy}
\end{figure}

\begin{figure}[!ht]
    \centering
    \includegraphics[width=1\linewidth]{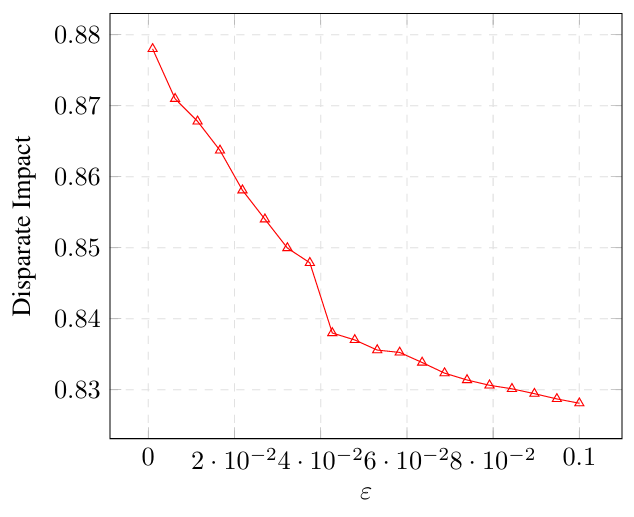}
    \caption{Weighted Ensemble L2-\ouralgo\  Disparate Impact vs. $\varepsilon_1$ (Adult)}
    \label{fig:WEL2_Adult_DI}
\end{figure}

\begin{figure}[!ht]
    \centering
    \includegraphics[width=1\linewidth]{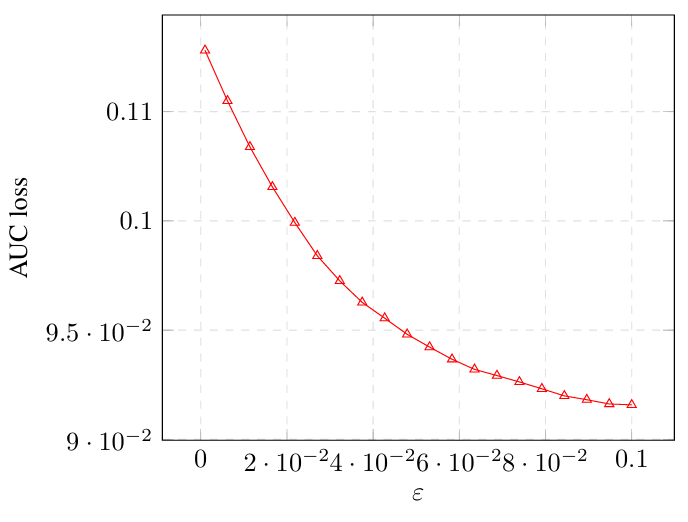}
    \caption{Weighted Ensemble L2-\ouralgo\ AUC loss vs. $\varepsilon_1$ (Adult)}
    \label{fig:WEL2_Adult_AUC}
\end{figure}

\begin{itemize}
    \item We have applied FROC with the our fairness parameter $\varepsilon = 0.01$ in \textbf{Figure \ref{fig:WEL2_Adult_ROC_FROC}}. As promised, the resulting ROCs are 'closer' to each other.
    \item In \textbf{Figure \ref{fig:WEL2_Adult_Accuracy}} and \textbf{Figure \ref{fig:WEL2_Adult_DI}}, we have the Accuracy vs. $\varepsilon_1$ and the Disparate Impact vs. $\varepsilon_1$ plot.
    \item This analysis gives us a maximum variance of $1.88\times10^{-6}$ and a maximum CoV (Coefficient of Variation) of $0.15\%$ for Accuracy.
    \item As for the Disparate Impact, the analysis gives us a maximum variance of $2.25\times10^{-5}$ and a maximum CoV of $0.55\%$.
    \item As seen in the plots, we observe that a $1\%$ drop in Accuracy improves the Disparate Impact by $5\%$.
    \item Finally, in \textbf{Figure \ref{fig:WEL2_Adult_AUC}}, we have the AUC loss vs. $\varepsilon$ plot. As seen in the figure, the AUC loss decays to $0$ as our fairness constraint loosens.
\end{itemize}

\subsubsection{Adult Dataset - Random Forest Gini}

\begin{figure}[!ht]
    \centering
    \includegraphics[width=1\linewidth]{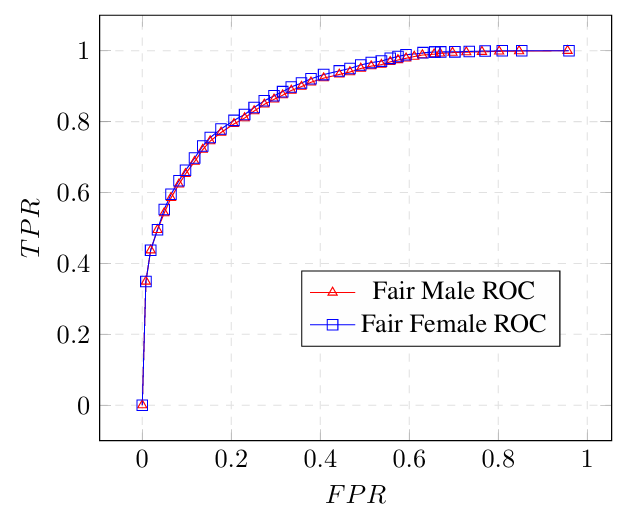}
    \caption{Random Forest (Gini) Baseline ROCs for Adult Dataset}
    \label{fig:RFG_Adult_Baseline_ROC_FROC}
\end{figure}

\begin{figure}[!ht]
    \centering
    \includegraphics[width=1\linewidth]{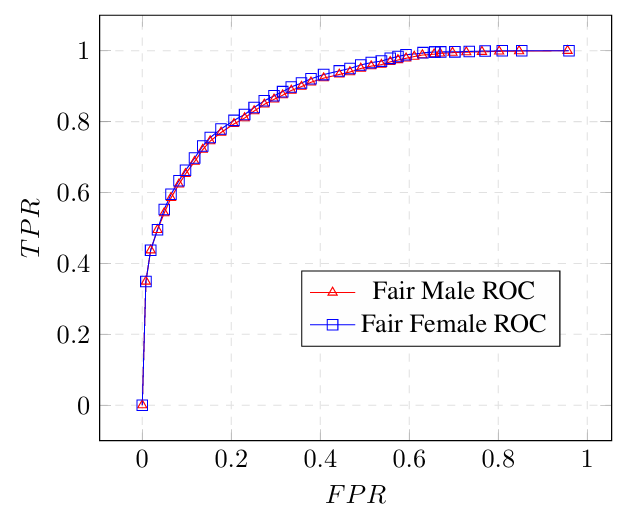}
    \caption{(Fair $\varepsilon_1 = 0.01$) Random Forest (Gini)-\ouralgo\   ROCs for Adult Dataset}
\end{figure}

\begin{figure}[!ht]
    \centering
    \includegraphics[width=1\linewidth]{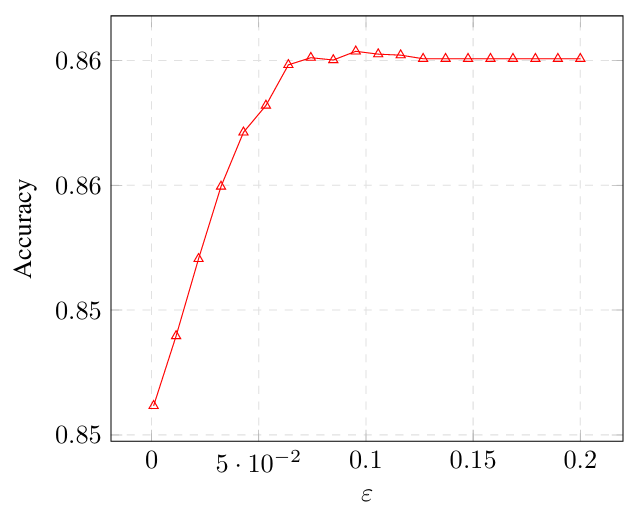}
    \caption{Random Forest (Gini)-\ouralgo\  Accuracy vs. $\varepsilon_1$ (Adult)}
    \label{fig:RFG_Adult_Accuracy}
\end{figure}

\begin{figure}[!ht]
    \centering
    \includegraphics[width=1\linewidth]{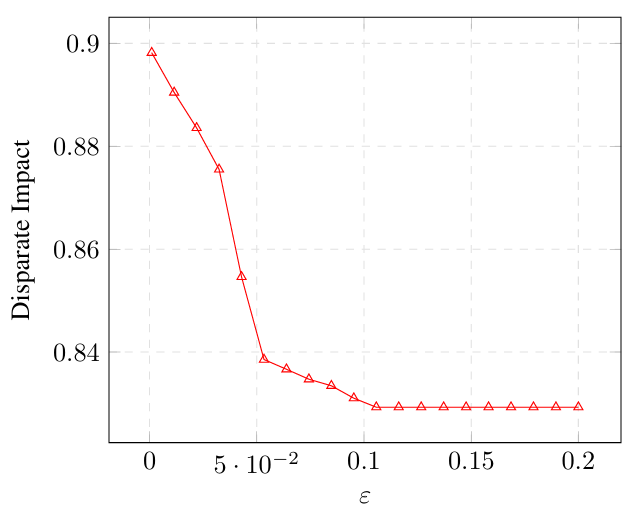}
    \caption{Random Forest (Gini)-\ouralgo\  Disparate Impact vs. $\varepsilon_1$ (Adult)}
    \label{fig:RFG_Adult_DI}
\end{figure}

\begin{figure}[!ht]
    \centering
    \includegraphics[width=1\linewidth]{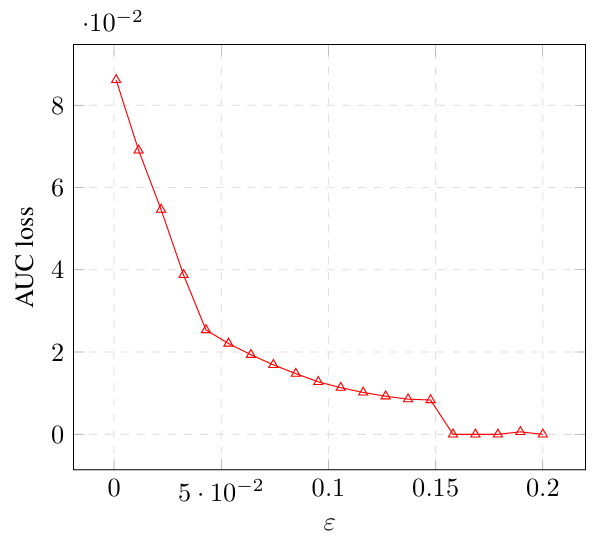}
    \caption{Random Forest (Gini)-\ouralgo\ AUC loss vs. $\varepsilon_1$ (Adult)}
    \label{fig:RFG_Adult_AUC}
\end{figure}

\begin{itemize}
    \item We have applied FROC with the our fairness parameter $\varepsilon = 0.01$ in \textbf{Figure \ref{fig:RFG_Adult_Baseline_ROC_FROC}}. As promised, the resulting ROCs are 'closer' to each other.
    \item In \textbf{Figure \ref{fig:RFG_Adult_Accuracy}} and \textbf{Figure \ref{fig:RFG_Adult_DI}}, we have the Accuracy vs. $\varepsilon_1$ and the Disparate Impact vs. $\varepsilon_1$ plot.
    \item This analysis gives us a maximum variance of $8.3\times10^{-7}$ and a maximum CoV (Coefficient of Variation) of $0.1\%$ for Accuracy.
    \item As for the Disparate Impact, the analysis gives us a maximum variance of $7.59\times10^{-6}$ and a maximum CoV of $0.75\%$.
    \item As seen in the plots, we observe that a $1\%$ drop in Accuracy improves the Disparate Impact by $7\%$.
    \item Finally, in \textbf{Figure \ref{fig:RFG_Adult_AUC}}, we have the AUC loss vs. $\varepsilon_1$ plot. As seen in the figure, the AUC loss decays to $0$ as our fairness constraint loosens.
\end{itemize}

\subsubsection{Adult Dataset - FNNC}

\begin{figure}[!ht]
    \centering
    \includegraphics[width=1\linewidth]{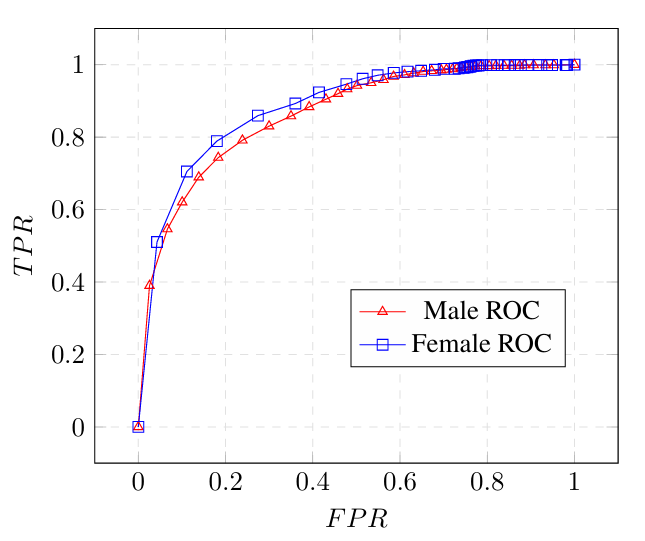}
    \caption{FNNC Baseline ROCs for Adult Dataset}
    \label{fig:FNNC_Adult_ROC}
\end{figure}

\begin{figure}[!ht]
    \centering
    \includegraphics[width=1\linewidth]{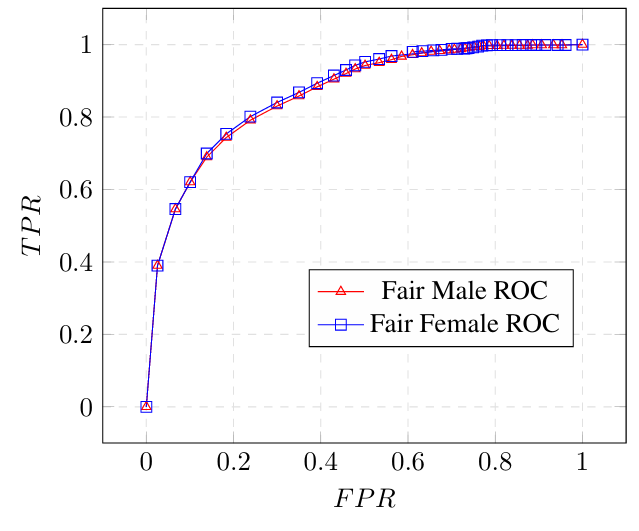}
    \caption{(Fair $\varepsilon_1 = 0.01$) FNNC-\ouralgo\   ROCs for Adult Dataset}
    \label{fig:FNNC_Adult_ROC_FROC}
\end{figure}

\begin{figure}[!ht]
    \centering
    \includegraphics[width=1\linewidth]{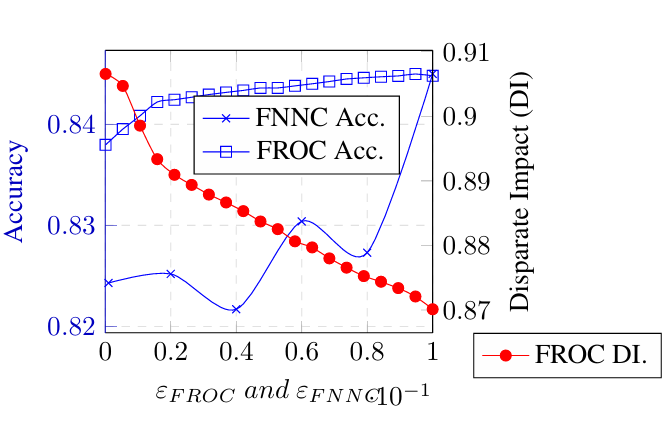}
    \caption{FNNC-\ouralgo\  Accuracy vs. $\varepsilon_1$ (Adult)}
    \label{fig:FNNC_Adult_Accuracy}
\end{figure}


\begin{figure}[!ht]
    \centering
    \includegraphics[width=1\linewidth]{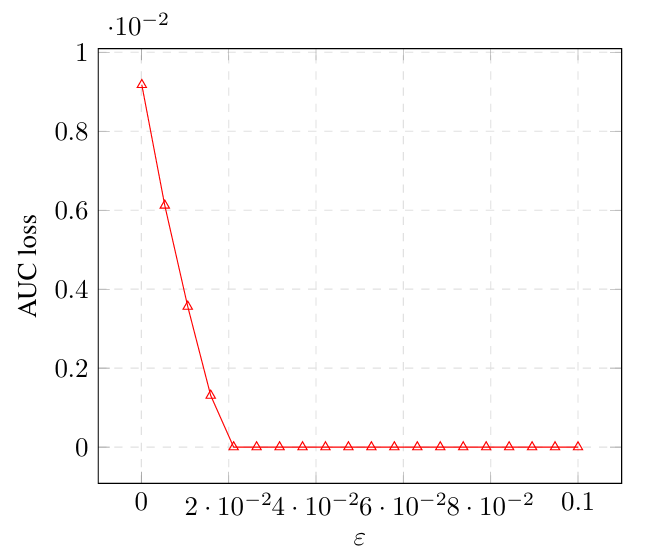}
    \caption{FNNC-\ouralgo\ AUC loss vs. $\varepsilon_1$ (Adult)}
    \label{fig:FNNC_Adult_AUC}
\end{figure}

\begin{itemize}
    \item We have applied FROC with the our fairness parameter $\varepsilon_1 = 0.01$ in \textbf{Figure \ref{fig:RFG_COMPAS_Baseline_ROC_FROC}}. As promised, the resulting ROCs are 'closer' to each other.
    \item In \textbf{Figure \ref{fig:FNNC_Adult_Accuracy}}, we have the Accuracy vs. $\varepsilon_1$ and the Disparate Impact vs. $\varepsilon_1$ plot. We also have the $\varepsilon_{FNNC} vs. \varepsilon_{FROC}$ plot.
    \item We find that in the FNNC is slightly lower than FROC in terms of accuracy. We assign it to the fact that FNNC may overachieve the target fairness for smaller values of $\varepsilon_1$ (Evident from Table 2 [Padala and Gujar 2021]). FROC drops AUC minimally to achieve target fairness.
    \item This analysis gives us a maximum variance of $6.6\times10^{-7}$ and a maximum CoV (Coefficient of Variation) of $0.09\%$ for Accuracy.
    \item As for the Disparate Impact, the analysis gives us a maximum variance of $1\times10^{-4}$ and a maximum CoV of $1.26\%$.
    \item As seen in the plots, we observe that a $1\%$ drop in Accuracy improves the Disparate Impact by $5\%$.
    \item Finally, in \textbf{Figure \ref{fig:FNNC_Adult_AUC}}, we have the AUC loss vs. $\varepsilon_1$ plot. As seen in the figure, the AUC loss decays to $0$ as our fairness constraint loosens.
\end{itemize}

\subsubsection{COMPAS Dataset - Weighted ensemble L2}

\begin{figure}[!ht]
    \centering
    \includegraphics[width=1\linewidth]{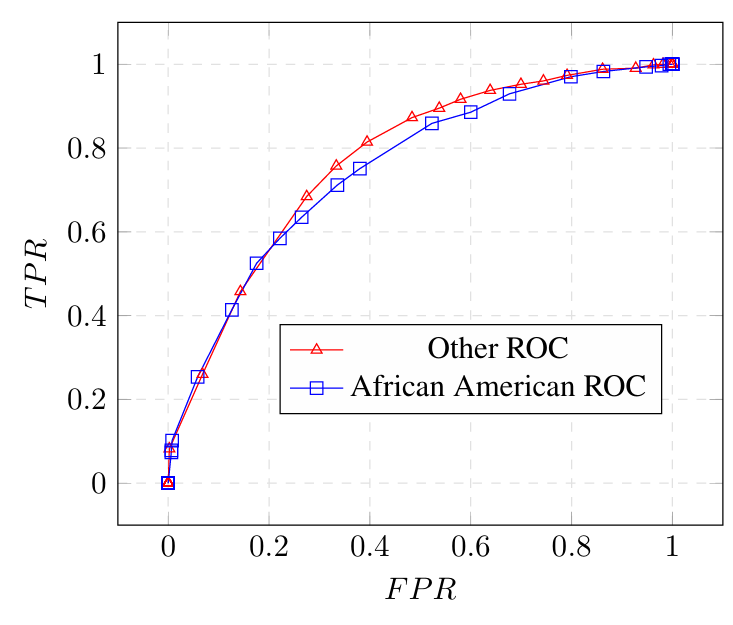}
    \caption{Weighted Ensemble L2 Baseline ROCs for COMPAS Dataset}
    \label{fig:WEL2_COMPAS_ROC}
\end{figure}

\begin{figure}[!ht]
    \centering
    \includegraphics[width=1\linewidth]{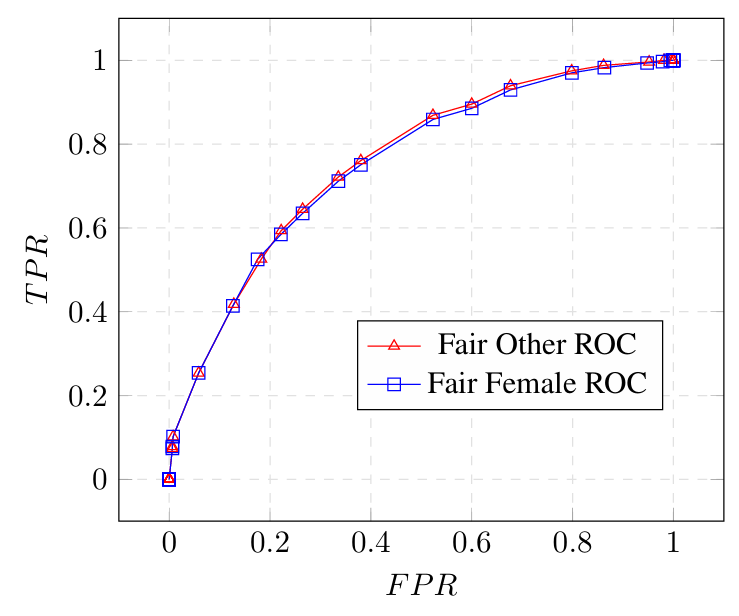}
    \caption{(Fair $\varepsilon_1 = 0.01$) Weighted Ensemble L2-\ouralgo\   ROCs for COMPAS Dataset}
    \label{fig:WEL2_COMPAS_ROC_FROC}
\end{figure}

\begin{figure}[!ht]
    \centering
    \includegraphics[width=1\linewidth]{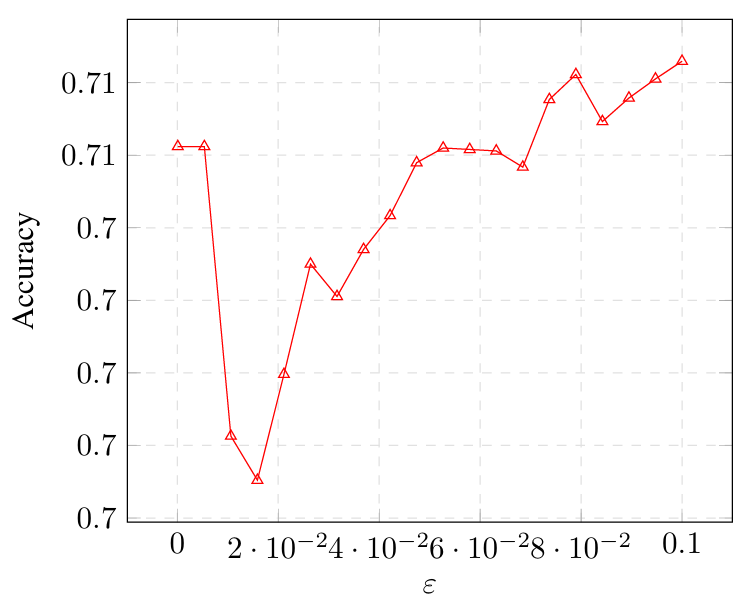}
    \caption{Weighted Ensemble L2-\ouralgo\  Accuracy vs. $\varepsilon_1$ (COMPAS)}
    \label{fig:WEL2_COMPAS_Accuracy}
\end{figure}

\begin{figure}[!ht]
    \centering
    \includegraphics[width=1\linewidth]{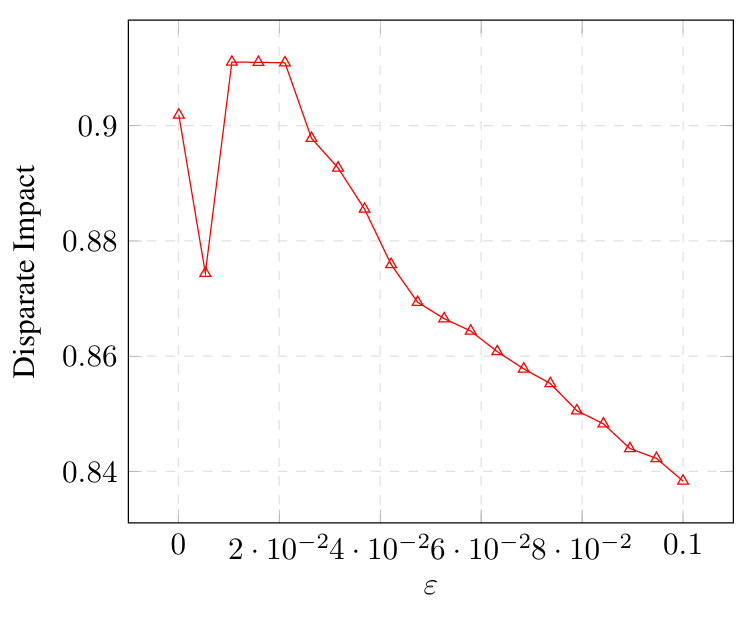}
    \caption{Weighted Ensemble L2-\ouralgo\  Disparate Impact vs. $\varepsilon_1$ (COMPAS)}
    \label{fig:WEL2_COMPAS_DI}
\end{figure}

\begin{figure}[!ht]
    \centering
    \includegraphics[width=1\linewidth]{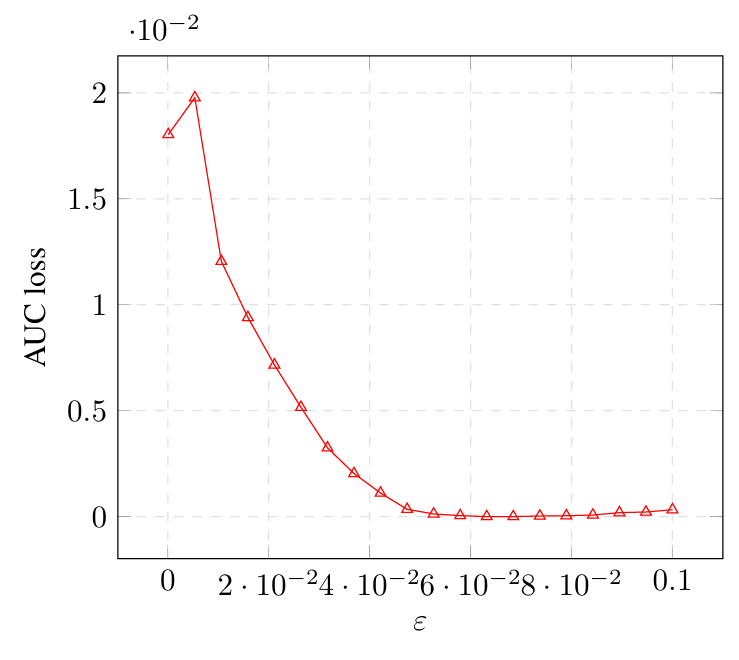}
    \caption{Weighted Ensemble L2-\ouralgo\ AUC loss vs. $\varepsilon_1$ (COMPAS)}
    \label{fig:WEL2_COMPAS_AUC}
\end{figure}

\begin{itemize}
    \item We have applied FROC with the our fairness parameter $\varepsilon_1 = 0.01$ in \textbf{Figure \ref{fig:WEL2_COMPAS_ROC_FROC}}. As promised, the resulting ROCs are 'closer' to each other.
    \item In \textbf{Figure \ref{fig:WEL2_COMPAS_Accuracy}} and \textbf{Figure \ref{fig:WEL2_COMPAS_DI}}, we have the Accuracy vs. $\varepsilon_1$ and the Disparate Impact vs. $\varepsilon_1$ plot.
    \item This analysis gives us a maximum variance of $1.44\times10^{-5}$ and a maximum CoV (Coefficient of Variation) of $0.54\%$ for Accuracy.
    \item As for the Disparate Impact, the analysis gives us a maximum variance of $1.6\times10^{-4}$ and a maximum CoV of $1.69\%$.
    \item As seen in the plots, we observe that a $1\%$ drop in Accuracy improves the Disparate Impact by $7\%$.
    \item Finally, in \textbf{Figure \ref{fig:WEL2_COMPAS_AUC}}, we have the AUC loss vs. $\varepsilon_1$ plot. As seen in the figure, the AUC loss decays to $0$ as our fairness constraint loosens.
\end{itemize}

\subsubsection{COMPAS Dataset - Random Forest Gini}


\begin{figure}[!ht]
    \centering
    \includegraphics[width=1\linewidth]{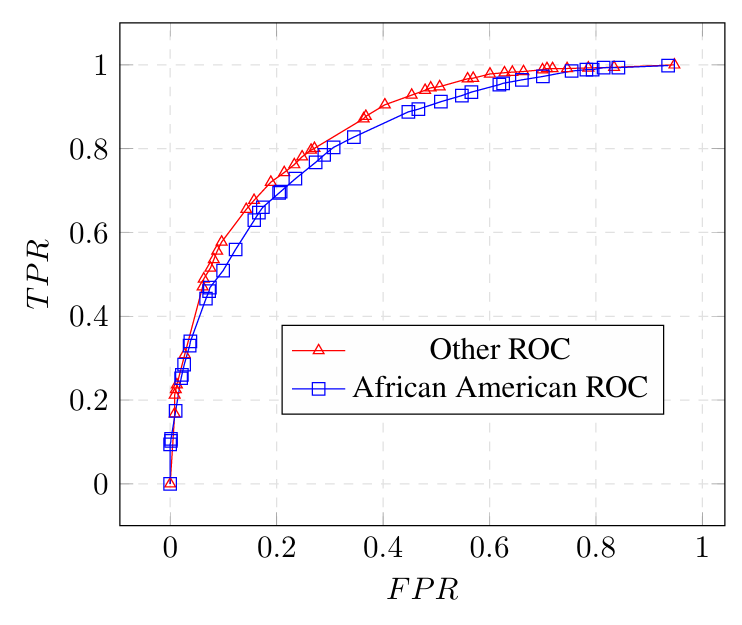}
    \caption{Random Forest (Gini) Baseline ROCs for COMPAS Dataset}
    \label{fig:RFG_COMPAS_Baseline_ROC}
\end{figure}

\begin{figure}[!ht]
    \centering
    \includegraphics[width=1\linewidth]{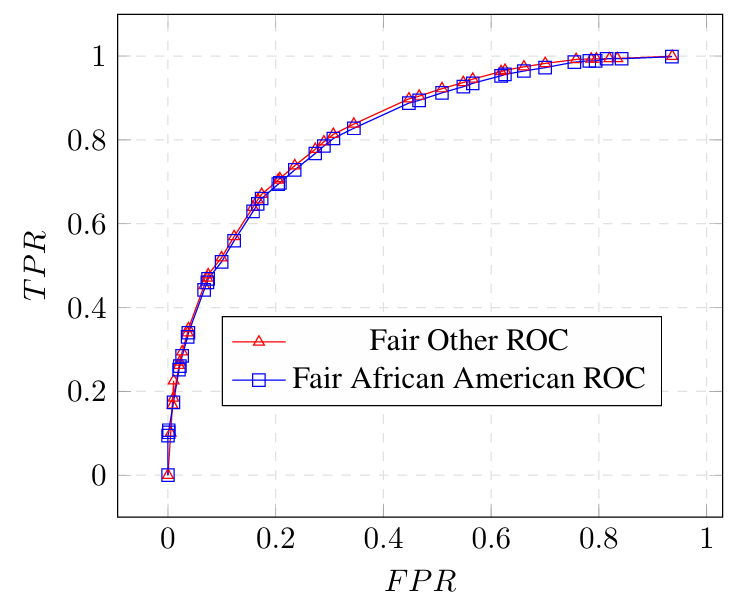}
    \caption{(Fair $\varepsilon_1 = 0.01$) Random Forest (Gini)-\ouralgo\   ROCs for COMPAS Dataset}
    \label{fig:RFG_COMPAS_Baseline_ROC_FROC}
\end{figure}

\begin{figure}[!ht]
    \centering
    \includegraphics[width=1\linewidth]{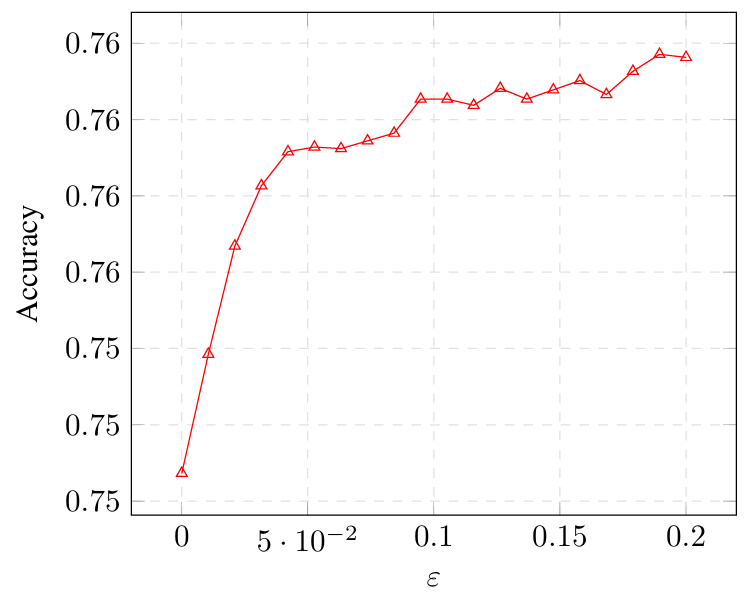}
    \caption{Random Forest (Gini)-\ouralgo\  Accuracy vs. $\varepsilon_1$ (COMPAS)}
    \label{fig:RFG_COMPAS_Accuracy}
\end{figure}

\begin{figure}[!ht]
    \centering
    \includegraphics[width=1\linewidth]{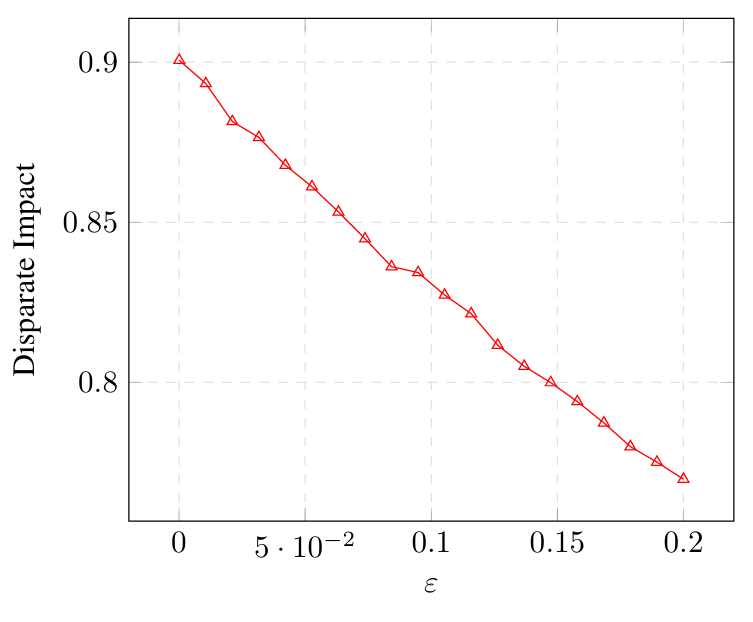}
    \caption{Random Forest (Gini)-\ouralgo\  Disparate Impact vs. $\varepsilon_1$ (COMPAS)}
    \label{fig:RFG_COMPAS_DI}
\end{figure}

\begin{figure}[!ht]
    \centering
    \includegraphics[width=1\linewidth]{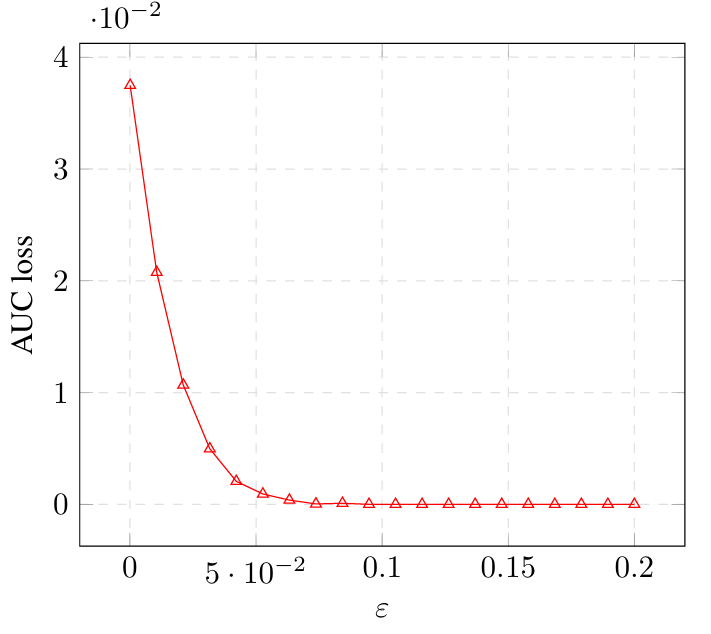}
    \caption{Random Forest (Gini)-\ouralgo\ AUC loss vs. $\varepsilon_1$ (COMPAS)}
    \label{fig:RFG_COMPAS_AUC}
\end{figure}

\begin{itemize}
    \item We have applied FROC with the our fairness parameter $\varepsilon = 0.01$ in \textbf{Figure \ref{fig:RFG_COMPAS_Baseline_ROC_FROC}}. As promised, the resulting ROCs are 'closer' to each other.
    \item In \textbf{Figure \ref{fig:RFG_COMPAS_Accuracy}} and \textbf{Figure \ref{fig:RFG_COMPAS_DI}}, we have the Accuracy vs. $\varepsilon_1$ and the Disparate Impact vs. $\varepsilon_1$ plot.
    \item This analysis gives us a maximum variance of $9.63\times10^{-6}$ and a maximum CoV (Coefficient of Variation) of $0.44\%$ for Accuracy.
    \item As for the Disparate Impact, the analysis gives us a maximum variance of $2\times10^{-4}$ and a maximum CoV of $1.56\%$.
    \item As seen in the plots, we observe that a $1\%$ drop in Accuracy improves the Disparate Impact by $7\%$.
    \item Finally, in \textbf{Figure \ref{fig:RFG_COMPAS_AUC}}, we have the AUC loss vs. $\varepsilon_1$ plot. As seen in the figure, the AUC loss decays to $0$ as our fairness constraint loosens.
\end{itemize}

\subsubsection{COMPAS Dataset - FNNC}

\begin{figure}[!ht]
    \centering
    \includegraphics[width=1\linewidth]{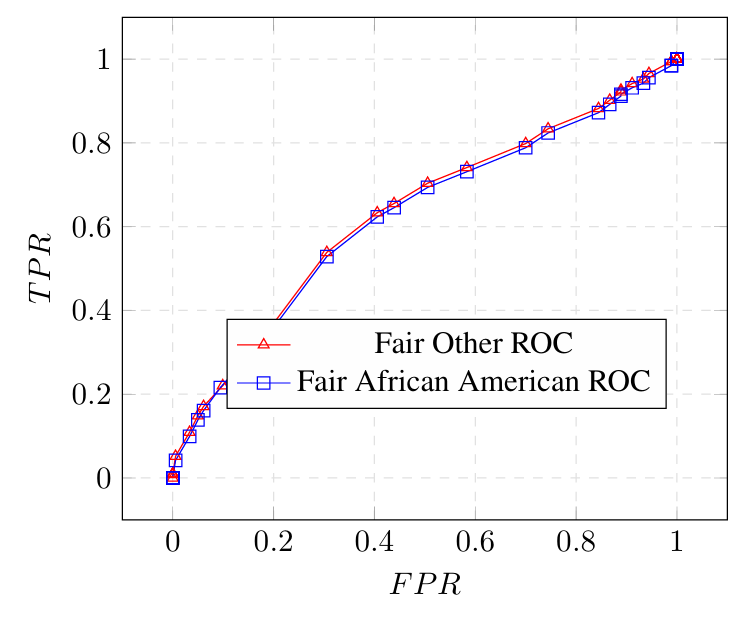}
    \caption{FNNC Baseline ROCs for COMPAS Dataset}
    \label{fig:FNNC_COMPAS_ROC}
\end{figure}

\begin{figure}[!ht]
    \centering
    \includegraphics[width=1\linewidth]{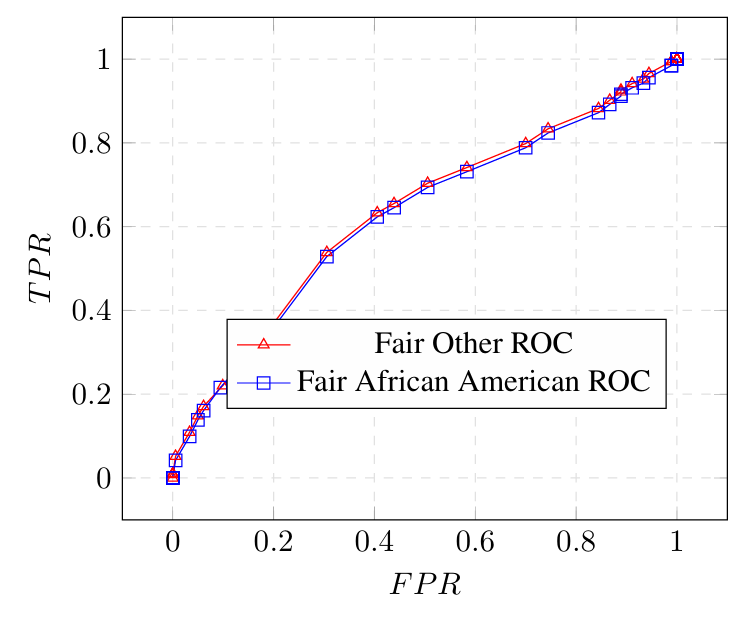}
    \caption{(Fair $\varepsilon_1 = 0.01$) FNNC-\ouralgo\   ROCs for COMPAS Dataset}
    \label{fig:FNNC_COMPAS_ROC_FROC}
\end{figure}

\begin{figure}[!ht]
    \centering
    \includegraphics[width=1\linewidth]{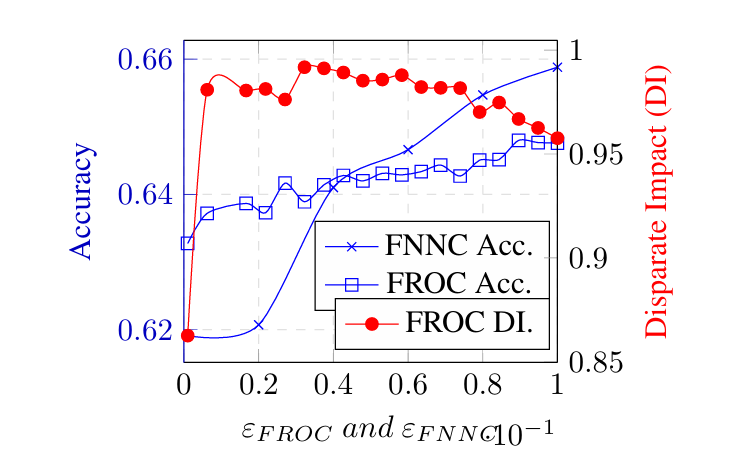}
    \caption{FNNC-\ouralgo\  Accuracy vs. $\varepsilon_1$ (COMPAS)}
    \label{fig:FNNC_COMPAS_Accuracy}
\end{figure}


\begin{figure}[!ht]
    \centering
    \includegraphics[width=1\linewidth]{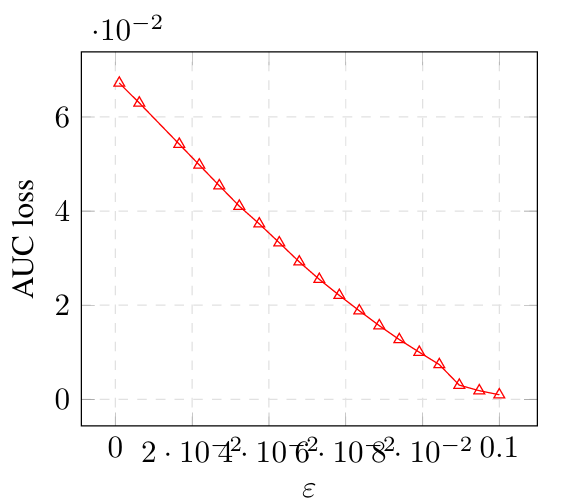}
    \caption{FNNC-\ouralgo\ AUC loss vs. $\varepsilon_1$ (COMPAS)}
    \label{fig:FNNC_COMPAS_AUC}
\end{figure}

\begin{itemize}
    \item We have applied FROC with the our fairness parameter $\varepsilon_1 = 0.01$ in \textbf{Figure \ref{fig:FNNC_COMPAS_ROC_FROC}}. As promised, the resulting ROCs are 'closer' to each other.
    \item In \textbf{Figure \ref{fig:FNNC_COMPAS_Accuracy}}, we have the Accuracy vs. $\varepsilon_1$ and the Disparate Impact vs. $\varepsilon_1$ plot.
    \item We find that in the FNNC is slightly lower than FROC in terms of accuracy. We assign it to the fact that FNNC may overachieve the target fairness for smaller values of $\varepsilon_{FNNC}$, (Evident from Table 2 [\cite{padala2020fnnc}]). FROC drops AUC minimally to achieve target fairness.
    \item This analysis gives us a maximum variance of $4.83\times10^{-6}$ and a maximum CoV (Coefficient of Variation) of $0.43\%$ for Accuracy.
    \item As for the Disparate Impact, the analysis gives us a maximum variance of $2.48\times10^{-5}$ and a maximum CoV of $0.5\%$.
    \item As seen in the plots, we observe that a $1\%$ drop in Accuracy improves the Disparate Impact by $3\%$.
    \item Finally, in \textbf{Figure \ref{fig:FNNC_COMPAS_AUC}}, we have the AUC loss vs. $\varepsilon_1$ plot. As seen in the figure, the AUC loss decays to $0$ as our fairness constraint loosens.
\end{itemize}

\subsubsection{CelebA Dataset}
\begin{figure}[!ht]
    \centering
    \includegraphics[width=1\linewidth]{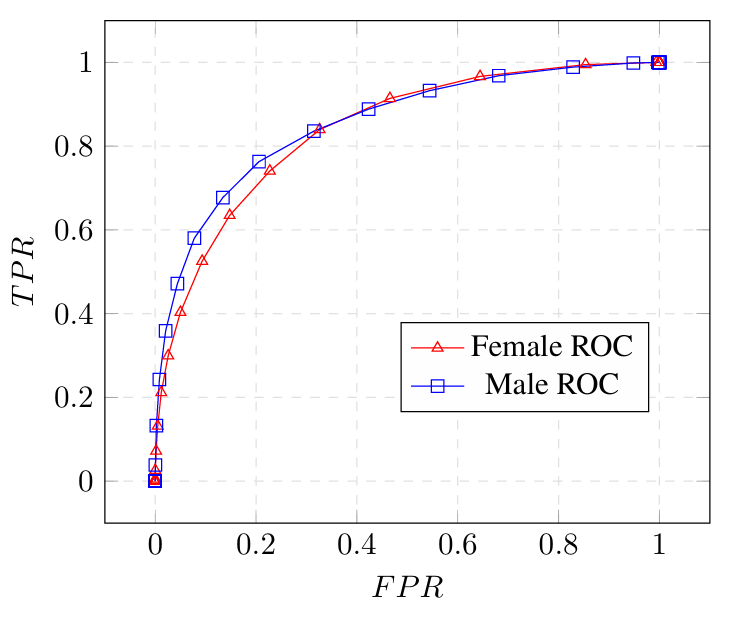}
    \caption{ResNet Baseline ROCs for CelebA Dataset}
    \label{fig:CelebA_ROC}
\end{figure}

\begin{figure}[!ht]
    \centering
    \includegraphics[width=1\linewidth]{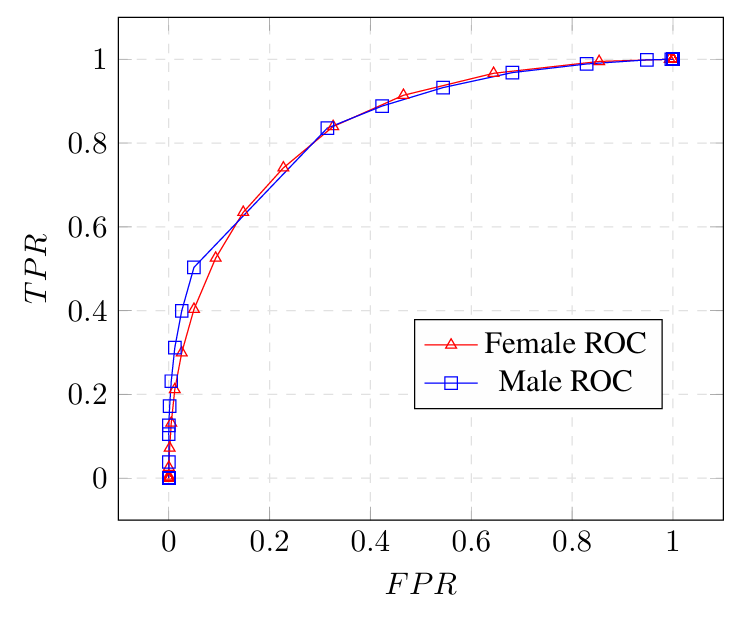}
    \caption{(Fair $\varepsilon_1 = 0.01$) ResNet-\ouralgo\ ROCs for CelebA Dataset}
    \label{fig:CelebA_ROC_FROC}
\end{figure}

\begin{figure}[!ht]
    \centering
    \includegraphics[width=1\linewidth]{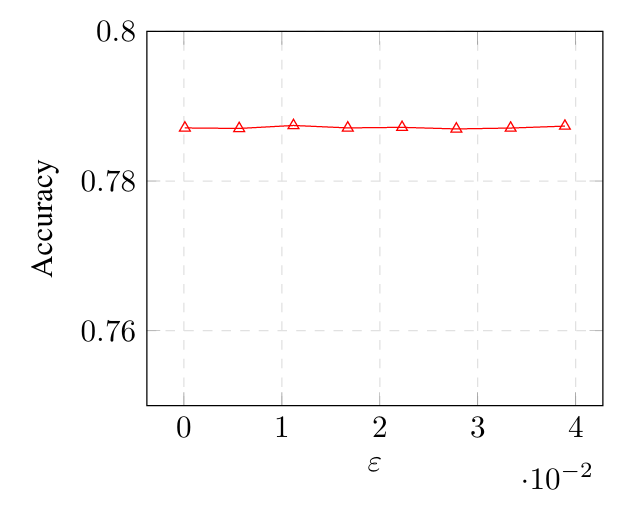}
    \caption{ResNEt-\ouralgo\  Accuracy vs. $\varepsilon_1$ (CelebA)}
    \label{fig:CelebA_Accuracy}
\end{figure}

\begin{itemize}
    \item We have applied FROC with the our fairness parameter $\varepsilon_1 = 0.01$ in \textbf{Figure \ref{fig:CelebA_ROC_FROC}}. As promised, the resulting ROCs are 'closer' to each other.
    \item This analysis gives us a maximum variance of $1.9\times10^{-7}$ and a maximum CoV (Coefficient of Variation) of $0.07\%$ for Accuracy (\textbf{Figure \ref{fig:CelebA_Accuracy}}).
    \item As for the Disparate Impact, since both the ROCs are very close to begin with, we find that there is not much improvement in terms of performance.
    \item The AUC is also similar in nature - it shows no clear trend.
\end{itemize}

\section{FROC implementation in Python}
The official and cleaned-up version of the code for this paper can be found in this \href{https://github.com/Avyukta-Manjunatha-Vummintala/FROC_code/tree/main}{link}.

\subsection{Preprocessing Code (Adult)}
\begin{lstlisting}[language=Python]
from autogluon.tabular import TabularDataset, TabularPredictor
import numpy as np
import matplotlib.pyplot as plt
from sklearn import datasets
from sklearn import metrics
from sklearn.metrics import roc_curve, roc_auc_score
from sklearn.model_selection import train_test_split
import math
import pandas as pd
import random as rd
import math




df_old = pd.read_csv('https://autogluon.s3.amazonaws.com/datasets/Inc/train.csv')

# column_names = ['age', 'workclass', 'fnlwgt', 'education', 'education-num',
#                 'marital-status', 'occupation', 'relationship', 'race', 'sex',
#                 'capital-gain', 'capital-loss', 'hours-per-week', 'native-country', 'class']



# df_old = pd.read_csv('/content/adult.data' , header = None , names = column_names)
# df.columns =  column = ['age' , 'workclass' , 'fnlwgt' , 'education-num' , 'marital-status' , 'occupation' , 'relationship', 'race' , 'sex' , 'capital-gain', 'capital-loss', 'hours-per-week', 'native-country' , 'class']
# Adult dataset is being loaded.
df = df_old.fillna(0)

print(df)
# print(df_old1)

## Modify for binary labels
df['class'].loc[df['class'] == '<=50K'] = 0
df['class'].loc[df['class'] == '>50K'] = 1

## Create the dataset
for i in list(df.columns):
    df[i] = df[i].astype('category').cat.codes


## Modify for binary protected attributes
df['sex'].loc[df['sex'] == 'Male'] = 1
df['sex'].loc[df['sex'] == 'Female'] = 0
print(df)



test_data = TabularDataset('https://autogluon.s3.amazonaws.com/datasets/Inc/test.csv')
y_test = test_data[label]  # values to predict
DF = test_data
test_data_nolab = test_data.drop(columns=[label])  # delete label column to prove we're not cheating
# test_data_nolab = test_data.drop(columns=[])  # delete label column to prove we're not cheating
test_data_nolab.head()
\end{lstlisting}
\subsection{Preprocessing Code (COMPAS)}
\begin{lstlisting}[language=Python]
from autogluon.tabular import TabularDataset, TabularPredictor
import numpy as np
import matplotlib.pyplot as plt
from sklearn import datasets
from sklearn import metrics
from sklearn.metrics import roc_curve, roc_auc_score
from sklearn.model_selection import train_test_split
import pandas as pd


import os
for dirname, _, filenames in os.walk('/kaggle/input'):
    for filename in filenames:
        print(os.path.join(dirname, filename))


data = TabularDataset('/content/propublica_data_for_fairml.csv')
data.info()
data.columns
label = 'Two_yr_Recidivism'
print("Summary of Two_yr_Recidivism variable: \n", data[label].describe())
#### Train test split ###
train_ix  = np.random.randint(0, len(data), int(0.8*len(data)))
# train_ix = range(len(data))
train = data.iloc[train_ix,:]
# train_data = train.iloc[train_ix, :]
train_data = train.iloc[:, [1,2,3,4,5,6,7,8,9,10,11]]
print(train_data)
train_labels = train.iloc[:,0]
# train_labels = train_labels[:, 0]
print(train_labels)



test_ix  = np.random.randint(0, len(data), int(0.2*len(data)))
# train_ix = range(len(data))
test = data.iloc[train_ix,:]
test_data = test.iloc[:, [1,2,3,4,5,6,7,8,9,10,11]]
print(test_data)
test_labels = test.iloc[:,0]
assert isinstance(test_labels, (np.ndarray, pd.Series))
# test_labels = test_labels[:, 0]


# train_data = pd.DataFrame(train_data, columns = ['Number_of_Priors', 'score_factor','Age_Above_FourtyFive', 'Age_Below_TwentyFive', 'African_American','Asian', 'Hispanic', 'Native_American', 'Other', 'Female','Misdemeanor'])
# train_labels = pd.DataFrame(train_labels, columns = ['Two_yr_Recidivism'])

# test_data = pd.DataFrame(test_data, columns = ['Number_of_Priors', 'score_factor','Age_Above_FourtyFive', 'Age_Below_TwentyFive', 'African_American','Asian', 'Hispanic', 'Native_American', 'Other', 'Female','Misdemeanor'])
# test_labels = pd.DataFrame(test_labels, columns = ['Two_yr_Recidivism'])
# train_prot =  tf.keras.utils.to_categorical(prot[train_ix, np.newaxis], num_classes=num_classes)
# train_labels = tf.keras.utils.to_categorical(data[train_ix, -1], num_classes=num_classes)
# train_labels = np.append(train_labels, train_prot, 1)

# test_ix = np.random.randint(0, len(data), int(0.2*len(data)))
# test_data = data[test_ix, :-1]
# test_prot =  tf.keras.utils.to_categorical(prot[test_ix, np.newaxis], num_classes=num_classes)
# test_labels = tf.keras.utils.to_categorical(data[test_ix, -1], num_classes=num_classes)
# test_labels = np.append(test_labels, test_prot, 1)
\end{lstlisting}
\subsection{Preprocessing Code (CelebA)}
\begin{lstlisting}[language=Python]
import numpy as np
import pandas as pd
import matplotlib.pyplot as plt
from collections import Counter
from sklearn.metrics import auc
from sklearn.metrics import roc_auc_score
from sklearn.preprocessing import normalize
from copy import deepcopy

# load and summarize the dataset
from pandas import read_csv
from collections import Counter
# define the dataset location
filename = 'adult.csv'
# load the csv file as a data frame
df = read_csv(filename, header=None, na_values='?')
# drop rows with missing
df = df.dropna()
# summarize the shape of the dataset
print(df.shape)
# summarize the class distribution
target = df.values[:,-1]
counter = Counter(target)
for k,v in counter.items():
	per = v / len(target) * 100
	print('Class=%s, Count=%d, Percentage=%.3f%%' % (k, v, per))

# select columns with numerical data types
num_ix = df.select_dtypes(include=['int64', 'float64']).columns
# select a subset of the dataframe with the chosen columns
subset = df[num_ix]
# create a histogram plot of each numeric variable
# subset.hist()
plt.show()

# fit a model and make predictions for the on the adult dataset
from pandas import read_csv
from sklearn.preprocessing import LabelEncoder
from sklearn.preprocessing import OneHotEncoder
from sklearn.preprocessing import MinMaxScaler
from sklearn.compose import ColumnTransformer
from sklearn.ensemble import GradientBoostingClassifier
from sklearn.ensemble import BaggingClassifier
from sklearn.svm import SVC
from sklearn.model_selection import train_test_split
from imblearn.pipeline import Pipeline

# load the dataset
def load_dataset(full_path):
  # load the dataset as a numpy array
  dataframe = read_csv(full_path, header=None, na_values='?')
  # drop rows with missing
  dataframe = dataframe.dropna()
  # split into inputs and outputs
  last_ix = len(dataframe.columns) - 1
  X, y = dataframe.drop(last_ix, axis=1), dataframe[last_ix]
  # select categorical and numerical features
  cat_ix = X.select_dtypes(include=['object', 'bool']).columns
  num_ix = X.select_dtypes(include=['int64', 'float64']).columns
  # label encode the target variable to have the classes 0 and 1
  y = LabelEncoder().fit_transform(y)
  return X.values, y, cat_ix, num_ix

# define the location of the dataset
full_path = 'adult.csv'
# load the dataset
X, y, cat_ix, num_ix = load_dataset(full_path)
# define model to evaluate
model = GradientBoostingClassifier(n_estimators=100)
model2 = SVC()

# one hot encode categorical, normalize numerical
ct = ColumnTransformer([('c',OneHotEncoder(handle_unknown = 'ignore'),cat_ix), ('n',MinMaxScaler(),num_ix)])
# define the pipeline
pipeline = Pipeline(steps=[('t',ct), ('m',model)])
pipeline2 = Pipeline(steps=[('t',ct), ('m',model2)])
# split test and train data
X_train, X_test, y_train, y_test = train_test_split(X, y, test_size=0.33)
# fit the model
trained_model = pipeline.fit(X_train, y_train)
trained_model2 = pipeline2.fit(X_train, y_train)
\end{lstlisting}
\subsection{FROC}
\begin{lstlisting}[language=Python]
def Cover( curve_x , curve_y , x , y ):
  j = 0
  for i in range(len(curve_x)):
    if ( i == len(curve_x)-1):
      print("Case")
      if( x <= curve_x[i] ):
        if( y <= curve_y[i]):
          return 1
        else:
          return 0
      else:
        return 0
      continue
    if (curve_x[i] <= x and curve_x[i+1] >= x):
      if( y <= curve_y[i] + (curve_y[i+1] - curve_y[i])*(x - curve_x[i])/(curve_x[i+1] - curve_x[i]) ):
        return 1
      else:
        return 0

def LinInterpolFill(X, Y, n):
    """
    Linearly interpolate between consecutive (X,Y) coordinates and fill in n points between them.

    Parameters:
    X (list): List of x-coordinates.
    Y (list): List of y-coordinates.
    n (int): Number of points to interpolate between each consecutive (X,Y) pair.

    Returns:
    x_interpolated (list): List of interpolated x-coordinates.
    y_interpolated (list): List of interpolated y-coordinates.
    """

    x_interpolated = []
    y_interpolated = []

    for i in range(len(X)-1):
        x0 = X[i]
        x1 = X[i+1]
        y0 = Y[i]
        y1 = Y[i+1]

        for j in range(n+1):
            x_j = x0 + (x1-x0)*j/n
            y_j = y0 + (y1-y0)*j/n
            x_interpolated.append(x_j)
            y_interpolated.append(y_j)

    return x_interpolated, y_interpolated


def FROC_original( iFPR0 , iTPR0 , iFPR1 , iTPR1 , granularity , epsilon ):
  # plt.plot( iFPR0 , iTPR0 , iFPR1 , iTPR1 )
  FPR0 = iFPR0.copy()
  TPR0 = iTPR0.copy()
  FPR1 = iFPR1.copy()
  TPR1 = iTPR1.copy()
  FPR0 = np.flip(FPR0)
  TPR0 = np.flip(TPR0)
  FPR1 = np.flip(FPR1)
  TPR1 = np.flip(TPR1)

  FFPR0 = FPR0.copy()
  FTPR0 = TPR0.copy()
  FFPR1 = FPR1.copy()
  FTPR1 = TPR1.copy()

  # plt.plot( FFPR0 , FTPR0 )


  # plt.plot( iFPR0 , iTPR0 , iFPR1 , iTPR1 )
  linFPR0 , linTPR0 = LinInterpolFill( FPR0 , TPR0 , granularity)
  # plt.plot( iFPR0 , iTPR0 , iFPR1 , iTPR1 )

  init = 0.2
  fin = 1

  n = len(FPR0)
  notFair = list(range(n))
  # plt.plot( iFPR0 , iTPR0 , iFPR1 , iTPR1 )

  for i in range(len(notFair)):
    if( FPR0[i] < init or FPR0[i] > fin):
      notFair[i] = 'f'
      # print("Preprocessing range removed: ",i)
      FFPR0[i] = FPR1[i]
      FTPR0[i] = TPR1[i]
  # plt.plot( iFPR0 , iTPR0 , iFPR1 , iTPR1 )

  for i in range(len(notFair)):
    if( abs(FPR0[i] - FPR1[i]) + abs(TPR0[i] - TPR1[i]) <= epsilon ):
      notFair[i] = 'f'
      # print("Preprocessing already fair: ",i)
  # plt.plot( iFPR0 , iTPR0 , iFPR1 , iTPR1 )

  while 'f' in notFair:
    notFair.remove('f')

  plt.plot( FFPR0 , FTPR0 , FFPR1 , FTPR1 )


  for i in range(len(notFair)):
    # print("In loop")
    # plt.plot( iFPR0 , iTPR0 , iFPR1 , iTPR1 )
    # print(Cover(FPR0 , TPR0 , FPR0[notFair[i]] , TPR0[notFair[i]]+epsilon))
    # print("Group0: ",FPR0[notFair[i]] , TPR0[notFair[i]])
    # print("Group1: ",FPR1[notFair[i]] , TPR1[notFair[i]])
    if( Cover(FPR0 , TPR0 , FPR1[notFair[i]] , TPR1[notFair[i]]+epsilon)  == 1):
      # plt.plot( iFPR0 , iTPR0 , iFPR1 , iTPR1 )
      FTPR0[notFair[i]] = TPR1[notFair[i]] + epsilon
      FFPR0[notFair[i]] = FPR1[notFair[i]]
      # print("Upshift done", FFPR0[notFair[i]] , FTPR0[notFair[i]])
    else:
      for j in range(len(linFPR0)-1):
        if( abs(linFPR0[j] - FPR1[notFair[i]]) + abs(linTPR0[j] - TPR1[notFair[i]]) <= epsilon and abs(linFPR0[j+1] - FPR1[notFair[i]]) + abs(linTPR0[j+1] - TPR1[notFair[i]]) > epsilon  ):
          # print("Cut")
          FFPR0[notFair[i]] = linFPR0[j]
          FTPR0[notFair[i]] = linTPR0[j]
        # else:
          # print("Not Cut")

  # print( notFair )
  return FFPR0 ,FTPR0 , FFPR1 , FTPR1
\end{lstlisting}
\subsection{Building the Classifier}
\begin{lstlisting}[language=Python]
def findInterval( vector , value ):
    # vector is a sorted vector.
    # if the vector is not sorted in a decreasing order, then declare an error.
    # Check if the vector is sorted in a decreasing order.
    for i in range(1, vector.shape[0]):
        if vector[i] > vector[i - 1]:
            print("Error: vector is not sorted in a decreasing order.")
            return -1
        
    # iF the value is outside the range of the vector, then throw an error.
    if value < vector[-1] or value > vector[0]:
        print("Error: value is outside the range of the vector.")
        return -1
        
    # Else, find the interval in which the value lies.
    for i in range( vector.shape[0] ):
        if vector[i] < value:
            return i - 1
        
    
# Now, let us test the findInterval function
vector = np.linspace(1, 0, 10)
print(vector)
value = 0.5

print(findInterval(vector, value))
            

def returnCoeff( ul , ur , dn  , p ):
    # let ul = (a,b)
    # let ur = (c,d)
    # let p = (x,y)
    # let dn = (x,x)
    # Assert that dn[0] == dn[1] == p[0]
    assert dn[0] == dn[1] == p[0]
    a = ul[0]
    b = ul[1]
    c = ur[0]
    d = ur[1]
    x = p[0]
    y = p[1]

    # Now, if p is equal to any of the other points, then return the coefficient as 1 for that point and 0 for the other points.
    if p[0] == ul[0] and p[1] == ul[1]:
        return (1, 0, 0)
    elif p[0] == ur[0] and p[1] == ur[1]:
        return (0, 1, 0)
    elif p[0] == dn[0] and p[1] == dn[1]:
        return (0, 0, 1)
    
    # Now, we find the coefficients of the line joining ul and ur.
    # Let h = ((c-x)/(c-a))*b + ((x-a)/(c-a))*d
    h = ((c-x)/(c-a))*b + ((x-a)/(c-a))*d
    # If c == a, then throw an error.
    if c == a:
        print("Error: Division by zero because c == a.")
        return -1
    
    # Now, we find C_ul, C_ur, C_dn
    C_ul = ((y - x)*(c - x))/((h - x)*(c - a))
    C_ur = ((y - x)*(x - a))/((h - x)*(c - a))
    C_dn = (h - y)/(h - x)

    # Now, if any of the coefficients are negative, then throw an error.
    if C_ul < 0 or C_ur < 0 or C_dn < 0:
        print("Error: Negative coefficient.")
        return -1
    
    # Now, if any of the coefficients are greater than 1 or nan, then throw an error.
    if C_ul > 1 or C_ur > 1 or C_dn > 1 or np.isnan(C_ul) or np.isnan(C_ur) or np.isnan(C_dn):
        print("Error: Coefficient greater than 1 or nan.")
        return -1
    
    # Assert that C_ul + C_ur + C_dn = 1
    # print(C_ul + C_ur + C_dn)
    assert C_ul + C_ur + C_dn - 1 < 0.00001

    # If any of the coefficients are na because of division by zero, then throw an error.
    if np.isnan(C_ul) or np.isnan(C_ur) or np.isnan(C_dn):
        print("Error: Division by zero.")
        return -1
    
    # If any of the nocoefficients are negative, then throw an error.
    if C_ul < 0 or C_ur < 0 or C_dn < 0:
        print("Error: Negative coefficient.")
        return -1

    # Assert that C_ul + C_ur + C_dn = 1
    # print(C_ul + C_ur + C_dn)
    assert C_ul + C_ur + C_dn - 1 < 0.00001

    # If C_ul + C_ur + C_dn != 1, then C_ul = 1 - C_ur - C_dn


    # Assert that C_ul*ul + C_ur*ur + C_dn*dn = p
    # print(C_ul*ul + C_ur*ur + C_dn*dn , p)
    assert C_ul*ul[0] + C_ur*ur[0] + C_dn*dn[0] - p[0] < 0.00001
    assert C_ul*ul[1] + C_ur*ur[1] + C_dn*dn[1] - p[1] < 0.00001

    # print(C_ul + C_ur + C_dn)


    return (C_ul, C_ur, C_dn)

    


# Test the returnCoeff function
ul = np.array([5, 5])
ur = np.array([8, 20])
dn = np.array([6, 6])

p = np.array([6, 7])

print(returnCoeff(ul, ur, dn, p))

def buildClassifier( ROC_up , Probs_up , point , y_test_up):
    # x = point[0] , y = point[1]
    x = point[0]
    y = point[1]

    # Now, we find the interval in ROC_up[0] in which x lies.
    interval = findInterval( ROC_up[0] , x )

    # Now, thresholds = np.linspace(0, 1, 1000)
    thresholds = np.linspace(0, 1, len(ROC_up[0]))

    # Create a classifier output using threshold = thresholds[interval]
    classifier_output = np.zeros( Probs_up.shape[0] )
    classifier_output[ Probs_up >= thresholds[interval] ] = 1

    # Find the FPR and TPR of the classifier_output
    FPR = np.sum( classifier_output * (1 - y_test_up) ) / np.sum( 1 - y_test_up )
    TPR = np.sum( classifier_output * y_test_up ) / np.sum( y_test_up )

    # Assert that FPR == ROC_up[0][interval] and TPR == ROC_up[1][interval]
    # print(FPR, TPR)
    # print(ROC_up[0][interval], ROC_up[1][interval])
    assert FPR == ROC_up[0][interval]
    assert TPR == ROC_up[1][interval]

    ul = np.array([ROC_up[0][interval], ROC_up[1][interval]])
    ur = np.array([ROC_up[0][interval + 1], ROC_up[1][interval + 1]])
    dn = np.array([x,x])
    p = np.array([x,y])

    C_ul , C_ur , C_dn = returnCoeff( ul , ur , dn , p )

    # print(C_ul, C_ur, C_dn)

    # Now, we create the classifier output for threshold = thresholds[interval + 1]
    classifier_output1 = np.zeros( Probs_up.shape[0] )
    classifier_output1[ Probs_up >= thresholds[interval + 1] ] = 1

    # Now, we create a random array with 1 with probability x and 0 with probability 1 - x
    rand_array = np.random.choice(2, Probs_up.shape[0], p=[1-x, x])
    # print(rand_array)

    # Check if the random array FPR and TPR are equal to x and x
    FPR = np.sum( rand_array * (1 - y_test_up) ) / np.sum( 1 - y_test_up )
    TPR = np.sum( rand_array * y_test_up ) / np.sum( y_test_up )

    # Assert that FPR == x and TPR == x
    # print(FPR, TPR , x)
    # assert FPR == x
    # assert TPR == x


    # Now, we create the final classifier output
    final_classifier_output = np.zeros( Probs_up.shape[0] )

    for i in range( Probs_up.shape[0] ):
        flag = 0
        # Create a random number that takes 0 with probability C_ul, 1 with probability C_ur and 2 with probability C_dn
        rand_num = np.random.choice(3, 1, p=[C_ul, C_ur, C_dn])
        if rand_num == 0:
            final_classifier_output[i] = classifier_output[i]
            flag = 1
        elif rand_num == 1:
            final_classifier_output[i] = classifier_output1[i]
            flag = 1
        else:
            final_classifier_output[i] = rand_array[i]
            flag = 1

        # Assert that flag == 1
        assert flag == 1

    # Now, we find the FPR and TPR of the final_classifier_output
    FPR = np.sum( final_classifier_output * (1 - y_test_up) ) / np.sum( 1 - y_test_up )
    TPR = np.sum( final_classifier_output * y_test_up ) / np.sum( y_test_up )

    # Assert that FPR == x and TPR == y
    # print(FPR, TPR)
    # print(x, y)
    assert FPR - x < 0.1
    assert TPR - y < 0.1

    return final_classifier_output



# Let us now test the buildClassifier function
j = 60
print(Female_FROC[0][j] , Female_FROC[1][j])
print(Female_ROC[0][j] , Female_ROC[1][j])
vec = buildClassifier(Female_ROC , Female_prob[:, 1] , [Female_FROC[0][j] , Female_FROC[1][j]] , Female_y_test)

def isOne( vector ):
    # vector is a vector of 0s and 1s
    # If all the elements of the vector are 1, then return 1
    # Else, return 0
    for i in range( vector.shape[0] ):
        if vector[i] != 1:
            return 0
    return 1
    

isOne(vec)
\end{lstlisting}

\bibliography{main}
\bibliographystyle{unsrt}
\end{document}